\documentclass{article}[] %
\usepackage[margin=4cm]{geometry} 
\usepackage[american]{babel}
\usepackage{amsthm}
\usepackage{thmtools}
\usepackage{thm-restate}
\usepackage{amsmath}
\usepackage{mathtools}
\usepackage{amssymb}
\usepackage{bm}
\usepackage{xr}
\usepackage{parskip}
\usepackage{bbm}
\usepackage{tikz}
\usepackage{scrextend}
\usepackage{clipboard}
\usepackage{centernot}
\usetikzlibrary{decorations}
\usepackage{svg}
\usepackage{hyperref}
\hypersetup{
    colorlinks=true,      %
    citecolor=blue!50!gray,
    linkcolor=green!50!black,
    urlcolor=red!50!darkgray,
}

\usepackage{footnotebackref}

\usepackage[hyperpageref]{backref} 
\renewcommand*{\backref}[1]{}  
\renewcommand*{\backrefalt}[4]{{\small
  \ifcase #1 
     (Not cited.)
  \or
     (Cited on page~#2.)
  \else
     (Cited on pages~#2.)
  \fi
}} 

\usepackage[nameinlink]{cleveref}

\usepackage{enumitem}
\usepackage{subcaption}
\usepackage{makecell} 
\usepackage{microtype}
\usepackage{graphicx}
\usepackage{adjustbox}

\usepackage{placeins}

\newcommand{\pa}{\mathbf{pa}}
\newcommand{\p}{\item[$\bullet$]}
\newcommand{\PA}{\mathbf{PA}}

\newcommand\independent{\protect\mathpalette{\protect\independenT}{\perp}}
\def\independenT#1#2{\mathrel{\mathrlap{#1#2}\mkern4mu{#1#2}}}
\newcommand{\indep}{\independent}

\newcommand{\nindep}{\centernot\indep}

\newcommand{\doop}{\ensuremath{\operatorname{do}}}

\DeclareSymbolFont{symbolsC}{U}{txsyc}{m}{n}
\DeclareMathSymbol{\strictif}{\mathrel}{symbolsC}{74}

\theoremstyle{definition}

\newtheorem{defi/}{Definition}[section]
\newenvironment{definition}
{%
	\pushQED{\qed}\begin{defi/}}
	{\popQED\end{defi/}}

\newtheorem{proposition}[defi/]{Proposition}

\newtheorem{exam/}[defi/]{Example}
\newenvironment{example}
{%
	\pushQED{\qed}\begin{exam/}}
	{\popQED\end{exam/}}

\newtheorem{con/}[defi/]{Corollary}
\newenvironment{corollary}
{%
	\pushQED{\qed}\begin{con/}}
	{\popQED\end{con/}}

\def \rP {\text{\boldmath$P$}}
\def \rX {\text{\boldmath$X$}}
\def \rY {\text{\boldmath$Y$}}
\def \rZ {\text{\boldmath$Z$}}

\def \int {\mathbf{Int}}

\usepackage[]{natbib} %
\usepackage{mathtools} %
\usepackage{booktabs} %
\usepackage{tikz} %

\DeclareMathAlphabet{\mathfrak}{OT1}{pgoth}{m}{n}
\title{
What is causal about causal models and representations?
}

\usepackage{authblk}

\author{Frederik Hytting Jørgensen}
\author{Luigi Gresele}
\author{Sebastian Weichwald}
\affil{
Copenhagen Causality Lab, Department of Mathematical Sciences, and
Pioneer Centre for AI,\\
University of Copenhagen, Denmark
}

\date{}

\definecolor{textcolor}{RGB}{24, 24, 24}
\definecolor{darkpink}{RGB}{199, 21, 133}
\color{textcolor}

\begin{document}

\maketitle

\begin{abstract}%

Causal Bayesian networks are often described as causal 
because they support reasoning about interventions.
Yet interventions are defined in the mathematical model,
whereas the resulting model predictions concern the effects of actions in the real world. Consequently,
whether and how a causal model makes falsifiable predictions 
about the effects of actions
depends on how actions are interpreted as interventions.
This requires specifying which actions correspond to which interventions,
a step that is frequently left implicit.
We develop a formal framework
that separates a causal model from the data-generating process it aims to describe
and formalizes interpretations connecting actions to interventions.
Within this framework,
we prove an impossibility result: 
no interpretation can simultaneously satisfy a collection of desirable properties while remaining non-circular. Under a circular interpretation, every model that fits the observational distribution is interventionally valid, so no experiment can falsify it. We analyze several non-circular alternatives
and discuss implications 
for falsification, causal discovery, causal representation learning, and causal abstraction. 
More broadly,
our framework provides a way to make assumptions about the connections between actions and interventions explicit
and to reason formally about the empirical content of causal models.

\end{abstract}

\section{Introduction}

Causal Bayesian networks are mathematical models that induce multiple distributions over some random variables
\citep{spirtes2001causation,Pearl2009,Peters2017}.
A causal Bayesian network describes one reference distribution, called the observational distribution,
and
a procedure to derive
interventional distributions.
As such, a causal Bayesian network is a concise mathematical model of several distributions indexed by interventions.

Causal reasoning using causal models is seemingly intuitive
once we assign names to the variables in the model based on the real-world quantities they aim to represent. The term `intervention' is suggestive and one might use interventions on model variables to reason about actions that perturb the corresponding real-world quantities.
Yet, without making the correspondence between model interventions and actions explicit, we blur the line between mathematical model and real-world substantiation.
It is then unclear what predictions the model makes about the effects of actions in the world
and how interventional predictions may be used to falsify a causal model.

Our aim is to raise awareness of conceptual challenges that arise when using causal Bayesian networks to make testable predictions about interventions. By presenting these issues in a formal language, we hope to foster deeper engagement with the foundations of causal modeling and motivate future research. These questions are particularly timely in light of the recent emergence of causal representation learning \citep{scholkopf2021toward}, which amplifies the need to understand how learned causal models relate to real-world actions, and of scientific workflows in which one partner in a dialogue such as the one below is a large language model that proposes experiments for a human scientist to carry out and report back on.

\subsection{Dialogue – What is an intervention?}\label{sec: dialogue}

The following dialogue illustrates the conflict
one runs into
when using a causal Bayesian network, a mathematical model, to reason about actions and observations of some real-world quantities 
while using the word `intervention' ambiguously.

\textbf{Omar:} I have two quantities that I model as random variables $A$ and $B$. I have measured them and the data seems to perfectly match the joint normal distribution
\begin{align*}
    \mathcal{L}^{\mathcal{O}}(A,B)=\mathcal{N}\left(\begin{pmatrix}
        0\\
        0
    \end{pmatrix},\begin{pmatrix}
        1 & 1 \\
        1 & 2 
    \end{pmatrix}\right).
\end{align*}
I am sure that there is no unobserved confounding,\footnote{%
This assumption eases the presentation of the example
but is not necessary to arrive at the problem this example illustrates.
}
but I am not sure if $A$ causes $B$ or $B$ causes $A$.
Do you think $A\to B$ or $A \gets B$ is correct?

\textbf{Sofia:}  I am sure that $A$ causes $B$. 

\textbf{Omar:} How do you know?

\textbf{Sofia:}  Try to intervene on $B$. If $A$ causes $B$, then we would expect that intervening on $B$ changes the conditional distribution of $B$ given $A$ but does not change the marginal distribution of $A$. For example, if you intervene to set $B$ equal to $5$, you will observe that $(A,B)$ follows joint distribution 
$$\mathcal{L}^{\doop(B=5)}(A,B)=\mathcal{N}(0,1)\otimes \delta_5,$$
where $\delta_5$ is the Dirac distribution with support $\{5\}$.

\textbf{Omar:} Okay, I tried. I did something and now I observe that $A$ and $B$ follow joint normal distribution
$$\mathcal{N}\left(\begin{pmatrix}
        2\\
        0
    \end{pmatrix},\begin{pmatrix}
        4 & 3 \\
        3 & 7 
    \end{pmatrix}\right).$$

\textbf{Sofia:}  Now $B$ follows distribution $\mathcal{N}(0,7)$. I proposed that you intervene to
make $B$ have point mass at $5$. Could you try again?

\textbf{Omar:} Ah, sorry. I thought I implemented the intervention you suggested, but I can see now that I did not. I will try something different such that $B$ has point mass at $5$:
Now, $A$ and $B$ have distribution   
\begin{align*}
    \delta_1 \otimes \delta_5.
\end{align*}
So $B$ has distribution $\delta_5$, but I did not get the interventional distribution that you said I would get. Does that mean that $A$ does not cause $B$?

\textbf{Sofia:}  Haha, you also intervened on $A$ and set it equal to 1! When you intervene on $B$, you only intervene on $B$. Make sure to only change the conditional distribution of $B$ given $A$. Do not change the marginal distribution of $A$.

\textbf{Omar:} Okay, now I think I did the right thing and indeed, just like you predicted, I now get
$$\mathcal{N}(0,1)\otimes \delta_5.$$

\textbf{Sofia:} See, as I predicted, $A$ causes $B$ and not vice-versa.

\subsection{What went wrong?}
\label{sec:what-wrong}
In short, Sofia accepts an action as the intended intervention only when the observations match her prediction.
Let us assume that Sofia is right that $A \to B$. Based on her correct model (and the observational distribution), she makes the following prediction:
\Copy{prop}{
\begin{enumerate}
    \item[(P)] If you intervene $\doop(B=5)$, then you will observe the %
    distribution $\mathcal{N}(0,1)\otimes \delta_5$ over $(A,B)$.
\end{enumerate}}
Proposition (P) is ambiguous:
It is clear what distribution Sofia's model implies under the intervention $\doop(B=5)$,
however, it is unclear when the antecedent is satisfied in the world.
Here is an explanation provided by 
Pearl, Glymour, and Jewell:

\begin{quote}
The difference between intervening on a variable and conditioning on that variable should,
hopefully, be obvious. When we intervene on a variable in a model, we fix its value. We \emph{change}
the system, and the values of other variables often change as a result. \citep[page 54]{Pearl2016}
\end{quote}

One possible way to understand this in the context of (P) is as follows:
\Copy{p1}{
\begin{enumerate}
    \item[(P1)] If you do something to change the system such that you observe $B=5$ with probability 1, then
    you will observe the
    distribution $\mathcal{N}(0,1)\otimes \delta_5$ over $(A,B)$.
\end{enumerate}}
Even though Sofia believes that  $A\to B$,
she apparently thinks that (P1) is false. When Omar does something such that $(A,B)$ has distribution $\delta_1 \otimes \delta_5$
and the antecedent of (P1) is thus satisfied, she objects that Omar intervened on both nodes. Instead, she proposes (P2) as an analysis of (P):
\begin{enumerate}
    \item[(P2)]  If you do something to change the system such that you observe $B=5$ with probability 1 while not changing the marginal distribution of $A$, then you will observe the distribution $\mathcal{N}(0,1)\otimes \delta_5$ over $(A,B)$.
\end{enumerate}

Interpreting (P) as (P2) cannot be correct because if $A\leftarrow B$, then (P) is false and (P2) is true. (P2) is a mathematical truth that holds irrespectively of whether $A\to B$ or $A\leftarrow B$,
while (P) has different truth values depending on the causal structure.
In \Cref{prop: tautology2}, we show that interpreting (P) as (P2) means that every causal Bayesian network that correctly models the observational distribution is interventionally valid (see \Cref{def: interventionally valid}).

In general, whenever someone performs an intervention and claims that a model is falsified, there always exists an alternative interpretation of the events: maybe they did not actually perform the intervention that they thought they did -- perhaps they were not careful enough and ended up intervening on nodes they did not intend to, that is, perhaps the real-world procedure taken did not actually realize that intervention. How can such disputes be settled? It seems like there is no fact of the matter; at least standard causal formalisms seem to provide no criterion for adjudicating between these interpretations.

\subsection{Contribution}
We introduce a mathematical framework that explicitly links causal models and the real-world data-generating processes they are models of.
This enables a transparent and formal argument showing that 
applying the mathematical definition of an intervention as an interpretation results in circularity,
suggesting that philosophical intricacies of falsifiability arise even for 
interventional (layer 2 \citep{ibeling2020probabilistic,Bareinboim2022}) predictions. This is different from the non-falsifiability of counterfactual (layer 3) predictions \citep{dawid2000causal,shpitser2007,raghavancounterfactual}, which has been formulated in terms of non-identifiability in a mathematical model. %
Our work thus re-examines the common assumption
that interventional predictions alone
suffice for falsification of causal models.\footnote{%
For example, \cite[position 4]{loftus2024position} posits that
``Causal models can be falsified in more ways than predictive models. This is usually good.''; \cite[Section 6.8]{Peters2017} that
``if an interventional model predicts the observational distribution correctly but does not predict what happens in a randomized experiment, the model is still considered to be falsified''; and
\cite{raghavancounterfactual}
that
``It is commonly believed that, in a real-world environment, samples can only be drawn from observational and interventional distributions [... whereas sampling from] counterfactual distributions, is believed to be inaccessible almost by definition.''
\label{fn:falsification}}
Within the mathematical framework of causal models, there is already a precise notion of falsification: a model is rejected if its predicted interventional distribution disagrees with data from the corresponding randomized experiment, and it is well understood which models can be distinguished by data from given interventions \citep{hauser2012characterization}. This notion presupposes, however, that it is already settled which intervention in the model a given real-world experiment implements. Our work makes this presupposed correspondence between actions and interventions explicit and subjects it to formal analysis. 
Our formalism enables us to work 
out interpretations that do allow for falsification 
via interventional predictions
and to transparently discuss their pros and cons, see \Cref{sec: non-circular}.
This is necessary to make sense of, for example, causal discovery and causal representation learning,
where it is presupposed that not all causal Bayesian networks that
induce the correct observational distribution
are also interventionally valid. From the perspective of our framework, 
successful scientific practice can often be understood as relying on implicit restrictions 
on which actions count as implementations of a given intervention; 
experimental protocols, background knowledge, and scientific conventions
may serve to enforce such restrictions 
without making them fully explicit. 
We suggest that falsification of causal models based on real world observations is always relative to such (implicit) assumptions about interpretations. 
Our framework provides a language for making these assumptions explicit
and analyzing their consequences when connecting causal models to empirical observations.

Our contributions also highlight that intuition and common sense may not be the best arbiters for determining if and how causal models make falsifiable predictions about the effects of certain actions.
Therefore, we present many examples to illustrate our theoretical results and their formal implications.
Some examples may appear contrived and abstract,
but this is intentional:
They are designed to explain and highlight the principles and technical subtleties involved in establishing a formal relationship between causal models and real-world data-generating processes.

\subsection{Article outline}

Conceptually, the paper proceeds in four steps:
1) distinguish actions in the world from interventions in a model;
2) formalize interpretations connecting the two;
3) show that several appealing requirements jointly lead to circularity;
4) investigate non-circular alternatives and their implications.

\begin{description}
\item[\Cref{sec: setting} --
\textbf{Framework:}]
We introduce the framework used in this article. Instead of assuming that the underlying data-generating process is a causal model, we have a generic set of distributions indexed by actions. We draw an important distinction between 1) a model emulating the distributions of a representation and 2) a model being an interventionally valid model of a representation.

\item[\Cref{sec: circular} --
\textbf{A circular interpretation: }]
We formalize a circular interpretation, which uses the model's own notion of intervention to interpret actions as interventions. Under this interpretation,
every CBN that correctly models the observational distribution is interventionally valid.
This result provides a key ingredient
for the impossibility result in \Cref{sec: interpretations}. 

\item[\Cref{sec: interpretations} --
\textbf{An impossibility result for interpretations:}]
We discuss several desiderata for interpretations of actions as interventions and show that any interpretation satisfying D1–D4 is necessarily circular.

\item[\Cref{sec: non-circular} --
\textbf{Non-circular interpretations: }]
We define and discuss five non-circular interpretations of actions as interventions,
making it possible to derive falsifiable predictions about the effects of actions.

\item[\Cref{sec: related work} --
\textbf{Implications for related research: }]

We discuss implications of our work for causal representation learning and causal discovery. In the appendix, we discuss implications for causal abstraction and connections to the philosophical literature on the logic of conditionals.

\item[\Cref{sec:conclusion} -- \textbf{Conclusion}] 
\end{description}

\section{Framework}\label{sec: setting}

The framework distinguishes four ingredients:
(i) actions in the world and the distributions they induce,
(ii) a representation,
(iii) a causal model over the representation,
and (iv) an interpretation connecting actions in the world to interventions in the causal model.
The purpose of the framework is to make the connection between causal models and the world more explicit.

We choose to formalize causal models as causal Bayesian networks \citep{spirtes2001causation,Pearl2009,Peters2017} instead of,
for example, structural causal models (SCMs).
This choice eases the mathematical presentation in the present manuscript,
while we think that the considerations in this paper apply equally to SCMs.
The resulting presentation is somewhat more explicit
than the density-factorization-based treatment
common in introductions to Bayesian networks.
In particular,
we formulate graphical models using Markov kernels,
as in modern treatments of graphical models
\citep[see, for example][]{Lauritzen2019}.
In \Cref{app: notation}, we provide a notation overview.

\begin{definition}{\textbf{Causal Bayesian network.}}\label{def: CGM}
A \textit{causal Bayesian network} (CBN) $\mathfrak{C}$, also called a causal graphical model, 
over real-valued random variables $\rZ=
(Z_1,\dots, Z_n)$ is a directed acyclic 
graph (DAG) $\mathcal{G}$\footnote{Throughout, we assume 
acyclicity to ensure that the 
collection of Markov kernels induces a 
well-defined distribution under every 
intervention.} over nodes $[n]=\{1,\dots, n\}$\footnote{%
The nodes of a graph $\mathcal{G}$ are technically the natural numbers $[n]$, but to improve readability we also consider the corresponding coordinates of $\rZ=(Z_1,...,Z_n)$ to be the nodes of $\mathcal{G}$.}
 and a collection of 
Markov kernels $\{\pa_i\mapsto 
p_i^{\mathfrak{C}}(\cdot \mid \pa_i) 
\cdot \nu_i \mid i\in 
[n]\}$.\footnote{By $\PA_i$ we denote the 
variables of $\rZ$ whose indices 
correspond to the parents of $i$ in $\mathcal{G}$, and by 
$\pa_i\in \mathbb{R}^{|\PA_i|}$ we denote some value of these variables.
The conditional distribution $\mathcal{L}^{\mathfrak{C}}(Z_i\mid \PA_i=\pa_i)$ has a density $p_{i}^\mathfrak{C}(\cdot\mid\pa_i)$ w.r.t.\ the $\sigma$-finite measure $\nu_i$ on $\mathbb{R}$. The measure $\nu_i$ is fixed across all $\pa_i$.
} These Markov kernels induce a joint distribution,
called the observational distribution,
denoted by $\mathcal{L}^{\mathfrak{C}}(\rZ)$, with density
\begin{align*}
    p^{\mathfrak{C}}_\rZ(z_1,\dots, z_n)=\prod_{i=1}^n p_i^{\mathfrak{C}}(z_i\mid \pa_i).
\end{align*}
Given some nonempty subset $J\subseteq [n]$, interventional distributions are obtained by, for each $j\in J$, replacing the kernel of $j$ with some new kernel $\pa_j\mapsto q_{j}(\cdot \mid \pa_j)\cdot \mu_j$. We denote this intervention by $\doop(j \gets q_j,j\in J)$.%
\footnote{An intervention $\doop(j \leftarrow q_j,j\in J)$ is allowed to change both the densities $p_{j}(\cdot \mid \cdot)$ and the measures $\nu_j$, even though the change of measure is suppressed in the notation. Sometimes, we write $\doop(Z_j=z)$ and take this to mean that the kernel of $j$ is replaced with $\pa_j\mapsto \delta_{z}$, where $\delta_{z}$ is the Dirac distribution with support $\{z\}$. We also sometimes use notation like $\doop(Z_j\gets \mathcal{N}(0,1))$ and take this to mean that the kernel of $j$ is replaced with $\pa_j\mapsto \mathcal{N}(0,1)$, that is, $Z_j$ is set to follow a standard normal distribution and is made independent of its parents. 
}
The interventional distribution under intervention $\doop(j\gets q_j, j\in J)$ is denoted by $\mathcal{L}^{\mathfrak{C}; \doop(j\leftarrow q_j, j\in J)}(\rZ)$ and has density given by
\begin{align*}
    p_\rZ^{\mathfrak{C};\doop(j\leftarrow q_j, j\in J)}(\bm{z})
    =&\prod_{i=1}^n p_i^{\mathfrak{C};\doop(j\leftarrow q_j, j\in J)}(z_i \mid \pa_i) 
    \\
    =&\prod_{i\not \in J}p^{\mathfrak{C}}_{i}(z_i\mid \pa_i)\prod_{i\in J} q_{i}(z_i\mid \pa_i).
\end{align*}
$p_i^{\mathfrak{C};\doop(j\leftarrow q_j, j\in J)}$ denotes the $i$'th kernel given by $\mathfrak{C}$ and intervention $d=\doop(j\gets q_j, j\in J)$, that is, $p_i^{\mathfrak{C};\doop(j\leftarrow q_j, j\in J)}=p_i^{\mathfrak{C}}$ 
for $i\notin J$, and $p_i^{\mathfrak{C};\doop(j\leftarrow q_j, j\in J)}=q_i$ for $i\in J$.
In this work, we do not consider interventions that 
change the DAG $\mathcal{G}$.
We assume that interventions induce distributions different from the observational distribution, that is, they satisfy
    $\mathcal{L}^{\mathfrak{C}; \doop(j\leftarrow q_j, j\in J)}(\rZ)\neq \mathcal{L}^{\mathfrak{C}}(\rZ)$.
This, for example, rules out intervening only on source nodes\footnote{A node is a \textit{source node} if it has no parents.} without changing at least some of their marginal distributions.
We say a DAG $\mathcal{G}$ is \textit{complete}
if 
for every pair of distinct
nodes $Z_j$ and $Z_i$, either $Z_j\to Z_i$ or $Z_j\gets Z_i$;
we say that a CBN is $\textit{complete}$ if its DAG is complete.
\end{definition}

\begin{definition}{\textbf{Single-node, multi-node, minimal, decomposable, and perfect interventions.}}
We say that an intervention $\doop(j\leftarrow q_j,j\in J)$ is a \textit{single-node intervention} if $|J|=1$ and a \textit{multi-node intervention} if $|J|>1$. We say that an intervention $\doop(j\leftarrow q_j,j\in J)$
is \textit{minimal} if for every nonempty proper subset $J^*\subsetneq J$ and kernels $\{q_j^*\}_{j\in J^*}$, $$\mathcal{L}^{\mathfrak{C};\doop(j\leftarrow q_j,j\in J)}(\rZ)\neq\mathcal{L}^{\mathfrak{C};\doop(j\leftarrow q_j^*,j\in J^*)}(\rZ),$$ that is, if no intervention on a proper subset of the nodes induces the same interventional distribution.
We say that an intervention $\doop(j\leftarrow q_j,j\in J)$ is \textit{decomposable} if $\doop(j\leftarrow q_j,j\in J^*)$ is an intervention for every nonempty subset $J^*\subseteq J$, that is, if $\mathcal{L}^{\mathfrak{C};\doop(j\leftarrow q_j,j\in J^*)}(\rZ)\neq\mathcal{L}^{\mathfrak{C}}(\rZ)$ for every nonempty subset $J^*\subseteq J$. 
We say that a kernel $\pa_i\mapsto 
p_i^{\mathfrak{C}}(\cdot \mid \pa_i) 
\cdot \nu_i$ is \textit{perfect} if the measure $p_i^{\mathfrak{C}}(\cdot \mid \pa_i) 
\cdot \nu_i$ is identical for all $\pa_i$. We say that an intervention $\doop(j\leftarrow q_j,j\in J)$ is \textit{perfect} if 
$q_j$
is perfect for all $j\in J$. 
\end{definition}
One of the novelties of this work is that we do not assume that the data-generating process is
a causal Bayesian network
or
a structural causal model.
Instead, we remain agnostic as to how the data is generated and then ask if a representation of the data-generating process can be modeled by a causal Bayesian network. 

\begin{definition}{\textbf{Data-generating process.}}\label{def: dgp}
	A \textit{data-generating process} is a tuple $\mathcal{D}=(\mathcal{A},\{\mathcal{L}^a(\rX^*)\}_{a\in \mathcal{A}})$ where
	\begin{enumerate}
		\p $\mathcal{A}$ is a set of \textit{actions},
  \p one action $\mathcal{O}\in \mathcal{A}$ is called the \textit{observational regime}, and
		\p $\rX^*\in \mathbb{R}^m$ is a multivariate random variable that we call \textit{low-level features} 
        with distributions given by $\mathcal{L}^a(\rX^*)$ for each action $a\in \mathcal{A}$. 
	\end{enumerate}
    \vspace{-0.5cm}
\end{definition}
 In general, we assume that observed data consists of low-level features, such as image pixels, which may not be directly suited for causal modeling (see, for example,~\citet{scholkopf2021toward, wang2024desiderata}). Instead, we may aim to causally model a high-level representation given by functions of these pixels, rather than the individual pixels themselves.

 Unlike most works on causal representation learning, we do not focus on the challenges of learning or identifying such representations. 
Instead, we investigate the implications of
hypothesizing that a representation is modeled by a causal model.

\begin{definition} {\textbf{Representation.}}
        \label{def:representation}
	A \textit{representation} of a data-generating process $\mathcal{D}=(\mathcal{A},\{\mathcal{L}^a(\rX^*)\}_{a\in \mathcal{A}})$ is a 
    multivariate random variable %
    $\rZ^*=(Z_1^*,\dots, Z_n^*)=h(\rX^*)=(h_1(\rX^*),\dots, h_n(\rX^*))$ for measurable functions $h_1,\dots,h_n$.
\end{definition}
See \Cref{def: CRL} for a definition of `causal representation'.

\begin{figure*}
 \centering
 \begin{tikzpicture}
 \node at (0,0) {\includesvg[width=\textwidth]{drawing_no_mathfrak.svg}};
 \node[scale=2] at (-5.5,0.5) {$\mathfrak{C}$};
 \end{tikzpicture}

    \caption{
    The framework presented in this article has three main components. 1) The observed low-level features $\rX^*$, 2) a representation $\rZ^*:=h(\rX^*)$, and 3) a hypothesized causal model $\mathfrak{C}$ with variables $\rZ$. 
    }
    \label{fig: setting}
    
\end{figure*}

We will consider a correspondence between variables $\bm{Z}=(Z_1,...,Z_n)$ in a causal Bayesian network and a representation $\bm{Z}^*=(Z^*_1,...,Z^*_{n})$ of a data-generating process.\footnote{%
We mark variables whose distributions are derived from the data-generating process with a superscript~$*$ and the corresponding variables in a CBN without the superscript. %
}
The observational distribution of the representation $\rZ^*$ is given by the push-forward measure 
${\mathcal{L}^{\mathcal{O}}(\rZ^*)=h_{\sharp}(\mathcal{L}^\mathcal{O}(\rX^*))}$.
We say that a CBN $\mathfrak{C}$ over nodes $\rZ$ is \textit{compatible} with representation $\rZ^*$ if $\mathcal{L}^{\mathfrak{C}}(\rZ)=\mathcal{L}^{\mathcal{O}}(\rZ^*)$, which in particular implies that $\rZ$ and $\rZ^*$ have the same dimension. In addition to the observational distribution, for each  action $a\in \mathcal{A}$, we have a distribution $\mathcal{L}^a(\rZ^*)=h_{\sharp}(\mathcal{L}^a(\rX^*))$.
Likewise, the causal Bayesian network $\mathfrak{C}$ induces different interventional distributions $\mathcal{L}^{\mathfrak{C};\doop(j\leftarrow q_j,j\in J)}(\rZ)$ for different interventions $\doop(j\leftarrow q_j,j\in J)$. In \Cref{def: interventionally valid}, we make precise what it means for the CBN $\mathfrak{C}$ to be an interventionally valid model of a representation.
See \Cref{fig: setting} for a visual summary of the setting considered in this work.

Throughout this work, we assume that the distribution of $\rZ^*$ under action $a\in \mathcal{A}$, denoted by $\mathcal{L}^a(\rZ^*)$, has a density w.r.t.\ some product of $\sigma$-finite measures.

We write  $\mathcal{L}^a(Z_j^* \mid \rY^*)\sim p$ to mean that the kernel $\bm{y}^*\mapsto p(\cdot \mid \bm{y}^*)\cdot \nu$ is a regular conditional probability distribution of $Z_j^*$ given $\rY^* \subseteq \rZ^*$ under $\mathcal{L}^a(\rZ^*)$. 

\subsection{Emulation and interventional validity}

Since CBNs are convenient to describe multiple distributions,
we often use a CBN to describe the distributions of a representation.

\begin{definition}{\textbf{Representation emulated by a CBN.}}\label{def: generation}
    Let a data-generating process $\mathcal{D}=(\mathcal{A},\{\mathcal{L}^a(\rX^*)\}_{a\in \mathcal{A}})$ and a CBN $\mathfrak{A}$ over nodes $\rZ=(Z_1,\dots, Z_n)$ be given. We say that a representation $\rZ^*=(Z_1^*,\dots, Z_n^*)$ of $\mathcal{D}$ is \textit{emulated by $\mathfrak{A}$ and interventions $\mathcal{I}^*$} if $\mathcal{I}^*$ is a set of interventions in $\mathfrak{A}$ and there is a surjective function $g:\mathcal{A}\setminus \{\mathcal{O}\} \to \mathcal{I}^*$, referred to as the \textit{link}, such that  
    \begin{enumerate}
        \item  
        $\mathcal{L}^{\mathcal{O}}(\rZ^*)=\mathcal{L}^{\mathfrak{A}}(\rZ)$, and
    \item   
        $\mathcal{L}^{a}(\rZ^*)=\mathcal{L}^{\mathfrak{A};g(a)}(\rZ)$
     for all $a\in \mathcal{A}\setminus \{\mathcal{O}\}$.
    \end{enumerate}
    \vspace{-0.5cm}
\end{definition}

Importantly, stating that a representation is emulated by a CBN $\mathfrak{A}$ and interventions $\mathcal{I}^*$
implies no other claims
about the data-generating process than the distributions it induces for the set of actions.
Moreover, for every data-generating process $\mathcal{D}=(\mathcal{A},\{\mathcal{L}^a(\rX^*)\}_{a\in \mathcal{A}})$ and representation $\rZ^*$ such that $\mathcal{L}^a(\rZ^*)\neq \mathcal{L}^{\mathcal{O}}(\rZ^*)$ for all $a\in \mathcal{A}\setminus \{\mathcal{O}\}$, there exists a CBN $\mathfrak{A}$ and a set of interventions $\mathcal{I}^*$ in $\mathfrak{A}$ such that\
$\rZ^*$ is emulated by $\mathfrak{A}$ and $\mathcal{I}^*$.\footnote{
To see this, pick a complete CBN $\mathfrak{A}$ such that \ $\mathcal{L}^{\mathfrak{A}}(\rZ)=\mathcal{L}^{\mathcal{O}}(\rZ^*)$, and for each $a\in \mathcal{A}\setminus \{\mathcal{O}\}$, let $d^a$ be an intervention such that $\mathcal{L}^a(Z_i^*|\PA_i^*)\sim p_{i}^{\mathfrak{A};d^a}$ for all $i\in [n]$ (this is possible since $\mathcal{L}^a(\rZ^*)$ has density w.r.t.\ a product of $\sigma$-finite measures). Now $\mathcal{L}^a(\rZ^*)=\mathcal{L}^{\mathfrak{A};d^a}(\rZ)$ for every $a\in \mathcal{A}\setminus \{\mathcal{O}\}$  since $\mathfrak{A}$ is complete,
 so $\rZ^*$ is emulated by $\mathfrak{A}$ and $\mathcal{I}^*=\{d^a\mid a\in \mathcal{A}\setminus \{\mathcal{O}\}\}$. (Each $d^a$ is indeed an intervention since $\mathcal{L}^{\mathfrak{A};d^a}(\rZ)=\mathcal{L}^a(\rZ^*)\neq \mathcal{L}^{\mathcal{O}}(\rZ^*)=\mathcal{L}^{\mathfrak{A}}(\rZ)$, where the inequality holds by the assumed condition on $\rZ^*$.)
 }
 Thus, while we may imagine that the data-generating process can be any physical mechanism, every representation of such a mechanism whose distribution changes under every action can be emulated by a CBN  (and is in fact emulated by several different CBNs if there is more than one node).\footnote{This argument generalizes one of the main points made in \citet{eberhardt2016green}, namely that the same system can seemingly be described equally well by different causal models.}
We rely on \Cref{def: generation} to simplify our presentation (avoiding other mathematical descriptions of data-generating processes in terms of, for example, stochastic differential equations or exhaustive enumerations of the distributions for each action).

The variables in causal discovery and the latent variables in causal representation learning
are high-level features
and it may be unclear what constitutes an intervention on those variables.
In particular, interventions are usually ambiguous \citep{Spirtes2004ambiguous,rubenstein2017causal}.
Consider performing a phacoemulsification cataract surgery. Assume that we have a high-level feature that indicates whether
the surgery was performed. Interventions on this variable do not correspond to unique physical processes. The surgery may be performed at different times of the day, in different locations, by various different doctors using slightly different types of equipment, etc.
In particular, if we are given a representation $\rZ^*$ and a compatible CBN with variables $\rZ$,
whether or not the model correctly predicts the effects of actions
depends on which actions in the data-generating process correspond to which interventions in the CBN. 
Interpretations of actions formalize this correspondence
by assigning
subsets of the modeled interventions $\mathcal{I}$
to each action $a\in \mathcal{A}$.
This allows us to discuss the implications of different interpretations.

\begin{definition}{\textbf{Interpretations of actions.}}\label{def: interpretation}
    Let a data-generating process $\mathcal{D}$, representation $\rZ^*$, and a compatible CBN $\mathfrak{C}$ be given. An \textit{interpretation} $\mathbf{Int}$ is a mapping that takes an action $a\in \mathcal{A}$ and a set of modeled interventions $\mathcal{I}$ in $\mathfrak{C}$ as input and outputs a subset 
    of $\mathcal{I}$. For a given set of modeled interventions $\mathcal{I}$, an interpretation induces a function $\mathbf{Int}^{\mathcal{I}}: \mathcal{A}\to \mathcal{P}(\mathcal{I})$, where $\mathcal{P}(\mathcal{I})$ is the power set of $\mathcal{I}$.
    We require that the observational regime is not interpreted as an intervention, that is, $\mathbf{Int}^{\mathcal{I}}(\mathcal{O})=\emptyset$ for every set of modeled interventions $\mathcal{I}$.
\end{definition}

Notice that we do not necessarily map every action to an intervention, that is, $\int^{\mathcal{I}}(a)$ can be the empty set. Likewise, for an intervention $d\in \mathcal{I}$, there may be no action $a\in \mathcal{A}$ such that $d\in \int^{\mathcal{I}}(a)$.

Without committing to a specific interpretation,
it is unclear whether and how a CBN correctly predicts the distribution of a data-generating process's representation for a given action.
Consequently, it is then unclear what is causal about a causal model and what the criteria for falsifying it as a valid model of a representation should be.
Once
given an interpretation $\mathbf{Int}$, we can ask whether a CBN makes a correct prediction about the distribution of $\rZ^*$ under a given action. 
In particular, if we interpret action $a$ as  intervention $d$, that is, if $d \in \mathbf{Int}^{\mathcal{I}}(a)$, we can ask if $\mathcal{L}^a(\rZ^*)=\mathcal{L}^{\mathfrak{C};d}(\rZ)$.

\begin{definition}{\textbf{Interventional validity.}}\label{def: interventionally valid}
Let a data-generating process $\mathcal{D}$, representation $\rZ^*$, compatible CBN $\mathfrak{C}$, set of interventions $\mathcal{I}$ in $\mathfrak{C}$, and an interpretation $\mathbf{Int}$ be given. If for all $a\in \mathcal{A}$ and $d\in \mathcal{I}$
\begin{align*}
    d\in \mathbf{Int}^{\mathcal{I}}(a)\Rightarrow\mathcal{L}^{a}(\rZ^*)=\mathcal{L}^{\mathfrak{C}; d}(\rZ),
\end{align*}
we say that $\mathfrak{C}$ is an \textit{$\mathcal{I}-\mathbf{Int}$ valid model of $\rZ^*$}.
\end{definition}

\section{$\int_C:$ Interpreting actions via the model's own notion of intervention is circular}\label{sec: circular}

In a CBN,
an intervention $\doop(j\leftarrow q_j, j\in J)$ modifies the kernels of $Z_j$ given $\PA_j$ for $j\in J$, while keeping the kernels fixed for $j\notin J$. Since we do not consider interventions that
change the graph,
we have, for all interventions $d$, that $\mathcal{L}^{\mathfrak{C};d}(\rZ)$ is Markov w.r.t.\ the DAG $\mathcal{G}$ of the CBN $\mathfrak{C}$.
These considerations might compel us to consider the following interpretation.

\begin{definition}{\textbf{$\int_C$. Using the mathematical notion of an intervention as an interpretation.}
}\label{def: tautology}
Let a data-generating process $\mathcal{D}$, representation $\rZ^*$, compatible CBN $\mathfrak{C}$, and set of interventions $\mathcal{I}$ in $\mathfrak{C}$ be given. We define interpretation $\mathbf{Int}_C$ 
by the following rule: An intervention $\doop(j\gets q_j, j\in J)\in \mathcal{I}$ is in $\mathbf{Int}_C^{\mathcal{I}}(a)$ if and only if the following three conditions hold:
\begin{enumerate}
\item[1)] $\mathcal{L}^a(Z^*_i \mid \PA_i^*)\sim q_i$
for all $i\in J$. That is, the action sets the conditionals of intervened nodes correctly. For example, if we interpret an action as intervention $\doop(Z_i=4)$, then $Z_i^*$ must have Dirac distribution with support $\{4\}$ under that action.
\item[2)]  $\mathcal{L}^a(Z^*_i \mid \PA_i^*)\sim p^{\mathfrak{C}}_{i}$
for all $i\notin J$. Intuitively, we do not intervene on nodes not in $J$.
\item[3)] $\mathcal{L}^{a}(\rZ^*)$ is Markov w.r.t.\ the DAG of $\mathfrak{C}$. That is, we do not introduce dependencies.
\end{enumerate}
\vspace{-0.5cm}
\end{definition}

The $C$ in $\int_C$ is for `circular'. This interpretation interprets an action as an intervention $d$ in $\mathfrak{C}$ if and only if the action induces the interventional distribution given by $\mathfrak{C}$ and $d$.

\begin{restatable}{proposition}{propTautologyOne}\label{prop: tautology1}
    Let a data-generating process $\mathcal{D}$, representation $\rZ^*$, compatible CBN $\mathfrak{C}$, and set of interventions $\mathcal{I}$ in $\mathfrak{C}$ be given. Then $d\in \mathcal{I}$ is in $\mathbf{Int}^{\mathcal{I}}_C(a)$ if and only if $\mathcal{L}^{a}(\rZ^*)=\mathcal{L}^{\mathfrak{C};d}(\rZ)$.
\end{restatable}

All proofs are provided in \Cref{app: proofs}.
The following corollary is a formalization of a circularity that has been hinted at in previous works \citep{baumgartner2009interdefining,baumgartner2013identifying,woodwardSEP,janzing2024phenomenological}.\footnote{For example: ``If, as Pearl apparently intends, we understand this [the notion of intervention] to include the requirement that an intervention on $X_i$
 must leave intact the causal mechanism if any, that connects $X_i$
 to its possible effects $Y$, then an obvious worry about circularity arises [...]'' \citep{woodwardSEP}. And \citet{baumgartner2013identifying} argue that it cannot, in principle, be justified on interventionist grounds alone that a given variable is an intervention variable.}
\begin{corollary}{\textbf{$\int_C$ is circular.}}\label{prop: tautology2}
    Let a data-generating process $\mathcal{D}$ and representation $\rZ^*$ be given. Then every compatible CBN $\mathfrak{C}$ 
    is an $\mathcal{I}-\mathbf{Int}_C$ valid model of $\rZ^*$ for every set of interventions $\mathcal{I}$ in $\mathfrak{C}$.
\end{corollary}

\Cref{prop: tautology1} and \Cref{prop: tautology2} follow immediately from the definitions. Yet, they 
highlight the need for alternative interpretations of which actions constitute interventions since otherwise there is nothing `causal' about a causal Bayesian network: Observational and interventional validity are equivalent under interpretation $\int_C$
and interventional predictions do not help with falsification (in contrast to common assumptions, see also \Cref{fn:falsification}).
If we dropped condition 3) of \Cref{def: tautology}, 
then every compatible CBN with a complete DAG would still be interventionally valid.

\section{Impossibility result for non-circular interpretations}\label{sec: interpretations}

 Now that we have formalized the circular interpretation, we can provide an impossibility result for interpretations of actions as interventions.

We formalize five desiderata \textbf{D0}--\textbf{D4} for interpretations. 
These desiderata 
all appear intuitively reasonable.
Therefore, one might expect that a reasonable interpretation should satisfy all of them.
We show in \Cref{prop: impossibiliy} that if an interpretation satisfies \textbf{D1}--\textbf{D4}, then it is 
the circular interpretation $\mathbf{Int}_C$ (\Cref{def: tautology}), which renders all compatible models interventionally valid.

\paragraph*{\textbf{Desideratum} \textbf{D0}:}
\textbf{Correct conditionals on intervened nodes.}
If we interpret action $a$ as an intervention $\doop(j\leftarrow q_j, j\in J)$, then that action must set
the conditional distribution of intervened nodes given their parents correctly. 
Formally, an interpretation $\mathbf{Int}$ satisfies desideratum \textbf{D0} if
\begin{enumerate}
    \p For every set of modeled interventions $\mathcal{I}$ and every action $a\in \mathcal{A}$,\newline if $\doop(j\leftarrow q_j, j\in J)\in \mathbf{Int}^{\mathcal{I}}(a)$, then $\mathcal{L}^a(Z_i^*\mid \PA_i^*)\sim q_i$ for all $i\in J$. 
\end{enumerate}

 We believe that any reasonable interpretation satisfies \textbf{D0} and therefore do not consider interpretations that may violate \textbf{D0}.
  In the context of hard interventions, \textbf{D0} is sometimes referred to as `effectiveness' \citep{galles1998axiomatic,Bareinboim2022, Ibeling2023}. Effectiveness is similarly considered an axiom in \citet{park2023measure}.

\paragraph*{\textbf{Desideratum} \textbf{D1}:} \textbf{If it behaves like an intervention, it is that intervention.}
If an action $a$ induces a 
distribution that equals a distribution induced by the model under an intervention in the intervention set, then we should interpret that action as that intervention. Formally, an interpretation $\mathbf{Int}$ satisfies desideratum \textbf{D1} if
\begin{enumerate}
    \p For every set of modeled interventions $\mathcal{I}$, every action $a\in \mathcal{A}$, and every intervention $d\in\mathcal{I}$, 
    if $\mathcal{L}^a(\rZ^*)=\mathcal{L}^{\mathfrak{C};d}(\rZ)$, then $d\in \mathbf{Int}^{\mathcal{I}}(a)$.\footnote{If an interpretation $\mathbf{Int}$ additionally satisfies the reverse implication of \textbf{D1}, 
such that ``For every set of modeled interventions $\mathcal{I}$, actions $a\in \mathcal{A}$, and $d\in \mathcal{I}$,  $\mathcal{L}^a(\rZ^*)=\mathcal{L}^{\mathfrak{C};d}(\rZ)$ if and only if $d\in \mathbf{Int}^{\mathcal{I}}(a)$'',
then $\mathbf{Int}^{\mathcal{I}}(a)=\mathbf{Int}^{\mathcal{I}}_C(a)$ for every action $a\in \mathcal{A}$ and every set of interventions $\mathcal{I}$, rendering any compatible CBN $\mathcal{I}-\mathbf{Int}$ valid for every set of interventions $\mathcal{I}$ in that CBN.} 
    
\end{enumerate}

\paragraph*{\textbf{Desideratum} \textbf{D2}:} \textbf{An action should not be interpreted as distinct interventions.} If the action $a$ is interpreted as two distinct interventions, then these two interventions should induce the same interventional distribution.
Formally, an interpretation $\mathbf{Int}$ satisfies desideratum \textbf{D2} if
\begin{enumerate}
    \p For every set of modeled interventions $\mathcal{I}$,  every action $a\in \mathcal{A}$,
    and all interventions $b,d\in\mathcal{I}$, if $d\in \mathbf{Int}^{\mathcal{I}}(a)$ and $b\in \mathbf{Int}^{\mathcal{I}}(a)$, then $\mathcal{L}^{\mathfrak{C};d}(\rZ)=\mathcal{L}^{\mathfrak{C};b}(\rZ)$.
\end{enumerate}

\paragraph*{\textbf{Desideratum} \textbf{D3}:} \textbf{Interpretations should not depend on the intervention set~$\mathcal{I}$.} Whether we interpret an action $a$ as an intervention $d\in \mathcal{I}$ should not depend on which other interventions are in $\mathcal{I}$. Formally, an interpretation $\mathbf{Int}$ satisfies desideratum \textbf{D3} if
\begin{enumerate}

    \p For all sets of modeled interventions $\mathcal{I}$ and $\mathcal{I}'$, every action $a\in \mathcal{A}$, and every intervention $d\in\mathcal{I}\cap\mathcal{I}'$, $d \in \mathbf{Int}^{\mathcal{I}}(a) \Leftrightarrow d\in \mathbf{Int}^{\mathcal{I}'}(a)$.
\end{enumerate}

\paragraph*{\textbf{Desideratum} \textbf{D4}:} \textbf{An intervention does not create new dependencies.} If an action $a$ does not induce a distribution that is Markov w.r.t.\ the DAG, then we should not interpret $a$ as an intervention (in this work, as is common, we only consider interventions that do not introduce dependencies between variables).
 Formally, an interpretation $\mathbf{Int}$ satisfies desideratum \textbf{D4} if
\begin{enumerate}
    \p For every set of modeled interventions $\mathcal{I}$ and every action $a\in \mathcal{A}$, if $\mathcal{L}^a(\rZ^*)$ is not Markov w.r.t.\ $\mathcal{G}$, then $\mathbf{Int}^{\mathcal{I}}(a)=\emptyset$. 
\end{enumerate}
\begin{restatable}{proposition}{propImpossibility}{\textbf{Impossibility result.}}\label{prop: impossibiliy}
    Let a data-generating process $\mathcal{D}$, representation $\rZ^*$, and a compatible CBN $\mathfrak{C}$ be given. Let $\mathbf{Int}$ be an interpretation that satisfies desiderata \textbf{D1}--\textbf{D4}. 
    Then, for every set of modeled interventions $\mathcal{I}$ in $\mathfrak{C}$ and for all actions $a\in \mathcal{A}$ with $\mathcal{L}^a(\rZ^*)\neq \mathcal{L}^{\mathcal{O}}(\rZ^*)$, $\mathbf{Int}^{\mathcal{I}}(a)= \mathbf{Int}_C^{\mathcal{I}}(a)$.
\end{restatable}
$\mathbf{Int}_C$ satisfies \textbf{D0}, so desiderata \textbf{D1}--\textbf{D4} together imply \textbf{D0} for all actions that change the distribution of the representation.\footnote{The restriction to actions $a$ with $\mathcal{L}^a(\rZ^*)\neq \mathcal{L}^{\mathcal{O}}(\rZ^*)$ cannot be dropped: if an action induces the observational distribution, then no intervention induces $\mathcal{L}^a(\rZ^*)$, so \textbf{D1} places no constraint on $\mathbf{Int}^{\mathcal{I}}(a)$, and an interpretation satisfying \textbf{D1}--\textbf{D4} may interpret $a$ as an arbitrary intervention, whereas $\mathbf{Int}_C^{\mathcal{I}}(a)=\emptyset$. Intuitively, the desiderata cannot regulate what the data cannot distinguish from the observational regime.}

\section{Non-circular interpretations}\label{sec: non-circular}

We now consider, in turn, possible interpretations that may
violate either one of the desiderata \textbf{D1} and \textbf{D2} to avoid the circularity of \Cref{prop: tautology2} implied by satisfying all desiderata. In \Cref{app: intervention set}  and \Cref{app: markov}, we consider interpretations that may violate \textbf{D3} and \textbf{D4}, respectively.
Taken together, this shows that no proper subset of the desiderata \textbf{D1}--\textbf{D4} implies any other of the desiderata \textbf{D1}--\textbf{D4}, so the desiderata can be considered separately.
In \Cref{sec: complexity}, we consider an interpretation that takes action complexity into account. See \Cref{sec: interpretations overview} for an overview of all interpretations considered in this paper.

\subsection{$\int_P$: Letting imperfect interventions falsify a model is one way out of circularity}\label{sec: intA}

\paragraph*{Informal overview of \Cref{sec: intA}.} If we insist that all actions correspond to perfect interventions, then it becomes possible to falsify a causal model: If we perform an action and the resulting observation cannot be explained by a perfect intervention, the causal model can be rejected. 
Explaining these observations by imperfect soft interventions may be a slippery slope leading to a circular interpretation where every observed distribution can be explained by some complex intervention in the model. In \Cref{ex: TC}, we show how to 
falsify a causal model under $\int_P$ (to be defined in \Cref{def: perfect}). 

Consider the following interpretation that may violate \textbf{D1},
but satisfies \textbf{D0} and \textbf{D2}--\textbf{D4}. 

\begin{definition}{\textbf{$\int_P$. An interpretation violating only D1.}}\label{def: perfect}
Let a data-generating process $\mathcal{D}$, representation $\rZ^
*$, compatible CBN $\mathfrak{C}$, and set of interventions $\mathcal{I}$ in $\mathfrak{C}$ be given. 
We define the interpretation $\mathbf{Int}_P$ by the following rule:
    An intervention $d=\doop(j \leftarrow q_j, j\in J)\in \mathcal{I}$ is in $\int_P^{\mathcal{I}}(a)$ if and only if the following four conditions hold:
    \begin{enumerate}
        \item[1)] $d$ is a perfect intervention.
        \item[2)] For all $i\in J$, $\mathcal{L}^a(Z_i^*\mid \PA_i^*)\sim q_i.$
        {That is, $\int_P$ satisfies \textbf{D0} (correct conditionals on intervened nodes).}
        \item[3)]  For all $i\notin J$,
        \begin{align*}
               & \text{$\PA^*_i$ is empty and } \mathcal{L}^a(Z_i^*)=\mathcal{L}^{\mathfrak{C}}(Z_i) \text{, or}\\
               & \text{$\PA^*_i$ is nonempty and }Z_i^* \nindep \PA_i^* \text{ in } \mathcal{L}^a(\rZ^*).
        \end{align*}
        That is, nodes not intervened on are either source nodes with unchanged distributions or not independent of their parents.

         \item[4)] $\mathcal{L}^a(\rZ^*)$ is Markov w.r.t.\ the DAG of $\mathfrak{C}$.
         {That is, $\int_P$ satisfies \textbf{D4} (an intervention does not create new dependencies)}.\footnote{%
         \Cref{prop: intA} would still hold if we omitted condition 4) of \Cref{def: perfect}.
         }
    \end{enumerate}
        \vspace{-0.5cm}
\end{definition}

The $P$ in $\int_P$ is for `perfect'. 
Condition 3) ensures that $\int_P$ satisfies 
\textbf{D2} (an action should not be 
interpreted as distinct interventions), see \Cref{app: intP sat D2}.
{It is straightforward to verify that $\int_P$ satisfies \textbf{D3} (interpretations should not depend on the intervention set $\mathcal{I}$)}.  $\int_P$ may violate \textbf{D1} (if it behaves like an intervention, it is that intervention) because it may be that $\mathcal{L}^a(\rZ^*)=\mathcal{L}^{\mathfrak{C};d}(\rZ)$ for some $d\in \mathcal{I}$ that is not a perfect intervention and some $a\in\mathcal{A}$, and thus $d\notin \int_P^{\mathcal{I}}(a)$. Under a non-circular interpretation like $\int_P$, a CBN $\mathfrak{C}$ can be an invalid model of a representation even though the representation is emulated by $\mathfrak{C}$. We now provide a partial characterization for when this happens under interpretation $\int_P$.

\begin{restatable}{proposition}{propIntP}\label{prop: intA}
    Let a data-generating process $\mathcal{D}$ be given. Assume that $\rZ^*$ is emulated by CBN $\mathfrak{C}$ and interventions $\mathcal{I}^*$. 

    (1) If $\mathcal{I}^*$ only contains perfect interventions, then $\mathfrak{C}$ is an $\mathcal{I}-\int_P$ valid model of $\rZ^*$
    for every set of interventions $\mathcal{I}$ in $\mathfrak{C}$. 

    (2) On the other hand, if $\mathcal{I}^*$ contains a minimal and decomposable intervention $\doop(j\leftarrow q_j, j\in J)$ for which there exist $s,t\in J$ such that $Z_s \indep \PA_s$ and $Z_t \nindep \PA_t$ in $\mathcal{L}^{\mathfrak{C};\doop(j\leftarrow q_j, j\in J)}(\rZ)$, 
     then there exists a set of interventions $\mathcal{I}$ in $\mathfrak{C}$ such that $\mathfrak{C}$ is not an $\mathcal{I}-\int_P$ valid model of $\rZ^*$.  
\end{restatable}

If there exists an action that induces a distribution like an intervention in \Cref{prop: intA}~(2), that is, a perfect intervention on some nodes and an imperfect intervention on others,
then
the model can be falsified.
Since for a given representation $\rZ^*$, it may be difficult to rule out such actions,
\Cref{prop: intA} (2) highlights a limitation of causal modeling using $\int_P$.
We now show an example where we falsify a CBN under interpretation $\int_P$.

\begin{example}{\textbf{Falsifying a total cholesterol model under interpretation $\int_P$.}}\label{ex: TC}
\begin{figure}
    \begin{subfigure}[b]{0.45\textwidth}
        \centering
    \begin{tikzpicture}
     \node at (-1.2,0) {\scalebox{1.5}{$\mathfrak{A}$:}};
    \node[draw, circle,minimum size=1.2cm] (LDL) at (0,0) {$\text{LDL}$};
    \node[draw, circle,minimum size=1.2cm] (HDL) at (4,0) {$\text{HDL}$};
    \node[draw, circle,minimum size=1.2cm] (HD) at (2,0) {$\text{HD}$};
    \draw[->,line width=1.5pt] (LDL) -- (HD);
    \draw[->,line width=1.5pt] (HDL) -- (HD);
    \end{tikzpicture}
        \caption{Low-level representation}
        \label{fig:subfigure1}
    \end{subfigure}
    \hfill
    \begin{subfigure}[b]{0.45\textwidth}
        \centering

    \begin{tikzpicture}
    \node at (3.8,0.5) {\scalebox{1.5}{$\mathfrak{C}$:}};
    \node at (3.8,0.5) {\scalebox{2}{}};
    \node[draw, circle,minimum size=1.2cm] (TC) at (5,0.5) {TC};
    \node[draw, circle,minimum size=1.2cm] (HD2) at (7,0.5) {HD};
    \draw[->,line width=1.5pt] (TC) -- (HD2);
    \end{tikzpicture}
        \caption{High-level representation}
        \label{fig:subfigure2}
    \end{subfigure}
   \caption{In \Cref{ex: TC}, we assume that the low-level representation $(\text{LDL}^*,\text{HDL}^*,\text{HD}^*)$ is emulated by a CBN $\mathfrak{A}$, with graph given in \Cref{fig:subfigure1}, and perfect interventions in~$\mathfrak{A}$. Under $\int_P$, we falsify the $(\text{TC}^*,\text{HD}^*):=(\text{LDL}^*+\text{HDL}^*,\text{HD}^*)$-compatible CBN~$\mathfrak{C}$ with graph given in \Cref{fig:subfigure2}.
        }
\end{figure}
 We consider the causal effect of total cholesterol on heart disease \citep{Spirtes2004ambiguous}. Throughout this example, we use the shorthands LDL for low-density lipoprotein, HDL for high-density lipoprotein, HD for heart disease, and TC for total cholesterol.
For simplicity, we use stylized distributions rather than physiologically realistic ones (for example, the variables may take negative values, and heart disease is modeled as a continuous score).
Let $\mathfrak{A}$ be a CBN given by the graph depicted in \Cref{fig:subfigure1} and kernels
\begin{align*}
    &\mathcal{L}^{\mathfrak{A}}(\text{LDL})= \mathcal{N}(0,1), \hspace{0.5cm}
    \mathcal{L}^{\mathfrak{A}}(\text{HDL})= \mathcal{N}(0,1),\\
    &\mathcal{L}^{\mathfrak{A}}(\text{HD} \mid \text{LDL}=x,\text{HDL}=y)= \mathcal{N}(2x-y,1).
\end{align*}
Assume that $\rX^*=(\text{LDL}^*,\text{HDL}^*,\text{HD}^*)$ is emulated by $\mathfrak{A}$ and perfect interventions $$\mathcal{I}^*=\{\doop(\text{LDL}\leftarrow \mathcal{N}(x,1),\text{HDL}\leftarrow \mathcal{N}(y,1))\mid x,y\in \mathbb{R},\ (x,y)\neq (0,0)\}.$$ %
Consider the representation given by
$\rZ^*=(\text{TC}^*,\text{HD}^*)=(\text{LDL}^*+\text{HDL}^*,\text{HD}^*)$. Let $\mathfrak{C}$ be the $\rZ^*$-compatible CBN with graph given in \Cref{fig:subfigure2} and kernels given by
\begin{align*}
    \mathcal{L}^{\mathfrak{C}}(\text{TC})=\mathcal{N}(0,2), \hspace{0.5cm}
    \mathcal{L}^{\mathfrak{C}}(\text{HD}\mid \text{TC}=t) = \mathcal{N}\left(\frac{t}{2},\frac{11}{2}\right).
\end{align*}
We consider perfect shift interventions that change the mean of total cholesterol $\mathcal{I}=\{\doop(\text{TC}\leftarrow \mathcal{N}(t,2))\mid t\in \mathbb{R}\setminus \{0\}\}$.
$\mathfrak{C}$ is not an 
$\mathcal{I}-\int_P$ valid
model of $\rZ^*$. 
To see this, consider an action $a'$ with $\mathcal{L}^{a'}(\text{TC}^*,\text{HD}^*)=\mathcal{L}^{\mathfrak{A};\doop(\text{LDL}\leftarrow \mathcal{N}(1,1),\text{HDL}\leftarrow \mathcal{N}(0,1))}(\text{LDL}+\text{HDL},\text{HD})$ (such an action exists by definition of $\mathcal{I}^*$).
We have that $\doop(\text{TC}\leftarrow \mathcal{N}(1,2))\in \int_P^{\mathcal{I}}(a')$ 
since $\mathcal{L}^{a'}(\text{TC}^*)=\mathcal{N}(1,2)$ and $\text{HD}^* {\nindep} \text{TC}^*$ in $\mathcal{L}^{a'}(\text{TC}^*,\text{HD}^*)$.
But $\mathcal{L}^{a'}(\text{TC}^*,\text{HD}^*) \neq \mathcal{L}^{\mathfrak{C};\doop(\text{TC}\leftarrow \mathcal{N}(1,2))}(\text{TC},\text{HD})$, 
for example, because the expected value $\mathbb{E}^{a'}(\text{HD}^*)=2$ is not equal to the expected value $\mathbb{E}^{\mathfrak{C};\doop(\text{TC}\leftarrow \mathcal{N}(1,2))}(\text{HD})=\frac{1}{2}$.
This falsifies CBN $\mathfrak{C}$ as an $\mathcal{I}-\int_P$ valid model of $(\text{TC}^*,\text{HD}^*)$, and $a'$ can be thought of as the falsifying experiment. Note that to falsify $\mathfrak{C}$, one need not have any knowledge about how total cholesterol is constituted of low-density lipoprotein and high-density lipoprotein. The model is falsified purely based on the distribution $(\text{TC}^*,\text{HD}^*)$ under action $a'$. In \Cref{app: total choles int C}, we argue that we cannot falsify $\mathfrak{C}$ under $\int_C$.
\end{example}

\subsection{$\int_S$: Letting multi-node interventions falsify a model is another way out of circularity}\label{sec: intB}

\paragraph*{Informal overview of \Cref{sec: intB}.}
If we insist that all actions correspond to single-node interventions, then it becomes possible to falsify a causal model:  If we perform an action and the resulting distribution cannot be explained by a single-node intervention, the causal model can be rejected. Explaining these observations by multi-node interventions may be a slippery slope leading to a circular interpretation where every observed distribution can be explained by some complex intervention in the model. In \Cref{ex: opposite direction}, we show how to falsify a causal model under $\int_S$ (to be defined in \Cref{def: single node}).
Consider the following interpretation that may violate \textbf{D2},
but satisfies \textbf{D0}, \textbf{D1}, \textbf{D3}, and \textbf{D4}.

\begin{definition}{\textbf{$\int_S$. An interpretation violating only D2.}}\label{def: single node}
        Let a data-generating process $\mathcal{D}$, representation $\rZ^*$, compatible CBN $\mathfrak{C}$, and set of interventions $\mathcal{I}$ in $\mathfrak{C}$ be given. We define the interpretation $\mathbf{Int}_S$ by the following rule: An intervention $\doop(j\leftarrow q_j,j\in J)\in \mathcal{I}$ is in $\int_S^{\mathcal{I}}(a)$ if and only if the following three conditions hold:
	\begin{enumerate}
            \item[1)] 	$\mathcal{L}^a(Z^*_i \mid \PA_i^*)\sim q_i$ for all $i\in J$.
            That is, $\int_S$ satisfies \textbf{D0} (correct conditionals on intervened nodes).
            \item[2)]  $\mathcal{L}^a(Z^*_i \mid \PA_i^*)\nsim p_i^{\mathfrak{C}}$ for all $i\in J$. 
            That is, the kernels of the observational distribution are incompatible with the conditionals of intervened nodes under action $a$.\footnote{\label{fn: counterexamples}\Cref{prop: d2} (1) would not hold if we omitted condition 2) of \Cref{def: single node}, see \Cref{app: counterexample footnote} for a counterexample.
            \Cref{prop: d2} (2) would still hold if we omitted condition 2).}
		\item[3)] $\mathcal{L}^a(\rZ^*)$ is Markov w.r.t.\ the DAG of $\mathfrak{C}$. That is, $\int_S$ satisfies \textbf{D4} (an intervention does not create new dependencies).\footnote{%
        \Cref{prop: d2} would still hold if we omitted condition 3) of \Cref{def: single node}.
        }
	\end{enumerate} 
    \vspace{-0.5cm}
\end{definition}
The $S$ in $\int_S$ is for `single-node'. $\int_S$ may violate \textbf{D2}, that is, an action can be interpreted as distinct interventions,
because if $\doop(j\leftarrow q_j,j\in J)\in \int_S^{\mathcal{I}}(a)$ is minimal and decomposable, then $\doop(j\leftarrow q_j,j\in J^*)\in \int_S^{\mathcal{I}}(a)$ for every nonempty subset $J^*\subsetneq J$ such that $\doop(j\leftarrow q_j,j\in J^*)\in \mathcal{I}$, but also $\mathcal{L}^{\mathfrak{C};\doop(j\leftarrow q_j,j\in J)}(\rZ)\neq \mathcal{L}^{\mathfrak{C};\doop(j\leftarrow q_j,j\in J^*)}(\rZ)$.

\begin{restatable}{proposition}{propIntS}\label{prop: d2}
        Let a data-generating process $\mathcal{D}$ be given. Assume that $\rZ^*$ is emulated by CBN $\mathfrak{C}$ and  interventions $\mathcal{I}^*$.
        
        (1) If $\mathcal{I}^*$ only contains single-node interventions, then $\mathfrak{C}$ is an $\mathcal{I}-\int_S$  valid model of $\rZ^*$ for every set of interventions $\mathcal{I}$ in $\mathfrak{C}$. 

        (2) On the other hand, if $\mathcal{I}^*$ contains a minimal and decomposable multi-node intervention, then there exists a set of interventions $\mathcal{I}$ in $\mathfrak{C}$ such that $\mathfrak{C}$ is not an $\mathcal{I}-\int_S$ valid model of $\rZ^*$. 
\end{restatable}
If there exists an action that induces a distribution like a multi-node intervention, then the model can be falsified. Since for a given representation $\rZ^*$, it may be difficult to rule out such actions, \Cref{prop: d2} (2) highlights a limitation of causal modeling using $\int_S$.

In \Cref{ex: TC}, we could have falsified the model $\mathfrak{C}$ under interpretation $\int_S$, like we did under interpretation $\int_P$.
In \Cref{app: ex opposite direction}, we present a more extreme example where $\rZ^*$ is emulated by a CBN $\mathfrak{A}$ and multi-node interventions $\mathcal{I}^*$ in a way that is fine-tuned to make the opposite causal direction interventionally valid. 

\subsection{$\int_K$: Interventions as simple actions}\label{sec: complexity}

The larger the set of actions, the more conditions must be satisfied for a causal model to be interventionally valid.
More formally,
there exist interpretations $\int$ such that
if $\mathfrak{C}$ is an
$\mathcal{I}-\int$
valid model of $\rZ^*$ with set of actions $\mathcal{A}$,
then $\mathfrak{C}$ may not be an $\mathcal{I}-\int$
valid model of $\rZ^*$ with set of actions $\mathcal{A}'\supsetneq \mathcal{A}$.
For example, 
under $\int_P$,
the model $\mathfrak{C}$ in \Cref{ex: TC} is not interventionally valid
but it would be for a sufficiently small subset of $\mathcal{A}$; and every compatible model is interventionally valid if $\mathcal{A}=\{\mathcal{O}\}$.
More generally, \Cref{prop: d2} (2) and \Cref{prop: intA} (2)
show that
while
$\int_S$ and $\int_P$
avoid the circularity of interpretation $\int_C$ (\Cref{prop: tautology2} and \Cref{prop: impossibiliy}),
these interpretations may prevent interventionally valid causal modeling of sensible representations if the set of actions is large.

\paragraph*{Informal overview of \Cref{sec: complexity}.}
In this section, we present an interpretation $\int_K$
which has intermediate restrictiveness
between
$\int_S$ and $\int_C$:
$\mathcal{I}-\int_S$ validity implies $\mathcal{I}-\int_K$ validity
and, since every compatible CBN is $\mathcal{I}-\int_C$ valid, $\mathcal{I}-\int_K$ validity implies $\mathcal{I}-\int_C$ validity
(while the reverse implications do not hold).
We accomplish this by
considering the complexity of actions and 
disqualifying actions that are not the simplest implementations of an intervention.
The following thought experiment motivates
why the complexity of actions is relevant to deciding which actions should be considered as which interventions.

\paragraph*{Thought experiment.}  Suppose I am leading a sedentary lifestyle and am considering taking up smoking during my vacation.
I have the following question: `How would smoking one pack of cigarettes over a week affect my heart health?'
What kind of experiment could be relevant to answer this question?
Following \citet{Dawid2021} and setting ethical issues aside, one approach might be the following study:
Pay participants who are similar to me (including having a sedentary lifestyle) to smoke a pack of cigarettes over a week and then measure an indicator of heart health before and after.
However, if participants started a rigorous exercise routine to offset the negative effects of smoking, the study would no longer capture how smoking alone impacts heart health; rather, the before-after measurements would reflect the combined effects of smoking and exercising on heart health.
Naively, we want the participants to take up smoking while keeping everything else fixed.
This is both a) impossible and b) undesirable. 
a) It is impossible because taking up smoking will necessarily affect other things as well. As \citet{Lewis1973} explains ``If we try too hard for exact similarity to the actual world in one respect, we will get excessive differences in some other respect.''\footnote{\citet[page 9]{Lewis1973} puts it like this, considering what would happen if kangaroos had no tails:
    ``We might think it best to confine our attention to worlds where
kangaroos have no tails and everything else is as it actually is; but there
are no such worlds. Are we to suppose that kangaroos have no tails
but that their tracks in the sand are as they actually are? Then we
shall have to suppose that these tracks are produced in a way quite
different from the actual way. Are we to suppose that kangaroos have
no tails but that their genetic makeup is as it actually is? Then we shall
have to suppose that genes control growth in a way quite different from
the actual way (or else that there is something, unlike anything there
actually is, that removes the tails). And so it goes; respects of similarity
and difference trade-off. If we try too hard for exact similarity to the
actual world in one respect, we will get excessive differences in some
other respect.''}
In our case, smoking will inevitably impact other aspects of participants' lives such as taking smoking breaks or carrying a lighter.
b) It is undesirable because we do not want to keep heart health fixed.
The problem seems to be that taking up smoking and exercising is not the simplest modification of the participants' daily routine that involves taking up smoking, and this might disqualify the participants' behavior as being an intervention on smoking.

Another way to understand the problem above is that the study participants are implementing interventions in a way that is counter to a commonsense consensus on how interventions should be implemented.
As pointed out in \citet[page 121]{Peters2017}, the notion of falsification ``includes the assumption that there is an agreement about what a randomized experiment should look like''.
In this work, we attempt to move beyond human intuition toward mathematical precision.
We do not think that human intuition is never good enough to do useful causal modeling.
 However, we believe that in many circumstances, such as those encountered in causal representation learning and causal abstraction, it is entirely unclear how human intuition could serve as a secure foundation, and we doubt that there is an implicit commonsense consensus on what constitutes a valid action to implement some intervention.
This doubt becomes more pressing when the two partners in the dialogue of \Cref{sec: dialogue} are not both human. Increasingly, a large language model may propose which experiments to run, a human scientist carries them out and reports the results back, and the language model judges whether a causal model has been falsified. Whatever implicit folklore human scientists may share about what it means to `intervene on $B$ only', there is less reason to expect that it is shared across this divide, and the link between actions and interventions then needs to be made explicit and verifiable rather than left to a presumed consensus.

We now provide an interpretation that is based on a given notion of complexity of actions.
For this, assume that there is some complexity measure $K: \mathcal{A}\to (0,\infty]$ that
assigns a positive extended real number to each action.

\begin{definition}{\textbf{$\int_K$. An interpretation violating D1 and D2.}} \label{def: complexity}
    Let a data-generating process $\mathcal{D}$, representation $\rZ^*$, compatible CBN $\mathfrak{C}$, set of interventions $\mathcal{I}$ in $\mathfrak{C}$, and complexity measure $K: \mathcal{A}\to (0,\infty]$ be given. We define the interpretation $\mathbf{Int}_K$ by the following rule: An intervention $d\in \mathcal{I}$ is in  $\int_K^{\mathcal{I}}(a)$ if and only if $a\in \underset{a'\in \mathcal{A}: d\in \int_S^{\mathcal{I}}(a')}{\arg \min}K(a')$. 
\end{definition}
Under this interpretation, the complexity measure of actions affects which actions are interpreted as implementations of which interventions and thereby whether a model is interventionally valid. 
If a representation $\rZ^*$ is emulated by $\mathfrak{C}$ and $\mathcal{I}^*$ with link $g: \mathcal{A}\setminus \{\mathcal{O}\}\to \mathcal{I}^*$,
we show that $\mathfrak{C}$ is an $\mathcal{I}-\int_K$ valid model of $\rZ^*$ for every set $\mathcal{I}\subseteq \mathcal{I}^*$ of minimal interventions
if the complexity measure $K$ is strictly increasing in the number of nodes intervened upon by $g(a)$ (and some regularity conditions are met). We use $|g(a)|$ to denote the number of nodes intervened upon by $g(a)$ in $\mathfrak{C}$.\footnote{%
Formally, if $\rZ^*$ is emulated by $\mathfrak{C}$ and $\mathcal{I}^*$ with link $g:\mathcal{A}\setminus\{\mathcal{O}\}\to\mathcal{I}^*$, we define $|g(a)| = |J|$ for $g(a) = \doop(j\gets q_j,j\in J)$.
}
The complexity of actions ought to be given by some consideration that is external to the model;
defining the complexity of actions in terms of $|g(a)|$ would be circular since $|g(a)|$ depends on the CBN $\mathfrak{C}$.

\begin{restatable}{proposition}{propComplexity}\label{prop: complexity}
   Let a data-generating process $\mathcal{D}$ be given. Assume that $\rZ^*$ is emulated by CBN $\mathfrak{C}$ and interventions $\mathcal{I}^*$ 
   with link $g: \mathcal{A}\setminus \{\mathcal{O}\}\to \mathcal{I}^*$.
   Assume that
    \begin{enumerate}
        \item  $K(a)=t(|g(a)|)$ for some strictly increasing $t: \mathbb{R}\to \mathbb{R}$. 
        \item  $\mathcal{L}^a(\rZ^*)$ has the same finite support for every $a\in \mathcal{A}$.
        \item  Every $d\in \mathcal{I}$ is minimal and  $\mathcal{I}\subseteq \mathcal{I}^*$.
    \end{enumerate}
    Then $\mathfrak{C}$ is an $\mathcal{I}-\int_K$ valid model of $\rZ^*$. 
\end{restatable}
We now provide a simple example with two different complexity measures and investigate how they affect the interventional validity of a model. 
\begin{example}{\textbf{Making an $\int_S$ invalid model valid by specifying a complexity measure.}}\label{ex: RL}
    Consider CBN $\mathfrak{A}$ with graph $P\to R$ and kernels
    \begin{align*}
        P&\sim \text{Unif}([6]),\hspace{0.5cm}
        R:=\begin{cases}
        +1 & P \in\{1, 3, 5\}\\
        -1 & P \in\{2, 4, 6\}
        \end{cases}.
    \end{align*}

\begin{figure}
        \centering
    \begin{tikzpicture}
    \draw[step=1cm,gray,very thin] (-3,-3) grid (2,0);
    \node at (-0.5,-1.5) {\includesvg[width=0.7cm]{stick_figure.svg}};

    \draw[thick, ->] (-1,0.2) -- (-3,0.2) node[midway, above] {$S^*=0$};
    \draw[thick, ->] (0,0.2) -- (2,0.2) node[midway, above] {$S^*=1$};
    \node at (-2.5,-0.5) {1};
    \draw[blue,thick] (-2.5,-0.5) circle (0.3cm);
    \node[left] at (-3.3,-0.5) {$R^*=+1$};
    \draw[blue,thick] (-2.5,-1.5) circle (0.3cm);
    \node at (-2.5,-1.5) {2};
    \node[left] at (-3.3,-1.5) {$R^*=-1$};
    \draw[blue,thick] (-2.5,-2.5) circle (0.3cm);
    \node at (-2.5,-2.5) {3};
     \node[left] at (-3.3,-2.5) {$R^*=+1$};
    \draw[blue,thick] (1.5,-0.5) circle (0.3cm);
    \node at (1.5,-0.5) {4};
    \node[right] at (2.3,-0.5) {$R^*=-1$};
    \draw[blue,thick] (1.5,-1.5) circle (0.3cm);
    \node at (1.5,-1.5) {5};
     \node[right] at (2.3,-1.5) {$R^*=+1$};
    \draw[blue,thick] (1.5,-2.5) circle (0.3cm);
    \node at (1.5,-2.5) {6};
     \node[right] at (2.3,-2.5) {$R^*=-1$};
     \node [draw, circle, thick, blue, align=left,text=black] at(-0.6,0.5) (p) {};
       \node at(-0.1,0.5)  {$P^*$};
\end{tikzpicture}
        \caption{Depiction of the situation in \Cref{ex: RL}. There are $6$ different locations,
        three on each of the left and right sides $S^*$,
        with associated rewards $R^*$ ($+1$ or $-1$). Observationally, the stick figure picks a position $P^*$
        using a uniform distribution over the $6$ locations. Whether the causal model $S\to R$ makes correct predictions about interventions on left/right, depends on which specific actions are interpreted as interventions.
        }
        \label{fig: RL}
    \end{figure}

    Assume that $(P^*,R^*)$ is emulated by CBN $\mathfrak{A}$ and intervention set
    \begin{align*}
    \mathcal{I}^*=\{d\mid &d \text{ is a single-node intervention on $P$} \\
    &\text{such that $\mathbb{P}^{\mathfrak{A};d}(P\in [6])=1$ }\}.
    \end{align*}
    $P^*$ stands for position and $R^*$ for reward,  see \Cref{fig: RL}. Consider the feature
    left/right given by $S^*=\mathbbm{1}(P^*\in \{4,5,6\})$, which is $1$ if the position is on the right and $0$ if the position is on the left.
    We now investigate if $\mathfrak{C}$, given by $S\to R$ and kernels
\begin{align*}
    &\mathbb{P}^{\mathfrak{C}}(S=0) = \tfrac{1}{2},
    \quad
    \mathbb{P}^{\mathfrak{C}}(S=1) = \tfrac{1}{2},\\
    &\mathbb{P}^{\mathfrak{C}}(R = r \mid S = 0) = 
    \begin{cases}
        \tfrac{2}{3} & \text{if } r = +1 \\
        \tfrac{1}{3} & \text{if } r = -1
    \end{cases}, \hspace{0.5cm}\\
    &\mathbb{P}^{\mathfrak{C}}(R = r \mid S = 1) = 
    \begin{cases}
        \tfrac{1}{3} & \text{if } r = +1 \\
        \tfrac{2}{3} & \text{if } r = -1
    \end{cases}
\end{align*}
    is an interventionally valid model of $(S^*, R^*)$ under interventions $\mathcal{I}=\{\doop(S=0),\doop(S=1)\}$. Notice that we have defined the kernels so that $\mathfrak{C}$ is compatible with $(S^*,R^*)$.

    $\mathfrak{C}$ is not an $\mathcal{I}-\int_S$ valid model of $(S^*, R^*)$. 
    To see this, consider action $a$ such that $\mathbb{P}^a(P^*=1,R^*=1)=1$ (such an action exists since it corresponds to an intervention in $\mathcal{I}^*$).
    Now $\doop(S=0)\in \int_S^{\mathcal{I}}(a)$ as $\mathcal{L}^a(S^*)=\delta_0$ (correct conditionals on intervened nodes) and $\mathcal{L}^a(S^*)\nsim p^{\mathfrak{C}}_S$. Yet, $\mathbb{P}^a(R^*=1)=1\neq \frac{2}{3}=\mathbb{P}^{\mathfrak{C};\doop(S=0)}(R=1)$.

    \paragraph*{Making $\mathfrak{C}$ interventionally valid by using inverse entropy as complexity.}
    If we define the complexity of actions $a$ by the inverse entropy of $P^*$ in $\mathcal{L}^a(P^*)$,
    $K(a)=\frac{1}{H^a(P^*)}$ (with the convention that $K(a)=\infty$ if $H^a(P^*)=0$), then $\mathfrak{C}$ is an $\mathcal{I}-\int_{K}$ valid model of $(S^*, R^*)$. Using the inverse entropy as complexity may be reasonable since higher entropy of $P^*$ intuitively means less specificity and thus intuitively corresponds to a less complex action.
    Using this measure of complexity, for example, implies that the least complex action that goes left with probability $1$, is an action $a$ such that $\mathcal{L}^a(P^*)=\text{Unif}(\{1,2,3\})$, 
    which implies that $\mathcal{L}^a(R^*\mid S^*=0)=\mathcal{L}^{\mathfrak{C}}(R\mid S=0)$.
    Therefore, for all $a$, $\doop(S=0)\in \int_{K}^{\mathcal{I}}(a)$ only if
    $\mathcal{L}^a(R^*\mid S^*=0)=\mathcal{L}^\mathfrak{C}(R\mid S=0)$
    and $\mathbb{P}^a(S^*=0)=1$;
    and likewise $\doop(S=1)\in \int_{K}^{\mathcal{I}}(a)$ only if $\mathcal{L}^a(R^*\mid S^*=1)=\mathcal{L}^\mathfrak{C}(R\mid S=1)$
    and $\mathbb{P}^a(S^*=1)=1$. 
    This implies that for all $a$, $\doop(S=0)\in \int_{K}^{\mathcal{I}}(a)$ only if $\mathcal{L}^a(S^*,R^*)=\mathcal{L}^{\mathfrak{C};\doop(S=0)}(S,R)$ and $\doop(S=1)\in \int_{K}^{\mathcal{I}}(a)$ only if $\mathcal{L}^a(S^*,R^*)=\mathcal{L}^{\mathfrak{C};\doop(S=1)}(S,R)$, so $\mathfrak{C}$ is an $\mathcal{I}-\int_K$ valid model of $(S^*,R^*)$.

    \paragraph*{Falsifying $\mathfrak{C}$ if distance is used as complexity measure.}
    If we instead assume that complexity is given by the number of steps
    required to get from the starting position to the respective field, that is,
    $$K(a)=\mathbb{E}^a\left(3\cdot\mathbbm{1}(P^*\in \{1,3,4,6\})+2\cdot\mathbbm{1}(P^*\in \{2,5\})\right),$$ then $\mathfrak{C}$  is not an $\mathcal{I}-\int_{K}$ valid model of $(S^*,R^*)$. To see this, consider an action $a$ such that $\mathbb{P}^a(P^*=2,R^*=-1)=1$
    (such an action exists since it corresponds to an intervention in $\mathcal{I}^*$).
    Then $\doop(S=0)\in \int_{K}^{\mathcal{I}}(a)$ since the least complex action that goes left with probability 1 is the one that always goes to position $P^*=2$, but $\mathbb{P}^a(R^*=-1)=1\neq \frac{1}{3}=\mathbb{P}^{\mathfrak{C};\doop(S=0)}(R=-1)$.

    That there exist actions such that $\mathcal{L}^a(R^* \mid  S^*)\nsim p^{\mathfrak{C}}_{R\mid S}$ even though intuitively all actions seem to be modifying side rather than the reward mechanism, is an example of what \citet[Definition 6]{zhu2024meaningful} call `macro-confounding', here between $R^*$ and $S^*$. 
\end{example}

\section{Discussion of implications for related research
}\label{sec: related work}

The previous sections develop a formal framework that distinguishes between 
(1) actions that perturb the distributions of real-world quantities, 
(2) a causal model whose variables correspond to these quantities, and 
(3) an interpretation connecting actions in the world to 
interventions in the model; predictions are obtained by combining 
all three. Our notion of interventional validity therefore concerns 
a causal model together with an interpretation. Without an explicit 
interpretation, it may be unclear whether a mismatch between 
prediction and observation should be attributed to the model itself 
or to a disagreement about which intervention an action implements.

Related concerns appear in the potential-outcomes literature 
under assumptions such as consistency and discussions of treatment versions and intervention specification \citep{vanderweele2013multiple}.
These can be viewed as addressing a closely related question: 
under what conditions should different concrete implementations be regarded as instances of the same intervention? 
Our framework formalizes this issue through interpretations connecting actions to interventions in causal Bayesian networks, making it amenable to mathematical analysis.

These issues also arise when using causal models to study language-model behavior.  An intervention intended to increase the honesty of a 
language model might be implemented through 
reinforcement learning from human feedback \citep
{ouyang2022training}, constitutional AI \citep
{bai2022constitutional}, or other training 
procedures. Unlike in established experimental sciences, no shared intuition settles whether these actions count as implementations of the same intervention. A scientific causal theory of model behavior therefore requires an explicit account of interpretation, specifying, for example, how other attributes of the model may or may not be affected.

In this section, we examine causal representation learning and causal discovery in light of the previous sections. Most notably, we argue that identifiability is not sufficient for interventional validity and provide a definition of causal representations that clarifies how there is not one unique causal representation. 
In \Cref{sec: CD} and \Cref{app: CD}, we show how our framework makes sense of the fundamental assumption in causal discovery that there exists a DAG over a given set of variables and how this may turn out to be false.

In \Cref{app: abstraction}, we argue that causal abstraction rests on an infinite regress, which can only be avoided by using interpretations. In \Cref{sec: conditionals} and \Cref{sec: pheno}, we discuss connections to the philosophical literature on the logic of conditionals and related work by \citet{janzing2024phenomenological}, respectively.

\subsection{Causal representation learning}\label{sec: CRL}

\subsubsection{Identifiability \& interventional validity in causal representation learning}\label{sec: CRL-identifiability}
\label{sec:latent-intervention}

Works on interventional causal representation learning  
\citep{squires2023linear, buchholz2024learning, jin2024learning,  varici2024general, von2024nonparametric, zhang2024identifiability} often consider the following setting: %
There are
some latent variables $\rZ^*=(Z_1^*,\dots, Z_n^*)$ emulated by a causal Bayesian network $\mathfrak{A}$ (with DAG $\mathcal{G}$) and interventions $\mathcal{I}^*$;\footnote{The data-generating process is often described in terms of a latent or augmented 
structural causal model --- see, for example, \citet[Section 2.1]{von2024nonparametric}, \citet[Assumption 2]{buchholz2024learning}, and \citet[Definition 2.1]{lidisentangled}.
}
the observed data $\rX^*=(X_1^*,\dots, X^*_m)$ is given by some mixing function $f: \mathbb{R}^n\to \mathbb{R}^m$,  $\rX^*=f(\rZ^*)$, where $f$ is commonly assumed (at least) to be a diffeomorphism onto its image (e.g., \citet{von2024nonparametric, varici2024general}).
In this setting, the goal is to recover $\rZ^*$ and $\mathcal{G}$ from the distribution of $\rX^*$ in different environments $E\subseteq \mathcal{A}$, that is, from $\{\mathcal{L}^a(\rX^*)\}_{a\in E}$.%
\footnote{Some works assume a known graph and focus only on learning the unmixing function~\citep{wendong2024causal, lidisentangled}.
Other works consider additional assumptions on the mixing and latent distribution \citep{ahuja2023interventional}. %
In this section, we focus on works that address the problem of learning causal representations based on multiple interventions or domains, rather than on counterfactual and multiview data~\citep{brehmer2022weakly,von2021self,shen2022weakly, daunhawer2023identifiability, xusparsity, yaomulti}, or purely observational data~\citep{xie2020generalized, kivva2021learning,  welch2024identifiability}, or settings based on temporal structure~\citep{ahuja2022properties,lachapelle2022disentanglement, lippe2022citris}.%
}
Even under strong assumptions, this can usually only be done up to certain ambiguities. At most, we can identify the unmixing function $f^{-1}: \text{Im}(f)\to \mathbb{R}^n$ up to an equivalence class.  Definition 2.6 of \citet{von2024nonparametric} introduces one such equivalence class.
\begin{definition}{\textbf{$\sim_{\text{CRL}}$ \textbf{\citep[Definition 2.6]{von2024nonparametric}}.}} Let $\mathcal{H}$ be a set of unmixing functions $h:\text{Im}(f)\to\mathbb{R}^n$, and let $\mathfrak{G}$ be the set of DAGs over $n$ nodes. Let $\sim_{\text{CRL}}$ be the equivalence relation over $\mathcal{H}\times \mathfrak{G}$ defined as 
\begin{align*}
    &(h_1,\mathcal{G}_1) \sim_{\text{CRL}} (h_2,\mathcal{G}_2) \\
    &\quad\iff\quad (h_2,\mathcal{G}_2)=(\rP_{\pi^{-1}}\circ \phi \circ h_1,\pi(\mathcal{G}_1))
\end{align*}
for some
graph isomorphism $\pi:\mathcal{G}_1\to\mathcal{G}_2$, where $\rP_\pi$ is the corresponding permutation matrix over $[n]$,
and
element-wise diffeomorphism 
$\phi(\bm{v})=(\phi_1(v_1),\dots, \phi_n(v_n))$ of $\mathbb{R}^n$.
\end{definition}

\paragraph*{$\sim_{\text{CRL}}$ preserves interventional validity if the latent representation is emulated by single-node interventions.}
The $\sim_\text{CRL}$ equivalence class
is quite special: If $\rZ^*=f^{-1}(\rX^*)$ is emulated by a CBN $\mathfrak{A}$ and single-node interventions in $\mathfrak{A}$ 
and $(h,\widetilde{\mathcal{G}})\sim_{\text{CRL}} (f^{-1},\mathcal{G})$, then there exists a CBN $\widetilde{\mathfrak{A}}$ such that $h(\rX^*)$ is emulated by single-node interventions in $\widetilde{\mathfrak{A}}$.\footnote{This follows from Proposition B.1 in \citet{von2024nonparametric}.}
This implies 
that $\widetilde{\mathfrak{A}}$ is an $\mathcal{I}-\int_S$ valid model of $h(\rX^*)$ for every set of interventions $\mathcal{I}$ in $\widetilde{\mathfrak{A}}$, see \Cref{prop: d2} (1).
 From this point of view, there is no sense in which $f^{-1}(\rX^*)$ is more a `causal representation' or the `ground truth' than $h(\rX^*)$,
and thus the choice of one representation as the assumed latent ground-truth is arbitrary.\footnote{%
\citet{von2024nonparametric} posit that ``since the scale of the variables is arbitrary, we clearly cannot predict the exact outcomes of interventions.'' 
This statement may be too strong if we adopt the perspective that there is no unique `true' causal representation, see \Cref{def: CRL}.  
When do we interpret an action $a$ as corresponding to intervention $\doop(H_1:=3)$? Under interpretation $\int_S$, for example, this would simply be an action such that $\mathcal{L}^a(h(\rX^*)_1)=\delta_3$ (assuming that conditions 2) and 3) of \Cref{def: single node} are also satisfied). It is no more difficult to predict the outcome of interventions using the representation $h(\rX^*)$ than the representation $f^{-1}(\rX^*)$.}
\paragraph*{$\sim_\text{CRL}$-identifiability does not imply interventional validity.}
Suppose that we have assumptions in place such that $f^{-1}$ is 
identifiable up to $\sim_\text{CRL}$ based on $\{\mathcal{L}^a(\rX^*)\}_{a\in E}$, and assume that $(h,\widetilde{\mathcal{G}})\sim_{\text{CRL}} (f^{-1},\mathcal{G})$.
Let $\widetilde{\mathfrak{A}}$ (with graph $\widetilde{\mathcal{G}}$) and $\widetilde{\mathcal{I}}^*$ be such that  $h(\rX^*)$ is emulated by $\widetilde{\mathfrak{A}}$ and $\widetilde{\mathcal{I}}^*$ with link $\widetilde{g}: \mathcal{A}\setminus \{ \mathcal{O}\}\to \widetilde{\mathcal{I}}^*$.
Then, it may be the case that $\widetilde{\mathfrak{A}}$ is not a $\{d\in \widetilde{\mathcal{I}}^*\mid d \text{ is a single-node intervention}\}-\int_S$ valid model of $h(\rX^*)$;
in particular, \Cref{prop: d2} (1) only establishes validity if all interventions in $\widetilde{\mathcal{I}}^*$ are single-node.
In fact, as we will illustrate in \Cref{ex: CRL}, $\widetilde{\mathfrak{A}}$ can fail to be a $\widetilde{g}(E)-\int_S$ valid model of $h(\rX^*)$ even if $\widetilde{g}(a)$ is a single-node intervention for all $a\in E$.
Therefore, to ensure interventional validity
we need to not only make assumptions about the observed environments (often necessary for identifiability), but also about the entire set of possible actions (this type of extrapolation to hypothetical (future) environments is discussed in \citet{buhlmann2020invariance}).
Conversely, even if the identified latent model is interventionally valid,
other models outside the equivalence class can also be interventionally valid (see \Cref{sec: CRL-def} and \Cref{ex: abstraction}).

\subsubsection{A definition of causal representations}\label{sec: CRL-def}

Using the notion of interventional validity (\Cref{def: interventionally valid}), we now provide a formalization of what it means for a representation to be a causal representation. To the best of our knowledge, this is the first definition in the literature that defines a causal representation based on its properties rather than stipulating that a certain representation is the `ground truth' causal one.%
\footnote{\citet{cohen2022towards} raised the related point that existing
frameworks for causal modeling ``give no guidance regarding variable
[...] representation, and [...] no indication as to which behaviour policies or
physical transformations of state space shall count as interventions.''}

\begin{definition}{\textbf{Causal Representation.}} \label{def: CRL}
    Let a data-generating process $\mathcal{D}$, an interpretation $\int$, and a set of interventions $\mathcal{I}$ be given. We say that a representation $\rZ^*$ of $\mathcal{D}$ is \textit{an $\mathcal{I}-\int$ causal representation of $\mathcal{D}$} if there exists a CBN $\mathfrak{C}$ such that $\mathcal{I}$ is a set of interventions in $\mathfrak{C}$ and $\mathfrak{C}$ is an $\mathcal{I}-\int$ valid model of~$\rZ^*$.
\end{definition}

The larger $\mathcal{I}$ is, the more interventions the model purports to make predictions about, and the interpretation $\int$ specifies which actions are interpreted as which interventions.
The interpretation $\int_C$ is special as every representation $\rZ^*$ of a data-generating process is an $\mathcal{I}-\int_C$ causal representation if $\mathcal{I}$ is a set of interventions in a $\rZ^*$-compatible CBN;
that is, without committing to certain non-circular interpretations,
there is apparently nothing distinctively causal about causal representations.

A consequence of \Cref{def: CRL} is that there is no unique causal representation, which clarifies how we may model a data-generating process at different levels of abstraction, see \Cref{app: abstraction}, and, as suggested by \citet{sadeghi2024axiomatization}, that ``[t]here is no need to take the true causal
graph as the primitive object''.
It also suggests an approach to causal representation learning that focuses on learning representations that satisfy interventional validity
rather than recovering some `true' latent representation emulated by a CBN.
In fact, a key insight of the present work is that a representation $\rZ^*$ can be emulated by a CBN $\mathfrak{C}$ without $\mathfrak{C}$ being an $\mathcal{I}-\int$ valid model of $\rZ^*$ (for non-circular interpretation $\int$ and nonempty set of interventions $\mathcal{I}$, see, for example, \Cref{prop: intA,prop: d2,prop: complexity}).
Furthermore, the assumptions that render a latent CBN representation identifiable (up to some equivalence class) need not ensure it is a causal representation with desirable properties
such as interventional validity, see \Cref{sec: CRL-identifiability}.

\paragraph*{Toward a new approach to causal representation learning.}
We can formulate the task as follows:  Let a data-generating process $\mathcal{D}$ and low-level features 
$\rX^*$ be given. Find a non-circular interpretation $\int$ and a transformation $h$ such that $h(\rX^*)$ is an $\mathcal{I}-\int$ causal representation for a suitable set of interventions $\mathcal{I}$. 

Requiring that $h(\rX^*)$ is a causal representation is not sufficient to get an interesting representation.
For example, the trivial representation $Z^*=h(\rX^*):=0$ is a causal representation, according to~\Cref{def: CRL}, since the CBN with one node $Z$ and distribution $\mathbb{P}(Z=0)=1$ is an $\mathcal{I}-\int$ valid model of $Z^*=0$ for every set of interventions $\mathcal{I}$ in that CBN if $\int$ is an interpretation that satisfies \textbf{D0}.
Therefore in addition to requiring interventional validity, we probably want our representations to
satisfy further criteria
(some of which have, sometimes implicitly, motivated representation learning approaches). We list some criteria in \Cref{app: desiderata}.
What precise properties to require of
a causal representation and what it may be used for,
is in our opinion a neglected question; see \citep{woodward2016problem,bing2024invariance,cadei2024smoke} for discussions.
While some existing representation learning approaches
strive for representations that satisfy criteria similar to the ones listed in \Cref{app: desiderata},
the resulting causal model does not make precise empirical predictions
when there is no explicit choice of how to interpret actions as interventions.
Since many transformations of the observed data may yield causal representations,
how to choose between 
representations and which properties to impose 
is a pressing question for future research.

\paragraph*{Learning objectives and identifiability.}
Once the requirements for a causal representation are well specified,
an important question for future research is the precise formulation of a learning objective for causal representations in the sense of~\Cref{def: CRL}.
Existing learning objectives in causal representation learning
may
yield causal representations
but will require an explicit choice of a suitable interpretation and additional assumptions.
Furthermore, our work suggests a 
different perspective on identifiability theory:
Rather than focusing on identifiability as a guarantee for the learned representation to be disentangled with respect to some `true' latent variables, an identifiability analysis could assess whether all maximizers of a given learning objective share properties we deem desirable (similar in spirit to, for example, the analysis by \citet{marconato2024all}).
For example%
,
in~\Cref{sec:latent-intervention}, we show that the $\sim_{\text{CRL}}$ equivalence class preserves interventional validity %
under the  interpretation $\int_S$~(\Cref{def: single node}). 
If one is willing to make assumptions on the set of all possible actions and not only those that induce the distributions used for learning the representation,
then
an identifiability result proving that maximizers of a given objective are in the $\sim_{\text{CRL}}$ equivalence class
would thus provide a sufficient (but not necessary) condition for all (or none) of the corresponding representations to be causal representations.
 Since coarser equivalence classes than $\sim_\text{CRL}$ may preserve interventional validity (depending on the chosen interpretation and assumptions), it may be possible to simplify the learning problem.

\subsection{Causal discovery}\label{sec: CD}

\paragraph*{What do we assume when we assume that observed variables are described by a CBN?}
The starting point of causal discovery is, in the words of  \citet[Assumption 6.1]{dawid2010beware}, the assumption that ``[t]here exists some DAG [$\mathcal{G}$] that is a causal DAG representation of the system.'' 
\citet{dawid2010beware} considers this a ``strong [a]ssumption'', suggesting that the assumption does not amount to mere emulation, which is always possible (see \Cref{def: generation} and discussion below it).
Our work makes precise what this assumption may amount to other than emulation, namely
 the existence of a CBN $\mathfrak{C}$ and a set of interventions $\mathcal{I}$ in $\mathfrak{C}$ such that $\mathfrak{C}$ is an $\mathcal{I}-\int$ valid model of the given representation $\rZ^*$ (see \Cref{def: CRL}).
Whether this assumption holds depends both on the interpretation $\int$ and the modeled interventions $\mathcal{I}$. 
We clarify that this assumption can fail to be true, namely for a given representation $\rZ^*$, non-circular interpretation $\int$, and set of interventions $\mathcal{I}$, there may not exist an
$\mathcal{I}-\int$ valid model of $\rZ^*$.
In \Cref{app: CD}, we demonstrate this with the classical example of the effect of total cholesterol on heart disease. Here, our framework enables us to make the claim formally precise that this representation does not admit a causal model. In \Cref{app: CD}, we also argue that identifiability does not imply interventional validity, similarly to the situation in causal representation learning.

\section{Conclusion}\label{sec:conclusion}

Without specifying which real-world actions correspond to which interventions in the mathematical model, it is unclear 
what it means for a causal Bayesian network to be a valid causal model of a representation.
We develop a formal framework for reasoning about the interventional validity of a causal model, which depends on the chosen interpretation of actions as interventions.
We discuss different interpretations by considering five desiderata, \textbf{D0}--\textbf{D4}, some of which must be violated to escape circularity.
Only when the interpretation is made precise can
causal models make testable predictions about future observations of a real-world system, which is crucial to enable falsification.\footnote{The interpretation of actions as interventions is not the only ambiguity that arises when relating causal models to real-world systems; for example, \citet{reisach2026time} highlight that a causal DAG is ambiguous about which points in time its variables refer to.}
A failed prediction therefore does not by itself indicate whether the model or the interpretation is at fault. 
We submit that rigorous thinking about the relationship between real-world systems and causal models is critical to facilitate the use of causal models in practice.
Otherwise,
it is unclear how to use causal models to
predict the effect of an action unless the system has already been observed under that very action, echoing the conclusion in ``Use and Abuse of Regression'' \citep{box1966use} that
``[t]o find out what happens to a system when you interfere with it you have to interfere with it (not just passively observe it).''

\paragraph*{Acknowledgments.}
This work was supported by the Pioneer Centre for AI, DNRF grant number P1.
FHJ and SW were supported by a research grant (0069071) from the Novo Nordisk Foundation.
LG was supported by the Danish Data Science Academy, which is funded by the Novo Nordisk Foundation (NNF21SA0069429).

\appendix

\section{Notation}\label{app: notation}

 \begin{table*}[hbp]
\adjustbox{center=\textwidth}{
\begin{tabular}{p{3cm} p{15cm}} \toprule
\textbf{Symbol} & \textbf{Description} \\ \midrule
$\mathcal{A}$ & A set of actions, see \Cref{def: dgp}. \\ \midrule
$\mathcal{O}$ & $\mathcal{O}\in \mathcal{A}$ denotes the action corresponding to the observational regime, see \Cref{def: dgp}. \\ \midrule
$\mathcal{G}$ &  
Directed acyclic graph over nodes $\{1, \dots, n\}$. \\ \midrule
$[n]$ &  
 The set $\{1,\dots, n\}$. \\ \midrule
$\mathfrak{C}$, $\mathfrak{A}$, $\mathfrak{H}$ &  Causal Bayesian networks (CBNs), see \Cref{def: CGM}.  \\ \midrule
$h$ & Function $\mathbb{R}^m\to \mathbb{R}^n$ transforming low-level features, see \Cref{def:representation}.  \\ \midrule
$\rZ$ & Multivariate random variable of CBN variables
$\rZ\in \mathbb{R}^n$.   \\ \midrule
$\rX^*$ & Multivariate random variable of low-level features, $\rX^*\in \mathbb{R}^m$, see \Cref{def: dgp}.   \\ \midrule
$\rZ^*$ & Multivariate representation of a data-generating process, $\rZ^*=h(\rX^*)\in \mathbb{R}^n$, see \Cref{def:representation}.   \\ \midrule
$\mathcal{L}$ & Denotes a distribution. For example, $\mathcal{L}^a(h(\rX^*))$ is the distribution induced by action $a\in \mathcal{A}$ over $h(\rX^*)$, see \Cref{def:representation}, and $\mathcal{L}^{\mathfrak{C};d}(\rZ)$ denotes the distribution induced by CBN $\mathfrak{C}$ over variables $\rZ$ and intervention $d$, see \Cref{def: CGM}.  \\ \midrule
$\mathcal{L}^a(Z_j^*\mid \bm{Y}^*)\sim p$ & Denotes the claim that the kernel $\bm{y}^*\mapsto p(\cdot \mid \bm{y}^*)\cdot \nu$ is a regular conditional probability distribution of $Z_j^*$ given $\bm{Y}^*\subseteq \rZ^*$ under distribution $\mathcal{L}^a(\rZ^*)$ (for some fixed $\sigma$-finite measure $\nu$, which is suppressed in the notation).%
\\ \midrule
$\mathcal{I}$ & Set of interventions in a CBN.  \\ \midrule
$\mathcal{I}^*$ & Set of interventions in a CBN. Used to denote interventions in a CBN that emulate a representation, see \Cref{def: generation}.  \\ \midrule
$\doop(j\gets q_j,j\in J)$ & An intervention on nodes $J$, see \Cref{def: CGM}.  \\ \midrule
$d$, $d^*$ & Denotes interventions, $d\in \mathcal{I}$, $d^*\in \mathcal{I}^*$.  \\ \midrule
$p_i^{\mathfrak{C}}$, $p_i^{\mathfrak{C};d}$ & 
$p_i^{\mathfrak{C}}$ denotes the $i$'th kernel given by CBN $\mathfrak{C}$. $p_i^{\mathfrak{C};d}$ denotes $i$'th kernel given by CBN $\mathfrak{C}$ and intervention $d=\doop(j\gets q_j, j\in J)$, that is, $p_i^{\mathfrak{C};d}=p_i^{\mathfrak{C}}$ for $i\notin J$, and $p_i^{\mathfrak{C};d}=q_i$ for $i\in J$.\\ \midrule
$\int$ & An interpretation that takes a set of interventions $\mathcal{I}$ in a CBN and an action $a\in \mathcal{A}$ and outputs a subset $\int^\mathcal{I}(a)$ of $\mathcal{I}$, see \Cref{def: interpretation} and \Cref{sec: interpretations overview}. \\ \midrule
$\int^{\mathcal{I}}$ & Given a set of interventions $\mathcal{I}$ in a CBN, an interpretation induces a mapping $\int^{\mathcal{I}}: \mathcal{A}\to \mathcal{P}(\mathcal{I})$, see \Cref{def: interpretation} and \Cref{sec: interpretations overview}. \\ \midrule
$\mathcal{N}(\bm{\mu},\bm{\Sigma})$ &  Joint normal distribution with mean $\bm{\mu}$ and covariance $\bm{\Sigma}$.\\ \midrule
$\delta_x$ &  
Dirac distribution with support $\{x\}$ ($x\in \mathbb{R}$).\\ \midrule
$\text{Unif}(S)$ & Uniform distribution over finite set $S$.  \\ \midrule
$\text{Ber}(p)$ & Bernoulli distribution with mean $p$.  \\ \midrule
$\otimes$ &  
$\nu_1 \otimes \nu_2$ denotes the product measure of two $\sigma$-finite measures $\nu_1$ and $\nu_2$.
\\
\bottomrule \end{tabular}  
}
\end{table*}

\newpage 

\section{Interpretations overview}\label{sec: interpretations overview}
\begin{center}
\captionof{table}{Overview of interpretations presented in this work.}

\adjustbox{width=\textwidth, max totalheight=0.85\textheight}{
\begin{tabular}{m{2.5cm} m{5cm} m{3.3cm} m{4.3cm}}
\toprule
\textbf{Interpretation} & \textbf{Definition: Interpret action $a$ as intervention $d$ if..} & \textbf{Violated desiderata} & \textbf{Can falsify a model?} \\ \midrule

\makecell[l]{$\int_C$ \\ \Cref{def: tautology} \\(Circular)}
& 
\vspace{0.3cm}
\begin{enumerate}[leftmargin=*, labelindent=0pt, itemindent=0pt]
    \item[..] $a$ sets the conditionals of intervened nodes correctly,
    \item[$\wedge$] $a$ changes no other conditionals,
    \item[$\wedge$] $a$ does not introduce dependencies.
\end{enumerate} & \textbf{None.}
&\begin{tabular}{@{}p{4cm}@{}}
\textbf{No.}\\ ($\int_C$ is circular, \Cref{prop: tautology2}.)
\end{tabular} 
 \\ \midrule

\makecell[l]{$\int_P$ \\ \Cref{def: perfect} \\(Perfect)}
& 
\vspace{0.3cm}
\begin{enumerate}[leftmargin=*, labelindent=0pt, itemindent=0pt]
    \item[..] $d$ is a perfect intervention,
    \item[$\wedge$] $a$ sets the conditionals of intervened nodes correctly,
    \item[$\wedge$]  non-intervened nodes are source nodes with unchanged distributions or not independent of parents,
    \item[$\wedge$] $a$ does not introduce dependencies.
\end{enumerate}
& \vspace{-0.3cm}\begin{tabular}{@{}p{3cm}@{}}
\textbf{D1} \\(If it behaves like an intervention, it is that intervention.)
\end{tabular}
&\vspace{-0.3cm}\begin{tabular}{@{}p{4cm}@{}}
\textbf{Yes.} \\ (Falsified if an action behaves like an imperfect intervention, \Cref{prop: intA}.) 
\end{tabular} 
\\ \midrule

\makecell[l]{$\int_S$ \\ \Cref{def: single node} \\(Single-node)}
& 
\vspace{0.3cm}
\begin{enumerate}[leftmargin=*, labelindent=0pt, itemindent=0pt]
    \item[..] $a$ sets the conditionals of intervened nodes correctly,
    \item[$\wedge$] $a$ changes the distribution of intervened nodes,
    \item[$\wedge$]  $a$ does not introduce dependencies.
\end{enumerate}
& \vspace{-0.3cm}\begin{tabular}{@{}p{3cm}@{}}
     \textbf{D2}\\ (An action should not be interpreted as distinct interventions).
\end{tabular}
& \vspace{-0.3cm}\begin{tabular}{@{}p{4cm}@{}}
\textbf{Yes.} \\(Falsified if an action behaves like a multi-node intervention, \Cref{prop: d2}.) 
\end{tabular} \\ \midrule

\makecell[l]{$\int_{\widetilde{\mathcal{I}},f}$ \\ \Cref{def: intervention set} \\($f$-least in $\widetilde{\mathcal{I}}$)}
&\vspace{0.3cm}\begin{enumerate}[leftmargin=*, labelindent=0pt, itemindent=0pt]
    \item[..] $d\in \int_C^{\mathcal{I}}(a)$,
    \item[$\vee$] $\int_C^{\mathcal{I}}(a)=\emptyset$ and $d$ is least element in $\widetilde{\mathcal{I}} \cap \int_S^{\mathcal{I}}(a)$, in the strict total order on $\widetilde{\mathcal{I}}$ induced by $f$.
\end{enumerate}
& \vspace{-0.3cm}\begin{tabular}{@{}p{3cm}@{}}
 \textbf{D3}\\ (Interpretations should not depend on the intervention set $\mathcal{I}$.)
\end{tabular}
& \begin{tabular}{@{}p{4cm}@{}}
\textbf{Yes.} (See \Cref{ex: intervention set}.) 
\end{tabular} \\ \midrule

\makecell[l]{$\int_M$ \\ \Cref{def: markov} \\(Markov)}
& \vspace{0.7cm}\begin{enumerate}[leftmargin=*, labelindent=0pt, itemindent=0pt]
    \item[..] $a$ sets the conditionals of intervened nodes correctly,
    \item[$\wedge$] $a$ changes no other conditionals.
\end{enumerate}
& \vspace{-0.3cm} \begin{tabular}{@{}p{3cm}@{}}
\textbf{D4} \\ (An intervention does not create new dependencies.)
\end{tabular}
& \vspace{-0.3cm}\begin{tabular}{@{}p{4cm}@{}}
\textbf{Yes, but} can only falsify CBNs with non-complete DAGs.\\ (Falsified if an action introduces dependencies between variables, \Cref{prop: intM}.)
\end{tabular} \\[.8cm] \midrule

\makecell[l]{$\int_K$ \\ \Cref{def: complexity} \\(Complexity $K$)\\[.1cm]}
& \vspace{0.3cm}
\begin{enumerate}[leftmargin=*, labelindent=0pt, itemindent=0pt]
    \item[..] $a\in \underset{a'\in \mathcal{A}: d\in \int_S^{\mathcal{I}}(a')}{\arg \min}K(a')$.
\end{enumerate}
&\textbf{D1} and \textbf{D2}.
& \textbf{Yes.} (See \Cref{ex: RL}.) \\
\bottomrule
\end{tabular}}
\end{center}

\newpage

\section{Proofs}\label{app: proofs}
This appendix contains the proofs of the results stated in the main text.

\propTautologyOne*
\begin{proof}[Proof of \Cref{prop: tautology1}]
    Let $d \in \mathcal{I}$ and $a\in \mathcal{A}$ be given. If $\mathcal{L}^{a}(\rZ^*)=\mathcal{L}^{\mathfrak{C};d}(\rZ)$,
    this immediately implies 1)--3) in \Cref{def: tautology} and hence that
    $d\in \mathbf{Int}^{\mathcal{I}}_C(a)$.

     Assume that $d=\doop(j\gets q_j, j\in J)\in  \mathbf{Int}^{\mathcal{I}}_C(a)$. Since $\mathcal{L}^a(\rZ^*)$ is Markov w.r.t.\ the DAG of $\mathfrak{C}$, there exists some CBN $\mathfrak{A}$ that has the same DAG as $\mathfrak{C}$ such that $\mathcal{L}^a(\rZ^*)=\mathcal{L}^\mathfrak{A}(\rZ)$. We then have that
     \begin{align*}
         \mathcal{L}^a(\rZ^*)&=\mathcal{L}^\mathfrak{A}(\rZ)\\
         &=\mathcal{L}^{\mathfrak{A};\doop(j\gets q_j,j\in J; j\gets p_j^{\mathfrak{C}},j\notin J)}(\rZ)\\
         &=\mathcal{L}^{\mathfrak{C};\doop(j\gets q_j,j\in J)}(\rZ),
     \end{align*}
    where the second equality follows by using 1) and 2) of \Cref{def: tautology}.\footnote{The superscript in the second line denotes replacing every kernel of $\mathfrak{A}$; this kernel replacement need not itself be an intervention in the sense of \Cref{def: CGM}, and we use it only as notation for the resulting distribution.}
\end{proof}

\propImpossibility*
\begin{proof}[Proof of \Cref{prop: impossibiliy}]
    Let $a\in \mathcal{A}$ with $\mathcal{L}^a(\rZ^*)\neq \mathcal{L}^{\mathcal{O}}(\rZ^*)$ and $\mathcal{I}$ be given.
    From \Cref{prop: tautology1} and \textbf{D1}, it follows that $\mathbf{Int}^{\mathcal{I}}_C(a)\subseteq \mathbf{Int}^{\mathcal{I}}(a)$.
    Assume that $d\in \mathbf{Int}^{\mathcal{I}}(a)$. By \textbf{D4} we have that $\mathcal{L}^a(\rZ^*)$ is Markov w.r.t.\ the DAG of $\mathfrak{C}$. Since $\mathcal{L}^a(\rZ^*)$ is Markov w.r.t.\ the DAG of $\mathfrak{C}$, we can find an intervention $b$ such that $\mathcal{L}^{\mathfrak{C};b}(\rZ)=\mathcal{L}^a(\rZ^*)$, namely an intervention $b$ such that $\mathcal{L}^a(Z_i^*\mid \PA_i^*)\sim p_{i}^{\mathfrak{C};b}$ for all $i\in [n]$.
    ($b$ is indeed an intervention since $\mathcal{L}^{\mathfrak{C};b}(\rZ)=\mathcal{L}^a(\rZ^*)\neq \mathcal{L}^{\mathcal{O}}(\rZ^*)=\mathcal{L}^{\mathfrak{C}}(\rZ)$, where the inequality holds by assumption on $a$ and the last equality by compatibility of $\mathfrak{C}$.)
    Consider $\widetilde{\mathcal{I}}=\mathcal{I}\cup \{b\}$. By condition \textbf{D1}, $b\in  \mathbf{Int}^{\widetilde{\mathcal{I}}}(a)$, and by condition \textbf{D3}, $d\in \int^{\widetilde{\mathcal{I}}}(a)$. By condition \textbf{D2}, this implies that $\mathcal{L}^{\mathfrak{C};d}(\rZ)=\mathcal{L}^{\mathfrak{C};b}(\rZ)$. \Cref{prop: tautology1} then gives us that $d\in  \mathbf{Int}^{\mathcal{I}}_C(a)$ since $\mathcal{L}^a(\rZ^*)=\mathcal{L}^{\mathfrak{C};d}(\rZ)$. 
\end{proof}

\propIntP*
\begin{proof}[Proof of \Cref{prop: intA}]
(1)
    Assume that $\mathcal{I}^*$ contains only perfect interventions. Consider some $a\in \mathcal{A}$ and some intervention $d=\doop(j\leftarrow q_j,j\in J)\in \int_P^{\mathcal{I}}(a)$. Since $\int_P^{\mathcal{I}}(\mathcal{O})=\emptyset$ (\Cref{def: interpretation}), we have $a\neq \mathcal{O}$. We need to show that $\mathcal{L}^{\mathfrak{C};d}(\rZ)=\mathcal{L}^{a}(\rZ^*)$. By assumption, $\rZ^*$ is emulated by $\mathfrak{C}$ and $\mathcal{I}^*$, so there is a perfect intervention $d^*=g(a)=\doop(j\gets \widetilde{q}_{j},j\in \widetilde{J})\in \mathcal{I}^*$ such that $\mathcal{L}^{a}(\rZ^*)=\mathcal{L}^{\mathfrak{C}; d^*}(\rZ)$. We have
\begin{align*}
    \mathcal{L}^a(\rZ^*)&=\mathcal{L}^{\mathfrak{C};d^*}(\rZ)\\
    &=\mathcal{L}^{\mathfrak{C};\doop(j\gets \widetilde{q}_j,j\in \widetilde{J})}(\rZ)\\
    &=\mathcal{L}^{\mathfrak{C};\doop(j\gets \widetilde{q}_j,j\in \widetilde{J}\setminus J; \  j\gets q_j,j\in J)}(\rZ)\\
    &=\mathcal{L}^{\mathfrak{C};\doop(j\gets q_j,j\in J)}(\rZ)\\
    &=\mathcal{L}^{\mathfrak{C};d}(\rZ)
\end{align*}
    The third equality follows from condition 2) of \Cref{def: perfect}. The fourth equality follows from condition 3) and the fact that $d^*$ is perfect: 
    For every node $i\notin J$, either (a) $\PA_i$ is empty and $\mathcal{L}^{\mathfrak{C};d^*}(Z_i)=\mathcal{L}^{\mathfrak{C}}(Z_i)$, or (b) $\PA_i$ is nonempty and $Z_i\nindep \PA_i$ in $\mathcal{L}^{\mathfrak{C};d^*}(\rZ)$, which implies that $i\notin \widetilde{J}$ since $d^*$ is perfect.  

(2)
We want to find an action $a$, a set of interventions $\mathcal{I}$, and an intervention $d \in \int_P^{\mathcal{I}}(a)$ such that $\mathcal{L}^{a}(\rZ^*)\neq \mathcal{L}^{\mathfrak{C};d}(\rZ)$. By assumption, $\mathcal{I}^*$ contains a minimal and decomposable intervention $d^*=\doop(j\leftarrow \widetilde{q}_j, j\in \widetilde{J})$ for which there exist $s,t\in \widetilde{J}$ such that $Z_s \indep \PA_s$ and $Z_t \nindep \PA_t$ in $\mathcal{L}^{\mathfrak{C};d^*}(\rZ)$.
     Let $\overline{J}:=\{j\in \widetilde{J} \mid Z_j\indep \PA_j \text{ in } \mathcal{L}^{\mathfrak{C};d^*}(\rZ)\} \subseteq\widetilde{J}\setminus\{t\} \subsetneq \widetilde{J}$, and let  $C:=\{j\in [n]\setminus \widetilde{J}\mid Z_j\indep \PA_j \text{ in } \mathcal{L}^{\mathfrak{C};d^*}(\rZ) \text{ and } \PA_j\neq \emptyset   \}$.
     Let $\doop(j\gets q_j, j\in \overline{J}\cup C)$ be a perfect intervention such that  $\mathcal{L}^{\mathfrak{C};\doop(j\gets q_j, j\in \overline{J}\cup C)}(Z_i)=\mathcal{L}^{\mathfrak{C};d^*}(Z_i)$ for all $i\in \overline{J}\cup C$.  
     For every $i\in \overline{J}\cup C$, $\mathcal{L}^{\mathfrak{C};\doop(j\gets q_j, j\in \overline{J}\cup C)}(Z_i)=\mathcal{L}^{\mathfrak{C};d^*}(Z_i)$ implies that $\mathcal{L}^{\mathfrak{C};d^*}(Z_i\mid \PA_i)\sim p_i^{\mathfrak{C};\doop(j\gets q_j, j\in \overline{J}\cup C)}$ 
     since $Z_i\indep \PA_i$ in $\mathcal{L}^{\mathfrak{C};d^*}(\rZ)$ and $p_i^{\mathfrak{C};\doop(j\gets q_j, j\in \overline{J}\cup C)}$ is perfect.
     Let the set of modeled interventions consist of this intervention, that is, $\mathcal{I}=\{\doop(j\leftarrow q_j,j\in \overline{J}\cup C)\}$.
     Let $a$ be an action such that $\mathcal{L}^a(\rZ^*)=\mathcal{L}^{\mathfrak{C};d^*}(\rZ)$ (such an action exists since $\rZ^*$ is emulated by $\mathfrak{C}$ and $\mathcal{I}^*$).
     We have that $\doop(j\leftarrow q_j,j\in \overline{J}\cup C)\in \int_P^{\mathcal{I}}(a)$ since
     1) it is perfect, 2) sets the conditionals of intervened nodes correctly, 3) every node $i$ not 
     in $\overline{J}\cup C$ is either a source node not in $\widetilde{J}$, or $Z_i\nindep \PA_i$ in $\mathcal{L}^{\mathfrak{C};d^*}(\rZ)$,\footnote{Let $i$ be a node not in $\overline{J}\cup C$. Since $i$ is not in $\overline{J}$, either $Z_i\nindep \PA_i$ in $\mathcal{L}^{\mathfrak{C};d^*}(\rZ)$ or $i$ is not in $\widetilde{J}$. If $Z_i\indep \PA_i$ in $\mathcal{L}^{\mathfrak{C};d^*}(\rZ)$ and $i$ is not in $\widetilde{J}$, $i$ must be a source node, otherwise $i$ would be in $C$.} and condition 4) of \Cref{def: perfect} trivially holds.

     If $\mathcal{L}^{\mathfrak{C};d^*}(\rZ)=\mathcal{L}^{\mathfrak{C};\doop(j\leftarrow q_j,j\in \overline{J}\cup C)}(\rZ)$,
     then
     for all $i\in C$ we have
     $\mathcal{L}^{\mathfrak{C};\doop(j\leftarrow q_j,j\in \overline{J}\cup C)}(Z_i \mid \PA_i)\sim p^{\mathfrak{C}}_i$ 
     since nodes in $C$ are not intervened upon by $d^*$. Therefore, if $\mathcal{L}^{\mathfrak{C};d^*}(\rZ)=\mathcal{L}^{\mathfrak{C};\doop(j\leftarrow q_j,j\in \overline{J}\cup C)}(\rZ)$, we would also have that $\mathcal{L}^{\mathfrak{C};d^*}(\rZ)=\mathcal{L}^{\mathfrak{C};\doop(j\leftarrow q_j,j\in \overline{J})}(\rZ)$, but this is a contradiction since $\overline{J}\subsetneq \widetilde{J}$ and $d^*$ is minimal. Therefore, we can conclude that $\mathcal{L}^{\mathfrak{C};d^*}(\rZ)\neq\mathcal{L}^{\mathfrak{C};\doop(j\leftarrow q_j,j\in \overline{J}\cup C)}(\rZ)$ and that $\mathfrak{C}$ is not an $\mathcal{I}-\int_P$ valid model of $\rZ^*$.
\end{proof}

\propIntS*
\begin{proof}[Proof of \Cref{prop: d2}]
        (1) Assume that $\mathcal{I}^*$ contains only single-node interventions. Consider some action $a$ with $\int_S^{\mathcal{I}}(a)\neq \emptyset$; in particular, $a\neq \mathcal{O}$ since $\int_S^{\mathcal{I}}(\mathcal{O})=\emptyset$ (\Cref{def: interpretation}).
        Since $\rZ^*$ is emulated by $\mathfrak{C}$ and $\mathcal{I}^*$,
        there is a single-node intervention $d^*\in \mathcal{I}^*$ such that $\mathcal{L}^{a}(\rZ^*)=\mathcal{L}^{\mathfrak{C}; d^*}(\rZ)$.
        Consider some intervention $d=\doop(j\leftarrow q_j, j\in J)\in \int_S^{\mathcal{I}}(a)$.
        We want to show that $\mathcal{L}^{\mathfrak{C};d^*}(\rZ)=\mathcal{L}^{\mathfrak{C};d}(\rZ)$ and do this by arguing that $\mathcal{L}^{\mathfrak{C};d^*}(Z_i |  \PA_i)\sim p^{\mathfrak{C};d}_{i}$ for all $i\in [n]$.
        
        Per condition 1) of \Cref{def: single node}, we have for all $i\in J$ that
        \begin{align*}
            \mathcal{L}^{\mathfrak{C};d^*}(Z_i\mid \PA_i)\sim q_i = p^{\mathfrak{C};d}_{i}.
        \end{align*}
        By condition 2), $\mathcal{L}^{\mathfrak{C};d^*}(Z_i\mid \PA_i)\nsim p^{\mathfrak{C}}_i$ for all $i\in J$, so the single-node intervention $d^*$ must be a single-node intervention on a node in $J$.
        Therefore, $d^*$ does not intervene on nodes outside of $J$, and thus we have for all $i\notin J$ that
        \begin{align*}
            \mathcal{L}^{\mathfrak{C};d^*}(Z_i\mid \PA_i)\sim p^{\mathfrak{C}}_{i}=p^{\mathfrak{C};d}_{i}.
        \end{align*}
         In summary, we have established that $\mathcal{L}^{\mathfrak{C};d^*}(Z_i\mid \PA_i)\sim p^{\mathfrak{C};d}_{i}$ for all $i\in [n]$, which implies that $\mathcal{L}^{\mathfrak{C};d^*}(\rZ)=\mathcal{L}^{\mathfrak{C};d}(\rZ)$.
        (2) Assume that there is a minimal and decomposable multi-node intervention $d^*=\doop(j\leftarrow q_j, j\in J)\in \mathcal{I}^*$ and let $a$
        be such that $\mathcal{L}^a(\rZ^*)=\mathcal{L}^{\mathfrak{C};d^*}(\rZ)$.
        Fix some $j'\in J$ and let $\mathcal{I}=\{\doop(j'\gets q_{j'})\}$ (which is well-defined since $d^*$ is decomposable and thus $\mathcal{L}^{\mathfrak{C};\doop(j' \gets q_{j'})}(\rZ)\neq\mathcal{L}^\mathfrak{C}(\rZ)$). $\doop(j'\gets q_{j'})\in \int_S^{\mathcal{I}}(a)$ as 1) $\mathcal{L}^{a}(Z_{j'}^*\mid \PA_{j'}^*)\sim q_{j'}$, 2) $\mathcal{L}^{a}(Z_{j'}^*\mid \PA_{j'}^*)\nsim p_{j'}^{\mathfrak{C}}$ (since $d^*$ is minimal), and condition 3) trivially holds.  But $\mathcal{L}^a(\rZ^*)\neq \mathcal{L}^{\mathfrak{C};\doop(j'\gets q_{j'})}(\rZ)$ since $d^*$ is minimal, and therefore, $\mathfrak{C}$ is not an $\mathcal{I}-\int_S$ valid model of~$\rZ^*$.
\end{proof}

\propComplexity*
\begin{proof}[Proof of \Cref{prop: complexity}]
    Let $a'\in \mathcal{A}$, $\mathcal{I}\subseteq \mathcal{I}^*$, and $d=\doop(j\gets q_j, j\in J)\in \int_K^{\mathcal{I}}(a')$ be given. We want to show that $\mathcal{L}^{a'}(\rZ^*)=\mathcal{L}^{\mathfrak{C};d}(\rZ)$.

    Since $d\in \int_S^{\mathcal{I}}(a')$ (by definition of $\int_K$), and since $\mathcal{L}^a(\rZ^*)$ has the same finite support for every $a$, we have for every $i\in J$ that $p_i^{\mathfrak{C};g(a')}(z_i\mid \pa_i)=p_i^{\mathfrak{C};d}(z_i\mid \pa_i)$ for every $z_i$ and $\pa_i$ in the support.

  Since $g^{-1}(\{d\})$ is nonempty (because $g: \mathcal{A}\setminus \{\mathcal{O}\} \to \mathcal{I}^*$ is surjective and $d\in\mathcal{I}\subseteq \mathcal{I}^*$) and $d\in \int^\mathcal{I}_S(a)$ for every $a\in g^{-1}(\{d\})$, we have that $K(a')\leq t(|J|)$, which implies that $|g(a')|\leq |J|$.
    
    Since every intervention in  $\mathcal{I}$ 
    is minimal, $d$ is minimal, and we have for every $i\in J$ that $p_i^{\mathfrak{C};g(a')}(z_i\mid \pa_i)\neq p_i^{\mathfrak{C}}(z_i\mid \pa_i)$ for some $z_i$ and $\pa_i$ in the support. Since $|g(a')|\leq |J|$, this means that $g(a')$ is an intervention on the nodes in $J$, and only those.
    Therefore,
    for $i\notin J$, we must have $p_i^{\mathfrak{C};g(a')}(z_i\mid \pa_i)= p_i^{\mathfrak{C}}(z_i\mid \pa_i)$ for every $z_i$ and $\pa_i$ in the support as these nodes are not intervened upon by $g(a')$. In summary, we have that $p_\rZ^{\mathfrak{C};g(a')}(\bm{z})=p_\rZ^{\mathfrak{C};d}(\bm{z})$ for all $\bm{z}$ in the support.
\end{proof}

\section{Interpretation $\int_P$ satisfies \textbf{D2}}\label{app: intP sat D2}
\begin{proposition}
    $\int_P$ satisfies desideratum \textbf{D2}.
\end{proposition}
\begin{proof}
Let a data-generating process $\mathcal{D}$, representation $\rZ^*$, compatible CBN $\mathfrak{C}$, and set of interventions $\mathcal{I}$ in $\mathfrak{C}$ be given.

Let action $a\in \mathcal{A}$ be arbitrary. We want to show that every intervention in $ \int^{\mathcal{I}}_P(a)$ induces the same distribution.  

Let $\widetilde{J}=\{i\in [n] \mid 
Z_i^*\indep \PA_i^* \text{ in } 
\mathcal{L}^a(\rZ^*)\}$, and let 
$b=\doop(j\gets \widetilde{q}_j,j\in 
\widetilde{J})$, be a perfect intervention 
where each $\widetilde{q}_j$ is given such that 
$\mathcal{L}^a(Z_i^*\mid \PA_i^*)\sim \widetilde{q}_i$ 
for every $i \in \widetilde{J}$.

Let $d=\doop(j\gets q_j,j\in J)\in \int_P^{\mathcal{I}}(a)$ be given arbitrarily (if $\int_P^{\mathcal{I}}(a)$ is empty, there is nothing to show). We now argue that $\mathcal{L}^{\mathfrak{C};d}(\rZ)=\mathcal{L}^{\mathfrak{C};b}(\rZ)$.
For every $i\in 
J$, 
since 
$d$ is perfect and sets correct conditionals on intervened nodes (\textbf{D0}), we have that $Z_i^* \indep \PA_i^*$ in $\mathcal{L}^a(\rZ^*)$ (so $i \in \widetilde{J}$), and thus 
$p_i^{\mathfrak{C};b}$ is also perfect. So for $i\in J$, both $p_i^{\mathfrak{C};b}$ and $p_i^{\mathfrak{C};d}$ are the constant kernel $\pa_i\mapsto \mathcal{L}^a(Z_i^*)$,
and therefore $\mathcal{L}^{\mathfrak{C};d}(Z_i\mid \PA_i)\sim
p_i^{\mathfrak{C};b}$.

Since $d$ satisfies condition 3) 
of \Cref{def: perfect}
(and by the definition of $\widetilde{J}$), every $i\in \widetilde{J}\setminus %
J$ is a source node 
and $\mathcal{L}^{\mathfrak{C};b}(Z_i) = \mathcal{L}^\mathfrak{C}(Z_i)$ which, since $d$ is not intervening on $i\in \widetilde{J}\setminus J$, implies $\mathcal{L}^{\mathfrak{C};d}(Z_i)=\mathcal{L}^{\mathfrak{C};b}(Z_i)$, so $\mathcal{L}^{\mathfrak{C};d}(Z_i\mid \PA_i)\sim 
p_i^{\mathfrak{C};b}$  for every $i\in \widetilde{J}\setminus J$. 

Finally, for every $i\notin 
\widetilde{J}$, we also have $\mathcal{L}^{\mathfrak{C};d}(Z_i\mid \PA_i)\sim p_i^{\mathfrak{C};b}$ since neither $d$ nor $b$ 
can be interventions on nodes outside of 
$\widetilde{J}$ as they are perfect 
interventions. Therefore, $\mathcal{L}^{\mathfrak{C};d}(\rZ)=\mathcal{L}^{\mathfrak{C};b}(\rZ)$, and since $d\in \int^{\mathcal{I}}_P(a)$ was arbitrary, we conclude that every intervention in $\int^{\mathcal{I}}_P(a)$ must induce the same distribution, namely $\mathcal{L}^{\mathfrak{C};b}(\rZ)$. 
\end{proof}

\section{Counterexamples regarding Footnote \ref{fn: counterexamples}}\label{app: counterexample footnote}

Consider interpretation $\int_{\widetilde{S}}$ defined exactly as $\int_S$ in \Cref{def: single node}, except that we drop condition 2). We now argue that \Cref{prop: d2} (1) is invalid for $\int_{\widetilde{S}}$. 

\begin{example}{\textbf{$\int_{\widetilde{S}}$ violates \Cref{prop: d2} (1).}}\label{ex: footnote}
Let $\mathfrak{C}$ be given by graph $X \to Y$ and kernels:
\begin{align*}
    \mathcal{L}^{\mathfrak{C}}(X)&=\text{Ber}(0.5)\\
    \mathcal{L}^{\mathfrak{C}}(Y\mid X=0)&=\text{Ber}(0.4)\\
    \mathcal{L}^{\mathfrak{C}}(Y\mid X=1)&=\text{Ber}(0.6).
\end{align*}
Assume that $(X^*,Y^*)$ is emulated by $\mathfrak{C}$ and interventions $\mathcal{I}^*=\{d^*\}$, where $d^*$ is a single-node intervention on $X$ given by
\begin{align*}
    \mathcal{L}^{\mathfrak{C};d^*}(X)&=\delta_0.
\end{align*}
Now, let $\mathcal{I}=\{d\}$, where $d$ is a single-node intervention on $Y$ given by 
\begin{align*}
    \mathcal{L}^{\mathfrak{C};d}(Y\mid X=0)&=\text{Ber}(0.4)\\
    \mathcal{L}^{\mathfrak{C};d}(Y\mid X=1)&=\text{Ber}(0.7).
\end{align*}
Let $a$ be an action such that $\mathcal{L}^a(X^*,Y^*)=\mathcal{L}^{\mathfrak{C};d^*}(X,Y)$. $d\in \int_{\widetilde{S}}^\mathcal{I}(a)$ since $\mathcal{L}^a(Y^*\mid X^*)\sim p^{\mathfrak{C};d}_{Y\mid X}$, but $\mathcal{L}^{a}(X^*,Y^*)\neq \mathcal{L}^{\mathfrak{C};d}(X,Y)$, contradicting \Cref{prop: d2} (1). Notice that $d\notin \int_{S}^\mathcal{I}(a)$ since $\mathcal{L}^a(Y^*\mid X^*)\sim p^{\mathfrak{C}}_{Y\mid X}$, violating condition 2) of \Cref{def: single node}. 
\end{example}

\section{Details on Example \ref{ex: TC}}\label{app: total choles int C}
We cannot falsify $\mathfrak{C}$ as an $\mathcal{I}-\int_C$ valid model of $(\text{TC}^*,\text{HD}^*)$.
$\doop(\text{TC}\leftarrow \mathcal{N}(1,2))\notin \int_C^{\mathcal{I}}(a')$ since it is not the case that 
$\mathcal{L}^{a'}(\text{HD}^*\mid \text{TC}^*)\sim p_{\text{HD}\mid 
\text{TC}}^{\mathfrak{C};\doop(\text{TC}\leftarrow \mathcal{N}(1,2))}$. 
Under action $a'$, the conditional distribution of heart disease given total cholesterol is different than in the observational regime. Therefore, under interpretation $\int_C$, we could only interpret $a'$ as an imperfect multi-node intervention on total cholesterol and heart disease. Indeed, $\mathfrak{C}$ is an $\mathcal{I}-\int_C$ valid model of $\rZ^*$, where $\doop(\text{TC}\gets \mathcal{N}(t,2))\in \int^{\mathcal{I}}_C(a)$ if $\mathcal{L}^a(\text{TC}^*,\text{HD}^*)=\mathcal{L}^{\mathfrak{A};\doop\left(\text{LDL}\leftarrow \mathcal{N}(\frac{t}{2},1),\text{HDL}\leftarrow \mathcal{N}(\frac{t}{2},1) \right)}(\text{LDL}+\text{HDL},\text{HD})$, that is, an action is only interpreted as a single-node intervention on total cholesterol if the conditional distribution of heart disease given total cholesterol is the same as in the observational regime.

\section{Falsifying a model under $\int_S$}\label{app: ex opposite direction}
\begin{example}{\textbf{a) Falsify a model under $\int_S$, and b) reverse the causal direction.}}\label{ex: opposite direction}
Assume that $\rZ^*=(Z_1^*,Z_2^*)$ is emulated by CBN $\mathfrak{A}$ and interventions $\mathcal{I}^*$. Let $\mathfrak{A}$ be given by
DAG $Z_1\to Z_2$ and kernels 
\begin{align*}
    \mathcal{L}^{\mathfrak{A}}(Z_1)&=\text{Ber}(0.5)\\
    \mathcal{L}^{\mathfrak{A}}(Z_2 \mid Z_1=1)&=\text{Ber}(0.6)\\
    \mathcal{L}^{\mathfrak{A}}(Z_2 \mid Z_1=0)&=\text{Ber}(0.4),
\end{align*}
and let 
\begin{align*}
    \mathcal{I}^*=\{&\doop(Z_1\gets \text{Ber}(0.6), Z_2=1),\\&\doop(Z_1\gets\text{Ber}(0.4), Z_2=0),\\&\doop(Z_1=1,Z_2\gets\text{Ber}(0.5)),\\&\doop(Z_1=0,Z_2\gets \text{Ber}(0.5))\}.
\end{align*}
\paragraph*{a) Falsify $\mathfrak{A}$.}
Let $\mathcal{I}=\{\doop(Z_1=1),\doop(Z_1=0),\doop(Z_2=1),\doop(Z_2=0)\}$ be the set of modeled interventions.
Let $a$ be an action such that $\mathcal{L}^{a}(\rZ^*)=\mathcal{L}^{\mathfrak{A};\doop(Z_1\gets\text{Ber}(0.6), Z_2=1)}(\rZ)$ (which exists since $\doop(Z_1\gets\text{Ber}(0.6), Z_2=1)\in \mathcal{I}^*$).
Now $\mathfrak{A}$ is falsified as an $\mathcal{I}-\int_S$ valid model of $\rZ^*$, for example, because $\doop(Z_2=1)\in \int_S^{\mathcal{I}}(a)$, but $\mathcal{L}^{a}(Z_1^*)=\text{Ber}(0.6)\neq \text{Ber}(0.5)= \mathcal{L}^{\mathfrak{A};\doop(Z_2=1)}(Z_1)$. 
\paragraph*{b) Reversing the causal direction.}
Instead, consider CBN $\mathfrak{C}$ with DAG $Z_1\leftarrow Z_2$ and kernels
\begin{align*}
    \mathcal{L}^{\mathfrak{C}}(Z_2)&=\text{Ber}(0.5)\\
    \mathcal{L}^{\mathfrak{C}}(Z_1 \mid Z_2=1)&=\text{Ber}(0.6)\\
    \mathcal{L}^{\mathfrak{C}}(Z_1 \mid Z_2=0)&=\text{Ber}(0.4).
\end{align*}
    $\mathfrak{C}$ is an $\mathcal{I}-\int_S$ valid model of $\rZ^*$ for every set of interventions $\mathcal{I}$ in $\mathfrak{C}$, 
    even though it has the opposite causal direction to $\mathfrak{A}$. 
    This is true per \Cref{prop: d2} (1) because $\rZ^*$ is emulated by $\mathfrak{C}$ and a set of single-node interventions, namely:
    \begin{align*}
        \mathcal{L}^{\mathfrak{A}; \doop(Z_1\gets\text{Ber}(0.6), Z_2=1)}(Z_1,Z_2)=\mathcal{L}^{\mathfrak{C}; \doop(Z_2=1)}(Z_1,Z_2)\\
        \mathcal{L}^{\mathfrak{A}; \doop(Z_1\gets\text{Ber}(0.4), Z_2=0)}(Z_1,Z_2)=\mathcal{L}^{\mathfrak{C}; \doop(Z_2=0)}(Z_1,Z_2)\\
        \mathcal{L}^{\mathfrak{A}; \doop(Z_1=1, Z_2\gets\text{Ber}(0.5))}(Z_1,Z_2)=\mathcal{L}^{\mathfrak{C}; \doop(Z_1=1)}(Z_1,Z_2)\\
        \mathcal{L}^{\mathfrak{A}; \doop(Z_1=0, Z_2\gets\text{Ber}(0.5))}(Z_1,Z_2)=\mathcal{L}^{\mathfrak{C}; \doop(Z_1=0)}(Z_1,Z_2).
    \end{align*}
In particular, $\rZ^*$ is emulated by both $\mathfrak{A}$ and $\mathfrak{C}$,
but only
$\mathfrak{C}$
is an
$\mathcal{I}-\int_S$ valid model of $\rZ^*$
for
 $\mathcal{I}=\{\doop(Z_1=1),\doop(Z_1=0),\doop(Z_2=1),\doop(Z_2=0)\}$.
The interventions  $\mathcal{I}^*$ in $\mathfrak{A}$ are fine-tuned to mimic single-node interventions in $\mathfrak{C}$. The lack of fine-tuning between interventions has previously been suggested as a possible fundamental property of causal models \citep{janzing2010causal,janzing2016algorithmic}.
\end{example}

\section{Further non-circular interpretations (Section \ref{sec: non-circular})}

\subsection{$\int_{\widetilde{\mathcal{I}},f}$: Violating only \textbf{D3}}\label{app: intervention set}

We provide an interpretation that satisfies \textbf{D0}--\textbf{D4}, except \textbf{D3}. We do not expect that this interpretation will be useful in itself; rather, we provide it to show that \textbf{D0}, \textbf{D1}, \textbf{D2}, and \textbf{D4} do not imply \textbf{D3}. We leave it for future work to investigate if there exist interesting interpretations that may violate \textbf{D3}.

\begin{definition}{\textbf{$\int_{\widetilde{\mathcal{I}},f}$. An interpretation violating only \textbf{D3}.}} \label{def: intervention set}
Let a data-generating process $\mathcal{D}$, representation $\rZ^*$, compatible CBN $\mathfrak{C}$, countable set of interventions $\widetilde{\mathcal{I}}$ in $\mathfrak{C}$, and an injective function $f:\widetilde{\mathcal{I}}\to \mathbb{N}$ be given. In addition, let a set of modeled interventions $\mathcal{I}$ in $\mathfrak{C}$ be given. We define the interpretation $\mathbf{Int}_{\widetilde{\mathcal{I}},f}$ by the following rule: An intervention $d\in \mathcal{I}$ is in $\mathbf{Int}^{\mathcal{I}}_{\widetilde{\mathcal{I}},f}(a)$ if and only if

\begin{enumerate}
\item[a)] $d\in \mathbf{Int}^{\mathcal{I}}_{C}(a)$,
\item[] or
\item[b)] $\mathbf{Int}^{\mathcal{I}}_{C}(a)=\emptyset$ and the following two conditions hold:
\begin{enumerate}
\item[1)] $d\in\widetilde{\mathcal{I}}\cap \int_{S}^{\mathcal{I}}(a)$, see \Cref{def: single node}.
\item[2)]
For every $b\in\mathcal{I}\setminus\{d\}$, if $b \in \widetilde{\mathcal{I}}\cap \int_S^{\mathcal{I}}(a)$, then $f(b)> f(d)$.
\end{enumerate}
\end{enumerate}
\end{definition}
$\mathbf{Int}^{\mathcal{I}}_{C}(a)\subseteq 
\mathbf{Int}^{\mathcal{I}}_{\widetilde{\mathcal{I}},f}(a)$ such that \textbf{D1} (if it behaves like an intervention, it is that intervention)
is satisfied. To ensure that 
$\mathbf{Int}^{\mathcal{I}}_{\widetilde{\mathcal{I}},f}(a)$ is not generally a subset of 
$\mathbf{Int}^{\mathcal{I}}_{C}(a)$, we also let $d\in \mathcal{I}$ be in $\mathbf{Int}^{\mathcal{I}}_{\widetilde{\mathcal{I}},f}(a)$ if $\mathbf{Int}^{\mathcal{I}}_{C}(a)$ is empty and two further conditions are satisfied. Informally, this works as follows:
For an action $a$ with $\mathbf{Int}^{\mathcal{I}}_{C}(a)=\emptyset$, we check which interventions
are in $\widetilde{\mathcal{I}}\cap \int_S^{\mathcal{I}}(a)\subseteq\mathcal{I}$. If  $\widetilde{\mathcal{I}}\cap \int_S^{\mathcal{I}}(a)$ is nonempty, we interpret $a$ as the
least element in $ \widetilde{\mathcal{I}}\cap \int_S^{\mathcal{I}}(a)$, using $f$ as the ordering. This ensures that $\int_{\widetilde{\mathcal{I}},f}$ does not violate \textbf{D2} (an action should not be interpreted as distinct interventions): all interventions in $\mathbf{Int}^{\mathcal{I}}_{C}(a)$ induce the same distribution $\mathcal{L}^a(\rZ^*)$, and if $\mathbf{Int}^{\mathcal{I}}_{C}(a)$ is empty, then $\mathbf{Int}^{\mathcal{I}}_{\widetilde{\mathcal{I}},f}(a)$ contains at most one intervention since $f$ is injective. $\mathbf{Int}_{\widetilde{\mathcal{I}},f}$ satisfies \textbf{D0} (correct conditionals on intervened nodes) and \textbf{D4} (an intervention does not create new dependencies) since $\int_S$ and $\int_C$ satisfy these desiderata. On the other hand, $\int_{\widetilde{\mathcal{I}},f}$ may violate \textbf{D3} (interpretations should not depend on the intervention set $\mathcal{I}$), as shown in the following example.

\begin{example}{\textbf{Interpretation $\int_{\widetilde{\mathcal{I}},f}$ may violate \textbf{D3}.}}\label{ex: intervention set}
 Let $\mathfrak{C}$ be a CBN whose DAG has no edges and with kernels given by
\begin{align*}
    Z_1&\sim \mathcal{N}(0,1)\\
    Z_2&\sim \mathcal{N}(0,1). 
\end{align*}
Let $\mathcal{I}^*=\{\doop(Z_1=0,Z_2=0)\}$, $\widetilde{\mathcal{I}}=\{\doop(Z_1=0),\doop(Z_1=0,Z_2=0)\}$, and define $f:\widetilde{\mathcal{I}}\to \mathbb{N}$ by $f(\doop(Z_1=0,Z_2=0))=1$ and $f(\doop(Z_1=0))=2$. Assume that $(Z_1^*,Z_2^*)$ is emulated by $\mathfrak{C}$ and $\mathcal{I}^*$. 
Consider an action $a\in \mathcal{A}\setminus\{\mathcal{O}\}$; since $\mathcal{I}^*$ has only one element, $\mathcal{L}^a(\rZ^*)=\mathcal{L}^{\mathfrak{C};\doop(Z_1=0,Z_2=0)}(\rZ)$.
Now
$\doop(Z_1=0)\in \int_{\widetilde{\mathcal{I}},f}^{\{\doop(Z_1=0)\}}(a)$, but $\doop(Z_1=0)\notin \int_{\widetilde{\mathcal{I}},f}^{\{\doop(Z_1=0),\doop(Z_1=0,Z_2=0)\}}(a)=\{\doop(Z_1=0,Z_2=0)\}$, contradicting \textbf{D3}. This example also shows that it is possible to falsify a compatible model under $\int_{\widetilde{\mathcal{I}},f}$ since $\doop(Z_1=0)\in \int_{\widetilde{\mathcal{I}},f}^{\{\doop(Z_1=0)\}}(a)$ but $\mathcal{L}^a(Z_2^*)=\delta_0\neq \mathcal{N}(0,1)=\mathcal{L}^{\mathfrak{C};\doop(Z_1=0)}(Z_2)$.
\end{example}

\subsection{$\int_M$: Violating only \textbf{D4} renders all CBNs with complete DAGs interventionally valid}\label{app: markov}

Assume that we have a representation $\rZ^*$ with $\mathcal{L}^a(\rZ^*)\neq \mathcal{L}^{\mathcal{O}}(\rZ^*)$ for all $a\in \mathcal{A}\setminus \{\mathcal{O}\}$ and a compatible and complete CBN $\mathfrak{C}$. If $\int$ is an interpretation that satisfies \textbf{D1}--\textbf{D3}, then $\mathfrak{C}$ is an $\mathcal{I}-\int$ valid model of $\rZ^*$ for every set of interventions $\mathcal{I}$ in $\mathfrak{C}$. The proof of this is the same as the proof of \Cref{prop: impossibiliy}, except that we do not need \textbf{D4} to have that $\mathcal{L}^a(\rZ^*)$ is Markov w.r.t.\ the graph of $\mathfrak{C}$ since this is a complete DAG (and Markovianity trivially holds). 
Therefore, violating only \textbf{D4} is not a viable strategy to avoid circularity if one believes that not all complete and compatible CBNs should be considered interventionally valid. 
Consider the following interpretation that may violate \textbf{D4}, but satisfies \textbf{D0}--\textbf{D3}. 

\begin{definition}{\textbf{$\int_M$. An interpretation violating only D4.}}\label{def: markov}
Let a data-generating process $\mathcal{D}$, representation $\rZ^*$, compatible CBN $\mathfrak{C}$, and set of interventions $\mathcal{I}$ in $\mathfrak{C}$ be given. We define the interpretation $\mathbf{Int}_M$ by the following rule: An intervention $\doop(j\gets q_j,j\in J)\in \mathcal{I}$ is in $\mathbf{Int}_M^{\mathcal{I}}(a)$ if and only if the following two conditions hold:
\begin{enumerate}
\item[1)] $\mathcal{L}^a(Z^*_i \mid \PA_i^*)\sim q_i$
for all $i\in J$.
That is, $\int_M$ satisfies \textbf{D0} (correct conditionals on intervened nodes).
\item[2)]  $\mathcal{L}^a(Z^*_i \mid \PA_i^*)\sim p^{\mathfrak{C}}_{i}$ %
for all $i\notin J$. Intuitively, we do not intervene on nodes not in $J$.
\end{enumerate}
\end{definition}

The $M$ in $\int_M$ is for `Markov'. Interpretation $\mathbf{Int}_M$ is the same as $\mathbf{Int}_C$ except that we do not require that $\mathcal{L}^a(\rZ^*)$ is Markov w.r.t.\ the graph of $\mathfrak{C}$ for $a$ to be interpreted as an intervention. Since $\int_M$ satisfies \textbf{D0}--\textbf{D3} but not \textbf{D4}, interventional validity requires actions to not introduce new dependencies, as shown by the following proposition.
\begin{proposition}\label{prop: intM}
Let a data-generating process $\mathcal{D}$, representation $\rZ^*$, and compatible CBN $\mathfrak{C}$ with graph $\mathcal{G}$ be given. 

    (1) If $\mathcal{L}^{a}(\rZ^*)$ is Markov w.r.t.\ $\mathcal{G}$ for all $a$, then $\mathfrak{C}$ is an $\mathcal{I}-\int_M$ valid model of $\rZ^*$ for every set of interventions $\mathcal{I}$ in $\mathfrak{C}$.
    
    (2) If there exist an action $a$ such that $\mathcal{L}^{a}(\rZ^*)$ is not Markov w.r.t.\ $\mathcal{G}$ and an intervention $d=\doop(j\gets q_j, j\in [n])$ in $\mathfrak{C}$ such that $\mathcal{L}^a(Z_j^*\mid \PA_j^*)\sim q_j$ for all $j\in [n]$, then there exists a set of interventions $\mathcal{I}$ in $\mathfrak{C}$ such that $\mathfrak{C}$ is not an $\mathcal{I}-\int_M$ valid model of $\rZ^*$.
\end{proposition}
\begin{proof}

    (1) Let action $a$ and set of interventions $\mathcal{I}$ in $\mathfrak{C}$ be given. Assume that $d\in \int_M^{\mathcal{I}}(a)$. We want to show $\mathcal{L}^a(\rZ^*)=\mathcal{L}^{\mathfrak{C};d}(\rZ)$. The proof is the same as for \Cref{prop: tautology1}.
    
    (2) Let $a$ be an action and $d=\doop\left(j\gets q_j, j\in [n]\right)$ an intervention such that $\mathcal{L}^a(Z_j^*\mid \PA_j^*)\sim q_j$ for all $j\in [n]$ and $\mathcal{L}^a(\rZ^*)$ is not Markov w.r.t.\  
    $\mathcal{G}$. Let $\mathcal{I}=\{d\}$. Now $d \in \int_M^{\mathcal{I}}(a)$ but $\mathcal{L}^a(\rZ^*)\neq \mathcal{L}^{\mathfrak{C};d}(\rZ)$ since $\mathcal{L}^a(\rZ^*)$ is not Markov w.r.t.\ $\mathcal{G}$, while $\mathcal{L}^{\mathfrak{C};d}(\rZ)$ is Markov w.r.t.\ $\mathcal{G}$.
\end{proof}
We now provide an example of how $\int_M$ might falsify a model.

\begin{example}{\textbf{Falsifying a model under interpretation $\int_M$.}}
    
    Let the set of actions be $\mathcal{A}=\{\mathcal{O},a\}$. Assume that $\mathcal{L}^{\mathcal{O}}(X^*,Y^*,Z^*)$ is given by the observational distribution induced by the SCM
    \begin{align*}
        &X^*:=\mathcal{E}_1\\
        &Y^*:=X^*+\mathcal{E}_2\\
        &Z^*:=\mathcal{E}_3
    \end{align*}
    with $\mathcal{E}_1,\mathcal{E}_2,\mathcal{E}_3\overset{\text{iid}}{\sim}\mathcal{N}(0,1)$, and assume that $\mathcal{L}^{a}(X^*,Y^*,Z^*)$ is given by the observational distribution induced by the SCM
    \begin{align*}
        &X^*:=\mathcal{E}_1\\
        &Y^*:=\frac{X^*}{2}+\frac{Z^*}{\sqrt{2}}+\frac{\mathcal{E}_2}{\sqrt{2}}\\
        &Z^*:=\mathcal{E}_3
    \end{align*}
    with $\mathcal{E}_1, \mathcal{E}_2,\mathcal{E}_3\overset{\text{iid}}{\sim}\mathcal{N}(0,1)$.
    Consider the compatible CBN $\mathfrak{C}$ given by the graph
\begin{center}
\begin{tikzpicture}
    \node[draw, circle] (X) at (0,0) {$X$};
    \node[draw, circle] (Y) at (2,0) {$Y$};
    \node[draw, circle] (Z) at (4,0) {$Z$};
    \draw[->,line width=1.5pt] (X) -- (Y);
\end{tikzpicture}
\end{center}
and kernels
    \begin{align*}
        &X\sim \mathcal{N}(0,1)\\
        &Y\mid X=x \sim \mathcal{N}(x,1)\\
        &Z \sim \mathcal{N}(0,1).
    \end{align*}
    Let $\mathcal{I}=\{\doop\left(Y \gets \mathcal{N}\left(\frac{x}{2},1\right)\right)\}$.\footnote{By $\mathcal{N}\left(\frac{x}{2},1\right)$ we denote the kernel $x\mapsto \mathcal{N}\left(\frac{x}{2},1\right)$, resulting in an imperfect intervention.
    }
    Now, $\doop\left(Y \gets \mathcal{N}\left(\frac{x}{2},1\right)\right)\in \int_M^{\mathcal{I}}(a)$ since  $\mathcal{L}^a(X^*)= \mathcal{N}(0,1)$, $\mathcal{L}^a(Y^*\mid X^*)\sim \mathcal{N}\left(\frac{x^*}{2},1\right)$, and $\mathcal{L}^a(Z^*)=\mathcal{N}(0,1)$,
    in line with the corresponding conditionals in $\mathcal{L}^{\mathfrak{C};\doop(Y\gets\mathcal{N}(\frac{x}{2},1))}(X,Y,Z)$.
    But $$\mathcal{L}^{a}(X^*,Y^*,Z^*)\neq\mathcal{L}^{\mathfrak{C}; \doop\left(Y\gets \mathcal{N}\left(\frac{x}{2},1\right)\right)}(X,Y,Z),$$ for example, because 
    $Y^*\nindep Z^*$ in $\mathcal{L}^{a}(X^*,Y^*,Z^*)$ while $Y\indep Z$ in $\mathcal{L}^{\mathfrak{C}; \doop\left(Y\gets \mathcal{N}\left(\frac{x}{2},1\right)\right)}(X,Y,Z)$. Intuitively, action $a$ behaves like an intervention that introduces 
    probabilistic dependence
    between variables $Y$ and $Z$. 
\end{example}

\section{Discussion of implications for related research (Section \ref{sec: related work})}

\subsection{Further details on causal representation learning (\Cref{sec: CRL})}

\begin{example}{\textbf{Identifiability does not imply interventional validity.}} \label{ex: CRL}
    Assume a data-generating process where $\mathcal{A}=\{\mathcal{O},a_1,a_2,a_3\}$ and that the latent variables $(A^*,B^*)$ have the following distributions
    \begin{align*}
        \mathcal{L}^\mathcal{O}(A^*,B^*)&=\mathcal{N}\left(\begin{pmatrix}
            0\\
            0
        \end{pmatrix}, \begin{pmatrix}
            1 & \frac{1}{2}\\
            \frac{1}{2} & 1
        \end{pmatrix}\right)\\
        \mathcal{L}^{a_1}(A^*,B^*)&=\mathcal{N}\left(\begin{pmatrix}
            1\\
            \frac{1}{2}
        \end{pmatrix}, \begin{pmatrix}
            1 & \frac{1}{2}\\
            \frac{1}{2} & 1
        \end{pmatrix}\right)\\
        \mathcal{L}^{a_2}(A^*,B^*)&=\mathcal{N}\left(\begin{pmatrix}
            0\\
            1
        \end{pmatrix}, \begin{pmatrix}
            1 & 0\\
            0 & 1
        \end{pmatrix}\right)\\
            \mathcal{L}^{a_3}(A^*,B^*)&=\mathcal{N}\left(\begin{pmatrix}
            1\\
            1
        \end{pmatrix}, \begin{pmatrix}
            1 & \frac{1}{2}\\
            \frac{1}{2} & 1
        \end{pmatrix}\right)
    \end{align*} 
    Let $\mathfrak{A}$ be a CBN with graph $A\to B$ and $\mathcal{I}^*_{\mathfrak{A}}$, consisting only of minimal interventions, be such that $(A^*,B^*)$ is emulated by $\mathfrak{A}$ and $\mathcal{I}^*_\mathfrak{A}$  with link $g_{\mathfrak{A}}:\mathcal{A}\setminus\{\mathcal{O}\}\to \mathcal{I}^*_{\mathfrak{A}}$.
    From $\mathcal{L}^{a_1}(A^*,B^*)$, $\mathcal{L}^{a_2}(A^*,B^*)$, and the fact that $\mathcal{I}^*_\mathfrak{A}$ only contains minimal interventions, we can deduce that $g_{\mathfrak{A}}(a_1)=\doop(A\gets \mathcal{N}(1,1))$ and $g_{\mathfrak{A}}(a_2)=\doop(B\gets \mathcal{N}(1,1))$ are perfect single-node interventions, and that
    $g_{\mathfrak{A}}(a_3)$ is a multi-node intervention.
    Assume that we observe $\rX^*=f(A^*,B^*)$ for a linear injective
    mixing function $f$ in environments $E=\{\mathcal{O},a_1,a_2\}$. Following \citet{squires2023linear}, we assume that the non-observational environments, $\{a_1,a_2\}$, correspond to perfect single-node interventions.\footnote{Specifically $\mathcal{L}^{a_1}(A^*,B^*)=\mathcal{L}^{\mathfrak{A};\doop(A\gets \mathcal{N}(1,1))}(A,B)$ and $\mathcal{L}^{a_2}(A^*,B^*)=\mathcal{L}^{\mathfrak{A};\doop(B\gets \mathcal{N}(1,1))}(A,B)$.}
    Since this assumption is met here, we can identify $\mathfrak{A}$ up to permutation and scaling \citep{squires2023linear}.   
    But the CBN $\mathfrak{A}$ with graph $A\to B$ is neither $\{g_{\mathfrak{A}}(a_1),g_{\mathfrak{A}}(a_2)\}-\int_S$ nor $\{g_{\mathfrak{A}}(a_1),g_{\mathfrak{A}}(a_2)\}-\int_P$ valid.\footnote{%
    To see this, notice, for example, that $g_{\mathfrak{A}}(a_1)=\doop(A\gets \mathcal{N}(1,1))\in \int_S^{\{g_{\mathfrak{A}}(a_1),g_{\mathfrak{A}}(a_2)\}}(a_3)$ and $g_{\mathfrak{A}}(a_1)=\doop(A\gets \mathcal{N}(1,1))\in \int_P^{\{g_{\mathfrak{A}}(a_1),g_{\mathfrak{A}}(a_2)\}}(a_3)$, that is, $a_3$ is interpreted as the intervention $g_\mathfrak{A}(a_1)$, while the distributions $\mathcal{L}^{a_3}(A^*,B^*)\neq \mathcal{L}^{\mathfrak{A};g_{\mathfrak{A}}(a_1)}(A,B)$ do not match.}
    Therefore, to ensure that identifiability implies interventional validity, we need additional assumptions.
\end{example}

\subsubsection{Desiderata for causal representation learning}\label{app: desiderata}
\begin{enumerate}
    
    \p We may want interventions in $\mathcal{I}$ to be implementable, that is, that for every intervention $d \in \mathcal{I}$ there exists an action $a$ such that $d \in \int^{\mathcal{I}}(a)$.\\
     This criterion is reminiscent of the idea that causes must be manipulable \citep{cartwright2007hunting,glymour2014commentary,PearlObesity2018,PearlDo2019}.

    \p We may want that for each variable there exists a single-node intervention in $\mathcal{I}$ that intervenes on it. 
    \newline
    This criterion is reminiscent of the idea of autonomy \citep{aldrich1989autonomy}.
     See \citet{janzing2024phenomenological} for a related requirement that we discuss in \Cref{sec: pheno}.

    \p  We may want to model only some interventions on specific nodes, that is, $\mathcal{I}$ should be restricted in some way.\\
    See, for example, \citet{Dawid2021}.
    \p We may want the transformation $h$ to disregard little or no information.
    \newline
    See \citet{bengio2013representation} for a review of representation learning 
    where this is discussed as a desirable criterion.
    \p We may want the representation to have a prespecified number of nodes.
    \newline
    See, for example, \citet{kekic2024targeted}.
    
    \p We may want the representation to contain a prespecified aspect of the data-generating process, that is, parts of $h$ may be prespecified.
    \newline
See, for example, \citet{chalupka2015} and \citet{weichwald2016merlin}.

    \p We may want the resulting representation to be useful for (computationally or statistically efficiently) solving a prespecified set of downstream tasks.
    \newline
    See, for example, \citet{kinney2020causal,gultchin21a, lachapelle2023synergies,dyer2024a,saengkyongam2024identifying, christgau2024efficient}.
\end{enumerate}

\subsection{Further details on causal discovery (\Cref{sec: CD})}\label{app: CD}

\begin{example}{\textbf{Total cholesterol \Cref{ex: TC} continued. Maybe no causal graph is adequate.}}  \label{ex: causal discovery}
Assume again that $(\text{LDL}^*,\text{HDL}^*,\text{HD}^*)$ is emulated by $\mathfrak{A}$ and interventions $\mathcal{I}^*=\{\doop(\text{LDL}\leftarrow \mathcal{N}(x,1),\text{HDL}\leftarrow \mathcal{N}(y,1))\mid x,y\in \mathbb{R},\ (x,y)\neq (0,0)\}$, and let $(\text{TC}^*,\text{HD}^*)=(\text{LDL}^*+\text{HDL}^*,\text{HD}^*)$.
In this example, we argue that the representation $(\text{TC}^*,\text{HD}^*)$ is not a causal representation given interpretation $\int_S$ and set of perfect interventions $\mathcal{I}=\{\doop(\text{TC}\gets \mathcal{N}(1,2)),\doop(\text{HD}\gets \mathcal{N}(1,6))\}$ that change 
the mean of total cholesterol or heart disease, respectively. Observationally, total cholesterol and heart disease are correlated. Therefore, every compatible CBN must have an edge between those two nodes, that is, either $\text{TC}\to \text{HD}$ or  $\text{TC}\gets \text{HD}$.

Let $\mathfrak{C}$ be a compatible CBN with graph $\text{TC}\to \text{HD}$. Then, similar to  \Cref{ex: TC},  we have an action $a$ such that $\doop(\text{TC}\gets \mathcal{N}(1,2))\in \int_S^{\mathcal{I}}(a)$, but $\mathbb{E}^{a}(\text{HD}^*)=2\neq \frac{1}{2}=\mathbb{E}^{\mathfrak{C};\doop(\text{TC}\gets \mathcal{N}(1,2))}(\text{HD})$. 

We now argue that having heart disease cause total cholesterol also results in an invalid model. Let $\mathfrak{H}$ be a CBN with graph $\text{TC}\gets \text{HD}$ and kernels given by
\begin{align*}
    \mathcal{L}^{\mathfrak{H}}(\text{HD})&=\mathcal{N}(0,6)\\
    \mathcal{L}^{\mathfrak{H}}(\text{TC}\mid \text{HD}=h)&= \mathcal{N}\left(\frac{h}{6},\frac{11}{6}\right),
\end{align*}
such that
$\mathcal{L}^{\mathfrak{H}}(\text{TC}, \text{HD})=\mathcal{N}\left(\begin{pmatrix}
        0\\
        0
    \end{pmatrix},\begin{pmatrix}
        2 & 1\\
        1 & 6
    \end{pmatrix}\right)$,
and so $\mathfrak{H}$ is compatible.
Let $a$ be an action such that
\begin{align*}
    &\mathcal{L}^a(\text{LDL}^*,\text{HDL}^*,\text{HD}^*)\\
    &\ =\mathcal{L}^{\mathfrak{A};\doop(\text{LDL}\gets \mathcal{N}\left(\frac{1}{2},1\right),\text{HDL}\gets \mathcal{N}(0,1))}(\text{LDL},\text{HDL},\text{HD}).
\end{align*}
Now, $\doop(\text{HD}\gets \mathcal{N}(1,6))\in \int_S^{\mathcal{I}}(a)$, but the distributions do not match, for example, because 
$\mathbb{E}^a(\text{TC}^*)=\frac{1}{2}$ while
$\mathbb{E}^{\mathfrak{H};\doop(\text{HD}\gets \mathcal{N}(1,6))}(\text{TC})=\frac{1}{6}$.

  In summary, there does not exist an
  $\mathcal{I}-\int_S$ valid model of $(\text{TC}^*,\text{HD}^*)$. 
\end{example}

\paragraph*{Identifiability does not imply interventional validity in causal discovery.}
In \Cref{ex: causal discovery}, we argue that the representation $(\text{TC}^*,\text{HD}^*)$ does not admit an interventionally valid causal model (for a given interpretation and set of interventions). In this example, it is also unclear which assumptions one would use to identify the causal direction.  Unfortunately, identifiability does not ensure interventional validity. 
For example, if the representation $\rZ^*$ is emulated by $\mathfrak{C}$ and $\mathcal{I}^*$ with link $g$ and the observed environments $\{\mathcal{L}^{a}(\rZ^*)\}_{a\in E}$, $E\subseteq \mathcal{A}$, correspond to single-node interventions in $\mathfrak{C}$, then we can identify $\mathfrak{C}$ \citep{eberhardt2006}. But this does not guarantee that $\mathfrak{C}$ is a $g(E)-\int$ valid model of $\rZ^*$ for the non-circular interpretations $\int_S$ and $\int_P$, see \Cref{ex: CRL}.
Work on causal discovery focuses on identifiability of an emulating CBN $\mathfrak{C}$
since it is implicitly assumed that $\mathfrak{C}$ is a valid model of $\rZ^*$ if $\rZ^*$ is emulated by $\mathfrak{C}$ and interventions in $\mathfrak{C}$.
One of the main contributions of this paper is to clearly distinguish the notions of `emulated by' (\Cref{def: generation}) and `interventional validity' (\Cref{def: interventionally valid}); see also \Cref{prop: intA,prop: d2,prop: complexity}.

\subsection{Causal abstraction}\label{app: abstraction}

Causal abstraction is about transforming one causal model into another causal model. Which constraints such transformations ought to satisfy has been up for debate \citep{rubenstein2017causal,beckers2019abstracting,otsuka2022,massidda2023causal,otsuka2024process}. The argumentation has been rooted in intuitions about which models can intuitively be considered abstractions of other models. The aim then has been to find mathematical formalizations that capture these intuitions. In this work, we take a different approach: Instead of considering when one model is a `valid abstraction' of another model, we ask when a model is an interventionally valid model of a representation. Rather than relying on intuitions about what models ought to count as abstractions of other models, our approach suggests that model transformations and abstractions should preserve or induce interventional validity. In this section, we argue that existing notions of abstraction do not necessarily align with the goal of preserving or inducing interventional validity, and that we need  interpretations of actions as interventions to avoid an
infinite regress.

\paragraph*{Existing notions of abstractions may not preserve interventional validity.
} The following example shows that  $\tau$-abstractions \citep{beckers2019abstracting} do not preserve interventional validity. We focus on $\tau$-abstractions since this is the strictest notion among those in the literature that considers a restricted set of low-level interventions. It follows that exact transformations, as presented in \citet{rubenstein2017causal}, also do not preserve interventional validity. 
\begin{example}{\textbf{Transforming an interventionally valid model into an invalid model by a $\tau$-abstraction.}}\label{ex: valid to invalid intS}
    Consider the SCM
    $M_{\rX}$ given by
    \begin{align*}
        X_1&:=U_1\sim \text{Unif}([4]).
    \end{align*}
    Let $\mathfrak{A}$ be a single-node CBN with observational distribution induced by $M_{\rX}$. Assume that $\rX^*$ is emulated by $\mathfrak{A}$ and interventions $\mathcal{I}^*=\{\doop(X_1=x) \mid x\in [4]\}$. Since $\mathcal{I}^*$ only has single-node interventions, $\mathfrak{A}$ is an $\mathcal{I}-\int_S$ valid model of $\rX^*$ for every set of interventions $\mathcal{I}$ in $\mathfrak{A}$, see \Cref{prop: d2} (1). We now present a $\tau$-abstraction
    of $(M_{\rX},\mathcal{I}^*)$ that does not preserve interventional validity. 

    Consider SCM $M_{\rY}$ given by
    \begin{align*}
        Y_1&:=N_1 \\
        Y_2&:=N_2 
    \end{align*}
    with $N_1,N_2\overset{\text{iid}}{\sim}\text{Unif}(\{0,1\})$,
    and let $\tau: [4] \to \{0,1\}^2$ be given by $$x\mapsto (\mathbbm{1}(x=3)+\mathbbm{1}(x=4), \mathbbm{1}(x=2)+\mathbbm{1}(x=4)).$$
    $\tau$ can be viewed as mapping from the integers $[4]$ to a binary representation of those integers. 
    We can verify that $M_{\rY}$ is a $\tau$-abstraction of $(M_{\rX},\mathcal{I}^*)$
    as defined by \citet{beckers2019abstracting}, see \Cref{app: check tau-abstraction}.
    Let $\mathfrak{C}$ be a CBN with graph and observational distribution induced by $M_{\rY}$, and let $\mathcal{I}$ be a set of interventions in $\mathfrak{C}$ such that $\doop(Y_2=0)\in \mathcal{I}$. $\mathfrak{C}$ is compatible with $\rY^*:=\tau(\rX^*)$ but is not 
    an $\mathcal{I}-\int_S$ valid model of $\rY^*$. To see this, consider $a$ such that $\mathcal{L}^a(X_1^*)=\mathcal{L}^{\mathfrak{A};\doop(X_1=1)}(\rX)$ (such an action exists since $\doop(X_1=1)\in \mathcal{I}^*$).
    Now, $\doop(Y_2=0)\in \int_S^\mathcal{I}(a)$
     (since $\mathbb{P}^a(Y^*_{2} = 0) = 1$)
    but $\mathcal{L}^a(\rY^*)\neq \mathcal{L}^{\mathfrak{C};\doop(Y_2=0)}(\rY)$, for example, because
    $\mathbb{P}^{a}(Y_{1}^*=0)=1$
    while
    $\mathcal{L}^{\mathfrak{C};\doop(Y_{2}=0)}(Y_{1})=\text{Unif}(\{0,1\})$.  
\end{example}

Similarly, in \Cref{app: beckers -- intA also not preserved} we show that interventional validity is not preserved by constructive soft abstractions
\citep{massidda2023causal}
under interpretation $\int_P$. While it is possible that some notions of abstractions preserve validity for some interpretations, this would be by coincidence rather than per definition.
In \Cref{ex: abstraction}, we show that we can preserve or induce interventional validity
 by a transformation that is not a valid $\tau$-abstraction.

\paragraph*{Existing notions of abstraction may disallow transforming invalid models into valid models.}
 It may seem puzzling why it would ever be useful to have a high-level model if a low-level model is known. While one motivation may be interpretability, our framework highlights another reason: Maybe we can transform an interventionally invalid model into an interventionally valid model.
 We show an example of this below, providing a new formal argument for why high-level models may be preferable to low-level models (for other motivations for high-level causal models see \citet{hoel2013quantifying,hoel2017map,anand2023,zennaro2024causally}).

\begin{example}{\textbf{Transforming an interventionally invalid model into an interventionally valid model.}} \label{ex: abstraction}
    Consider the SCM $M_{\rX}$ given by 
    \begin{align*}
        X_1&:=U_1\\
        X_2 &:=U_2\\
        X_3&:= X_1+X_2+U_3,
    \end{align*}
    where $U_1,U_2,U_3 \overset{\text{iid}}{\sim} \mathcal{N}(0,1)$. Let $\mathfrak{A}$ be a CBN with graph and observational distribution induced by $M_{\rX}$. Assume that $\rX^*$ is emulated by $\mathfrak{A}$ and
    interventions $\mathcal{I}^*=\{\doop(X_1=x_1,X_2=x_2), \doop(X_1=x_1), \doop(X_2=x_2) \mid x_1,x_2\in \mathbb{R}\}$. $\mathfrak{A}$ is not an $\mathcal{I}^*-\int_S$ valid model of $\rX^*$. To see this, fix $x_1,x_2\in \mathbb{R}$ and consider an action $a$ such that $\mathcal{L}^a(\rX^*)=\mathcal{L}^{\mathfrak{A};\doop(X_1=x_1,X_2=x_2)}(\rX)$ (which exists since $\doop(X_1=x_1,X_2=x_2)\in\mathcal{I}^*$). Now $\doop(X_1=x_1)\in \int_S^{\mathcal{I}^*}(a)$, but $\mathcal{L}^{a}(\rX^*)\neq \mathcal{L}^{\mathfrak{A};\doop(X_1=x_1)}(\rX)$, for example, because $\mathcal{L}^{a}(X_2^*)$ has point mass while $\mathcal{L}^{\mathfrak{A};\doop(X_1=x_1)}(X_2)$ is a normal distribution.  

    Instead, consider now
    the transformation $\tau: \mathbb{R}^3\to \mathbb{R}^2, (x_1,x_2,x_3)\mapsto (x_1+x_2,x_3)$, and the SCM $M_{\rY}$ given by
    \begin{align*}
        Y_1&:=\sqrt{2}N_1\\
        Y_2&:=Y_1+N_2,%
    \end{align*}
    where $N_1, N_2\overset{\text{iid}}{\sim} \mathcal{N}(0,1)$. Let $\mathfrak{C}$ be a CBN with graph and observational distribution induced by $M_{\rY}$. $\mathfrak{C}$ is an  $\mathcal{I}-\int_S$ (and $\mathcal{I}-\int_P$)  valid model of $\tau(\rX^*)$ for every set of interventions $\mathcal{I}$ in $\mathfrak{C}$.
    But $M_{\rY}$ is not a $\tau$-abstraction of  $(M_\rX,\mathcal{I}^*)$
    as defined by \citet{beckers2019abstracting}, see \Cref{app: beckers}.
    This shows that sometimes a transformation can induce interventional validity without being a $\tau$-abstraction.
    
    If instead $\mathcal{I}^*=\{\doop(X_1=x_1),\doop(X_2=x_2) \mid x_1,x_2\in \mathbb{R}\}$, then $\mathfrak{A}$ would be an  $\mathcal{I}-\int_S$ (and $\mathcal{I}-\int_P$) valid model of $\rX^*$ for every set of interventions $\mathcal{I}$ in $\mathfrak{A}$. But $M_{\rY}$ would still not be a $\tau$-abstraction of $(M_{\rX},\mathcal{I}^*)$. 
    This shows that sometimes a transformation can preserve interventional validity without being a $\tau$-abstraction.
\end{example}

We think that the perspective of preserving
or inducing interventional validity is useful to rigorously ground the notion of valid model transformations.
We now argue that the foundation of causal abstraction is dubious without an explicit interpretation.

\paragraph*{Causal abstraction rests on an infinite regress.}
In existing works on abstraction \citep{rubenstein2017causal, beckers2019abstracting, beckers2020approximate,rischel2021compositional,massidda2023causal,xia2024neural} there is a map $\omega: \mathcal{I}_L\to \mathcal{I}_H$ between interventions in the low-level model and the high-level model. Implicitly, this suggests an interpretation $\int$ such that for all  $a\in \mathcal{A}$, $d_H\in \int^{\mathcal{I}_{H}}(a)$ if and only if $d_L\in \int^{\mathcal{I}_{L}}(a)$ for some $d_L\in \omega^{-1}(\{d_H\})$. But to determine if $d_L\in \int^{\mathcal{I}_{L}}(a)$ we would presumably need yet another model on an even lower level, leading to an infinite regress. The definition of interventions by \citet{woodward2005making} suffers from an analogous problem as explained by \citet{baumgartner2009interdefining,baumgartner2013identifying}. In some concrete applications, this potentially infinite regress may come to a halt at a level of abstraction where there is no ambiguity about which actions constitute interventions. One example of this is the work on causal abstraction of artificial neural networks, where there seems to be no ambiguity about what constitutes an intervention on the level of neuron activations \citep{geiger2021causal,geiger2022inducing,geiger2024causal,geiger2024finding}. We think this is the exception rather than the rule; in most applications, there is no level of abstraction where interventions are non-ambiguous. The case of neural networks is peculiar because the network is implemented to literally be a causal model
and the graph is given by the network topology. This means that on the level of neurons, we can adopt the circular interpretation $\int_C$ (\Cref{def: tautology}); it does not matter that other causal models of the neural network are also interventionally valid under interpretation $\int_C$ because we have prior justification to regard one of them as the causal model. Since this case is an exception, we usually need an interpretation that does not lead to a (potentially infinite) regress. The interpretations presented in this work avoid the regress by depending only on the distribution of the variables on the one given modeling level (and potentially taking the complexity of actions into account).

\subsection{Logic of conditionals}\label{sec: conditionals}

In this section, we clarify the connections between causal models and the logic of conditionals.
The connection between causal models and (counterfactual) conditionals has received attention from researchers questioning the use of causal models in algorithmic fairness \citep{Hu2020, Kasirzadeh2021}.
Our work is, as far as we know, the first to spell out a precise connection between interpretations of causal models and different analyses of conditionals. 
It turns out that considerations from the philosophical literature on conditionals are relevant for how we interpret causal models,
contrary to what is suggested, for example, in \citet{Pearl2009} (see below).

\paragraph*{The material conditional.}
In \Cref{sec:what-wrong} we considered the proposition \Paste{prop}
Since we assume that Sofia is right that $A$ causes $B$,
the CBN implies $\mathcal{L}^{\doop(B=5)}(A,B) =  \mathcal{N}(0,1) \otimes \delta_5$,
so (P) should be a true proposition.
The proposition has the form of a conditional, that is, a proposition of the form `If $p$, then $q$' \citep{sep-logic-conditionals}. 
In mathematics, it is common to interpret conditionals as the material conditional `$p\Rightarrow q$',
where $p\Rightarrow q \equiv \neg p \vee q$.\footnote{We use `$\Rightarrow$' rather than the more commonly used `$\to$' to distinguish from graph notation.} The material conditional does not provide the correct analysis of propositions like (P); this can be seen by considering the analogous proposition
    \begin{enumerate}
    \item[(P')] If you intervene $\doop(B=6)$, then you will observe the distribution $\mathcal{N}(0,1)\otimes \delta_5$ over $(A,B)$.
\end{enumerate}
Since $\mathcal{L}^{\doop(B=6)}(B) = \delta_6 \neq \delta_5$, (P') should be a false proposition.
Furthermore, (P') should be false regardless of whether anyone intervened $\doop(B=6)$.
For example, if someone intervened $\doop(B=7)$ that would not make (P') true. 
However, if we analyze (P') as a material conditional, then (P') is true if you do not perform the intervention $\doop(B=6)$, that is, $p\Rightarrow q$ is true if $p$ is false.

Another way to see that (P') is not a material conditional is that both (P') and the reverse
\begin{enumerate}
    \item[(P'')] If you observe the distribution $\mathcal{N}(0,1)\otimes \delta_5$ over $(A,B)$, then you intervened $\doop(B=6)$. 
\end{enumerate}
should be false. But this is not possible if (P') is a material conditional since it is a tautology that $(p\Rightarrow q) \vee (q \Rightarrow p)$. These are some of the `paradoxes of material implication'.
Since (P') is apparently not a material conditional, by analogy, (P) also is not a material conditional.\footnote{That causal statements are not to be analyzed as material conditionals has, in the words of \citet{shoham1990nonmonotonic}, ``been taken into account by all philosophers interested in the subject.''}

\paragraph*{Strict implication.} An alternative interpretation of propositions like (P) is that they express strict implication \citep{Lewis1912implication, Zach2019}. Using the notation of modal logic, strict implication $\strictif$ is defined by $p \strictif q \equiv \square (p \Rightarrow q)$ which means that the material conditional $p \Rightarrow q$ is true in every `accessible world' \citep{sep-logic-conditionals}. Consider the interpretation of (P) as (P1):
\Paste{p1} 
We regard the `accessible worlds' as the distributions $\{\mathcal{L}^a(A,B)\}_{a\in \mathcal{A}}$ induced by the available actions $\mathcal{A}$. If we interpret (P1) as a strict implication 
with this set of  accessible worlds, then (P1) means that 

\begin{enumerate}[label=\hspace{0.1cm}(P-strict), leftmargin=*, align=left]
    \item For every $a\in \mathcal{A}$, $(\mathcal{L}^a(B)=\delta_5) \Rightarrow (\mathcal{L}^a(A,B)=\mathcal{N}(0,1)\otimes \delta_5).$
\end{enumerate}

In the introductory example, see \Cref{sec: dialogue}, this turned out to be false since there apparently was an action $a$ such that $\mathcal{L}^a(A,B)=\delta_1\otimes \delta_5$.
Interpreting (P1) as 
(P-strict) is analogous to interpretation $\int_S$.
Someone might object that (P1) should not be interpreted as a strict implication and below we provide arguments against interpreting (P1) as (P-strict).

 \paragraph*{Why (P-strict) is probably not the correct analysis of (P).} Assume for the sake of argument that (P-strict) provides the correct analysis of (P). Then, by analogy,
\begin{enumerate}[label=\hspace{0.1cm}(D-strict), leftmargin=*, align=left]
    \item For every $a\in \mathcal{A}$, $(\mathcal{L}^a(A,B)=\delta_0 \otimes \delta_5) \Rightarrow (\mathcal{L}^a(A,B)=\mathcal{N}(0,1)\otimes \delta_5).$
\end{enumerate}
provides the correct analysis of 
\begin{enumerate}
    \item[(D)] If you intervene $\doop(A=0,B = 5)$, then you will observe the distribution
$\mathcal{N}(0,1) \otimes \delta_5$ over $(A, B)$.
\end{enumerate}
Since we assume that $A\to B$,
(P) is true. Since we assume that (P-strict) provides the correct analysis of (P), (P-strict) must also be true. (D) is false and since (D-strict) is supposed to provide an analysis of (D), (D-strict) must also be false. But now we have a contradiction since (P-strict) implies (D-strict). This follows from monotonicity of strict implication: $p \strictif q \models (p\wedge d)\strictif q$. Here, (P-strict) implies (D-strict) because $\mathbb{P}(A=0,B=5)=1$ if and only if $\mathbb{P}(A=0)=1 \wedge \mathbb{P}(B=5)=1$. Therefore, if (P1) is to provide the correct analysis of (P), we must interpret (P1) as a conditional that does not satisfy monotonicity.

\paragraph*{Minimal change semantics.}
\citet{Stalnaker1968} and \citet{Lewis1973}
used this type of argument to show that subjunctive conditionals are generally not strict implications
(though this is not undisputed \citep{kai2001,gillies2007counterfactual,williamson2020suppose}). 
Here is an example from \citet{Lewis1973}:
\begin{quote}
 If Otto had come, it would have been a lively party; but if both Otto and Anna had come, it would have been a dreary party; but if Waldo had come as well, it would have been lively; but ... \citep[Page 10]{Lewis1973}
\end{quote}
Since these sentences are felicitous, they cannot express strict implications. As an alternative \citet{Stalnaker1968} and \citet{Lewis1973} proposed an analysis based on minimal change semantics. In essence, this approach considers a subjunctive conditional to be true if the consequent is true in the closest possible world where the antecedent is true, see \citet{Zach2019} for a simple exposition. If we measure `closeness of possible worlds' by the complexity $K$ of the action, then this approach is analogous to using interpretation $\int_K$. In contrast to $\int_S$, $\int_K$ is a non-monotonic interpretation: Consider two interventions $\doop(j\leftarrow q_j,j \in J), \doop(j\leftarrow q_j,j \in J')\in \mathcal{I}$ 
with $J\supsetneq J'$. Then $\doop(j\leftarrow q_j,j \in J)\in \int_S^{\mathcal{I}}(a)$ implies that $\doop(j\leftarrow q_j,j \in  J') \in \int_S^{\mathcal{I}}(a)$.
This inference is not valid for $\int_K$ since the actions that change the conditional distributions for all nodes in $J$ may be more complex than actions that change the conditional distribution only for nodes in $J'$. The drawback of using an interpretation like $\int_K$ is that we need a theory of complexity of actions, or, alternatively, of similarity among worlds.
This approach is rife with difficulties \citep{ Fine1975, Lewis1979,Hajek2014}.\footnote{One seemingly simple way to define similarity of worlds is using the causal model. This is suggested in 
 \citet{galles1998axiomatic}, where they write ``In essence, causal models define an obvious distance measure among worlds $d(w,w')$, given by the minimal number of local interventions needed for transforming $w$ into $w'$.'' This idea is problematic: If the distances among worlds are given by the causal model, then it amounts to using the circular interpretation $\int_C$.} According to Pearl 
\begin{quote}
Such difficulties do not enter the structural account. In contrast with Lewis's theory, counterfactuals [including interventional claims] are not based on an abstract notion of similarity among hypothetical worlds; instead, they rest directly on the mechanisms [...] that produce those worlds and on the invariant properties of those mechanisms. [...] [S]imilarities and priorities -- if they are ever needed --  may be read into the \textit{do}($\cdot$) operator as an afterthought [...], but they are not basic to the analysis. \citep[][page 239-240]{Pearl2009}
\end{quote}

If one relies on an interpretation like $\int_K$ these notions are in fact basic to the analysis. 
Using the $\doop(\cdot)$ operator to refer both to operations in a mathematical model and to actions in the world,
risks obfuscating these fundamental issues, rather than solving them.%
\footnote{One example of this conflation appears in \citet{galles1998axiomatic}: ``[D]efine the action $\doop(X=x)$ as the minimal change in [the causal model] $M$  required to make $X=x$ [...].''}

\subsection{Relation to Phenomenological causality}\label{sec: pheno}

\citet{janzing2024phenomenological} 
present a related approach to grounding causality in actions. 
 Translated into the notation of the present work, Definition 3 in \citet{janzing2024phenomenological} states that a graph $\mathcal{G}$ is a valid causal graph over variables $\rZ^*=(Z_1^*,\dots,Z_n^*)$ relative to a set of actions $\mathcal{A}$ if (1) $\mathcal{L}^{\mathcal{O}}(\rZ^*)$ is Markov w.r.t.\ $\mathcal{G}$, and (2) $\mathcal{A}\setminus \{\mathcal{O}\}=\cup_{j=1}^n \mathcal{A}_j$,  where each $\mathcal{A}_j$ is nonempty, and for every $a\in \mathcal{A}_j$, 
\begin{align*}
    &\mathcal{L}^a(Z_j^* \mid \PA_j^*)\neq \mathcal{L}^\mathcal{O}(Z_j^* \mid \PA_j^*), \text{ and}\\
    &\mathcal{L}^a(Z_i^* \mid \PA_i^*)= \mathcal{L}^\mathcal{O}(Z_i^* \mid \PA_i^*) \text{ for $i\neq j$. }
\end{align*}
Interpreting this definition within our framework of interventional validity, it is similar to $\int_S$ in \Cref{def: single node}.\footnote{Since \citet{janzing2024phenomenological} also require that $\mathcal{A}$ only consists of `elementary actions', Definition~3 of \citet{janzing2024phenomenological} also bears resemblance to $\int_K$ (\Cref{def: complexity}). In our work, we do not require that $\mathcal{A}$ only consists of `elementary actions', and instead draw a formal distinction between $\int_K$ and $\int_S$.} Indeed, analogous to \Cref{prop: d2} (2), if $\rZ^*$ is emulated by CBN $\mathfrak{C}$ (with graph $\mathcal{G}$) and interventions $\mathcal{I}^*$, then (using the terminology of \citet{janzing2024phenomenological}) $\mathcal{G}$ is not a valid causal graph over variables $\rZ^*$  if $\mathcal{I}^*$ contains a minimal decomposable multi-node intervention.
There are two differences between $\int_S$ and Definition 3 of \citet{janzing2024phenomenological} to take notice of. 

\paragraph*{1)} \citet{janzing2024phenomenological} require that each $\mathcal{A}_j$ is nonempty, that is, intuitively, that it is possible to intervene on every node.
This requirement may be overly restrictive as the following example illustrates. Assume we are
 interested in the causal effect of a treatment $T^*$ on an outcome $Y^*$, where $W^*$ is a confounder. Say we model the situation using CBN $\mathfrak{C}$ over nodes $(T,Y,W)$ with graph $\mathcal{G}$:
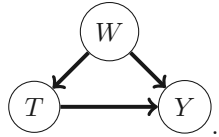
\begin{center}
\begin{tikzpicture}
    \node[draw, circle] (T) at (0,0) {$T$};
    \node[draw, circle] (Y) at (2,0) {$Y$};
    \node[draw, circle] (W) at (1,1) {$W$};
    \draw[->,line width=1.5pt] (T) -- (Y);
    \draw[->,line width=1.5pt] (W) -- (Y);
    \draw[->,line width=1.5pt] (W) -- (T);
\end{tikzpicture}.
\end{center}

If, for every $a\in \mathcal{A}$,
\begin{align}
    &\mathcal{L}^{a}(Y^*\mid T^*, W^*)\sim p^{\mathfrak{C}}_{Y \mid T,W}, \text{ and} \label{eq1}\\ 
    &\mathcal{L}^{a}(W^*)\sim p^{\mathfrak{C}}_{W}, \label{eq2}
\end{align}
then $\mathcal{G}$ is not a valid causal graph (using the terminology of \citet{janzing2024phenomenological}) because no action changes the conditional distribution of outcome given treatment and confounding or the marginal distribution of the confounding.
Yet, this invariance
is what enables
$\mathfrak{C}$ to be useful for making predictions about the actions that affect the conditional distribution of $T^*$ given $W^*$.
Indeed, our formal framework correctly captures the usefulness of such a model: If \Cref{eq1} and \Cref{eq2} hold for every $a\in \mathcal{A}$, then $\mathfrak{C}$ is an $\mathcal{I}-\int$ valid model of $(T^*, Y^*, W^*)$ for every set of interventions $\mathcal{I}$ in $\mathfrak{C}$ if $\int$ is an interpretation that satisfies \textbf{D0}, which includes all interpretations considered in this paper.  

In addition, there is a subtle technical issue in \citet{janzing2024phenomenological} that our framework handles correctly.

\paragraph*{2)} 
Under $\int_S$, \Cref{def: single node}, we interpret an action $a$ as an intervention $d=\doop(j\leftarrow q_j)$ on node $j$ if $\mathcal{L}^a(Z_j^*\mid \PA_j^*)\sim q_j$, $\mathcal{L}^a(Z_j^*\mid \PA_j^*)\nsim p^{\mathfrak{C}}_j$, and $\mathcal{L}^a(\rZ^*)$ is Markov w.r.t.\ the DAG.
Therefore, whether $d\in \int^{\mathcal{I}}_S(a)$ depends on both the graph and the specific kernels in the CBN model. 
In the approach by \citet{janzing2024phenomenological},
to decide whether a causal graph is valid,
one has to determine for all $i$ whether the conditional distributions $\mathcal{L}^a(Z_i^*\mid \PA_i^*)$ are the same as $\mathcal{L}^\mathcal{O}(Z_i^*\mid \PA_i^*)$.
However, whether or not $\mathcal{L}^a(Z_i^* \mid \PA_i^*)=\mathcal{L}^\mathcal{O}(Z_i^* \mid \PA_i^*)$ cannot in general be determined based on $\mathcal{L}^a(\rZ^*)$ and $\mathcal{L}^\mathcal{O}(\rZ^*)$, and thus it is unclear what it means.  To illustrate this point, consider that $\mathcal{L}^{\mathcal{O}}(Z_1^*,Z_2^*)=\mathcal{N}\left(\begin{pmatrix}
    0\\
    0
\end{pmatrix},\begin{pmatrix}
    1 &  1\\
    1 & 2
\end{pmatrix}\right)$ and $\mathcal{L}^a(Z_1^*,Z_2^*)=\delta_0 \otimes \delta_5$  and ask whether $\mathcal{L}^a(Z_2^* \mid Z_1^*)=\mathcal{L}^\mathcal{O}(Z_2^* \mid Z_1^*)$.
One might argue that this equality does not hold because $p^{\mathcal{O}}_{Z_2^*\mid Z_1^*}(\cdot \mid 0)$ is a normal density with mean $0$ and therefore does not have point mass at $5$. This, however, does not follow from the observational distribution $\mathcal{L}^{\mathcal{O}}(Z_1^*,Z_2^*)$ since conditional distributions of $Z_2^*$ given $Z_1^*$ can always be changed on null sets w.r.t.\ $\mathcal{L}(Z_1^*)$ without altering the joint distribution of $(Z_1^*,Z_2^*)$ \citep{Lauritzen2019}. In particular, there exists a CBN $\mathfrak{C}$ such that $\mathcal{L}^{\mathcal{O}}(Z_1^*,Z_2^*)=\mathcal{L}^{\mathfrak{C}}(Z_1,Z_2)$ and $\mathcal{L}^{a}(Z_1^*,Z_2^*)=\mathcal{L}^{\mathfrak{C};\doop(Z_1=0)}(Z_1,Z_2)$.
We avoid this problem by comparing $\mathcal{L}^a(Z_2^*\mid Z_1^*)$ with the specific kernels given by the CBN and write, for example, $\mathcal{L}^a(Z_2^*\mid Z_1^*)\sim p^{\mathfrak{C}}_{Z_2\mid Z_1}$ if $p^{\mathfrak{C}}_{Z_2\mid Z_1}$ is a valid Markov kernel for conditional distribution of $Z_2^*$ given $Z_1^*$ in $\mathcal{L}^a(Z_1^*,Z_2^*)$.

This subtlety occurs even for discrete variables. Let CBN $\mathfrak{A}$ be given by graph $X\to Y$ and kernels
\begin{align*}
    \mathcal{L}^{\mathfrak{A}}(X)&=\delta_0\\
    \mathcal{L}^{\mathfrak{A}}(Y\mid X=0)&=\text{Ber}(0.4)\\
    \mathcal{L}^{\mathfrak{A}}(Y\mid X=1)&=\text{Ber}(0.6).
\end{align*}
Assume that $(X^*,Y^*)$ is emulated by $\mathfrak{A}$ and interventions $\mathcal{I}^*=\{\doop(X=1)\}$.
Given action $a\in\mathcal{A}\setminus\{\mathcal{O}\}$,
there is no way to determine whether 
$\mathcal{L}^{\mathcal{O}}(Y^*\mid X^*)=\mathcal{L}^{a}(Y^*\mid X^*)$
since the support for $X^*$ under the two distributions is not overlapping. But we can, for example, say that
$\mathcal{L}^{a}(Y^*\mid X^*)\sim p^{\mathfrak{A};\doop(X=1)}_{Y\mid X}=p^{\mathfrak{A}}_{Y\mid X}$
and
$\mathcal{L}^{\mathcal{O}}(Y^*\mid X^*)\sim p^{\mathfrak{A};\doop(X=1)}_{Y\mid X}=p^{\mathfrak{A}}_{Y\mid X}$.

\subsection{Details on Causal Abstraction (\Cref{app: abstraction})}

\subsubsection{Further details on \Cref{ex: valid to invalid intS}}\label{app: check tau-abstraction}

To check that $(M_{\rY},\mathcal{I}_Y)$ for $\mathcal{I}_Y=\omega_\tau(\mathcal{I}^*)$ is a $\tau$-abstraction of $(M_{\rX},\mathcal{I}^*)$ \citep[Definition 3.13 in][]{beckers2019abstracting}, we must check that 
\begin{enumerate}
    \item $\tau$ is surjective.
    \item There exists a surjective function $\tau_U:[4]\to \{0,1\}^2$ such that $\tau(M^d_{\rX}(u_1))=M^{\omega_{\tau}(d)}_{\rY}(\tau_U(u_1))$ for every $d\in \mathcal{I}^*$ and $u_1\in [4]$.
    \item 
    The third condition of Definition 3.13 in \citet{beckers2019abstracting} holds since $\omega_\tau$ is defined for all interventions in $\mathcal{I}^*$ and we choose $\mathcal{I}_Y=\omega_\tau(\mathcal{I}^*)$.
\end{enumerate}

$\tau: [4]\to \{0,1\}^2$ given by $x\mapsto (\mathbbm{1}(x=3)+\mathbbm{1}(x=4), \mathbbm{1}(x=2)+\mathbbm{1}(x=4))$ is clearly surjective. For each $d\in \mathcal{I}^*$, we compute $\omega_{\tau}(d)$ using Definition 3.12 of \citet{beckers2019abstracting}:
\begin{enumerate}
    \p $\omega_\tau(\doop(X_1=1))=\doop(Y_1=0,Y_2=0)$,
    \p $\omega_\tau(\doop(X_1=2))=\doop(Y_1=0,Y_2=1)$,
    \p $\omega_\tau(\doop(X_1=3))=\doop(Y_1=1,Y_2=0)$,
    \p $\omega_\tau(\doop(X_1=4))=\doop(Y_1=1,Y_2=1)$.
\end{enumerate}
Let $\tau_U:=\tau$. Now, it is straightforward to check the second requirement. For example,
\begin{align*}
    \tau\left(M^{\doop(X_1=1)}_{\rX}(4)\right)&=\tau(1)=(0,0)\\
    &=M^{\doop(Y_1=0,Y_2=0)}_{\rY}(\tau_U(4)),
\end{align*}
and
\begin{align*}
    \tau\left(M^{\doop(X_1=3)}_{\rX}(4)\right)&=\tau(3)=(1,0)\\
    &=M^{\doop(Y_1=1,Y_2=0)}_{\rY}(\tau_U(4)).
\end{align*}

(Note that $\omega_\tau$ would map the empty intervention to the empty intervention and
$\tau(M_{\rX}(u_1))=M_{\rY}(\tau_U(u_1))$ for every $u_1\in [4]$,
that is, $M_\rY$ is also a $\tau$-abstraction of $M_\rX$
if we additionally included the empty intervention in $\mathcal{I}^*$
(this is analogous to the CBN $\mathfrak{C}$ being compatible with $\rY^*=\tau(\rX^*)$).)

\subsubsection{Interventional validity is not preserved by constructive soft abstractions under interpretation $\int_P$}\label{app: beckers -- intA also not preserved}

The following example
shows that constructive soft abstractions
\citep{massidda2023causal}
do not preserve interventional validity;
the example
is a slight modification of Example 4 in \citet{massidda2023causal}, adding nodes $X_0$ and $Y_0$. 

\begin{example}{
\textbf{Transforming an interventionally valid model into an invalid model by a constructive soft abstraction.}}
Let SCM $M_\rX$ be given by
\begin{align*}
    X_0&:=U_0\\
    X_1&:=U_1\\
    X_2&:=U_2\\
    X_3&:=\min(X_1,X_2)\\
    X_4&:=(X_1-X_2)^2
\end{align*}
with $U_0, U_1,U_2 \overset{\text{iid}}{\sim}\text{Unif}(\{0,1\})$,
and let SCM $M_\rY$ be given by
\begin{align*}
    Y_0&:=N_0\\
    Y_1&:=N_1\\
    Y_2&:=N_2\\
    Y_3&:=Y_1+Y_2
\end{align*}
with $N_0, N_1,N_2 \overset{\text{iid}}{\sim}\text{Unif}(\{0,1\})$.
Let $\tau: \{0,1\}^5\to \{0,1\}^3\times \{0,1,2,3\}$
be given by $$(x_0,x_1,x_2,x_3,x_4)\mapsto (x_0,x_1,x_2, 2x_3+x_4).$$ 
$M_{\rY}$ is a constructive soft $\tau$-abstraction of $M_{\rX}$ \citep{massidda2023causal}.  
Let $\mathfrak{A}$ be a CBN with graph and observational distribution induced by $M_\rX$. Assume that $\rX^*$ is emulated by $\mathfrak{A}$ and $\mathcal{I}^*=\{\doop(X_0=0,X_4=1)\}$. Since $\mathcal{I}^*$ only contains perfect interventions, $\mathfrak{A}$ is an $\mathcal{I}-\int_P$ valid model of $\rX^*$ for every set of interventions $\mathcal{I}$ in $\mathfrak{A}$, see \Cref{prop: intA} (1). Let $\mathfrak{C}$ be a CBN with graph and observational distribution induced by $M_{\rY}$, and let $\mathcal{I}$ be a set of interventions in $\mathfrak{C}$ such that $\doop(Y_0=0)\in \mathcal{I}$. Then,
$\mathfrak{C}$ is not an $\mathcal{I}-\int_P$ valid model of $\rY^*:=\tau(\rX^*)$. To see this let $a$ be an action such that $\mathcal{L}^a(\rX^*)=\mathcal{L}^{\mathfrak{A};\doop(X_0=0, X_4=1)}(\rX)$ (such an action exists since $\doop(X_0=0, X_4=1)\in \mathcal{I}^*$). Now $\doop(Y_0=0)\in \int_P^{\mathcal{I}}(a)$ but $\mathcal{L}^{a}(\rY^*)\neq \mathcal{L}^{\mathfrak{C};\doop(Y_0=0)}(\rY)$, for example, because $\mathbb{E}^a(Y_3^*)=\mathbb{E}^{\mathfrak{A};\doop(X_0=0, X_4=1)}(2X_3+X_4)=\frac{3}{2}\neq 1=\mathbb{E}^{\mathfrak{C};\doop(Y_0=0)}(Y_3)$.
\end{example}

\subsubsection{
Further details on \Cref{ex: abstraction}}\label{app: beckers}

To show that $(M_{\rY},\mathcal{I}_Y)$ is not a $\tau$-abstraction of $(M_{\rX},\mathcal{I}^*)$ \citep[Definition 3.13 in][]{beckers2019abstracting}, we must show a violation of at least one of the following criteria 
\begin{enumerate}
    \item $\tau$ is surjective.
    \item
    There exists a surjective function $\tau_U: \mathbb{R}^3\to \mathbb{R}^2$ such that $$\tau(M^d_{\rX}(u_1,u_2,u_3))=M^{\omega_{\tau}(d)}_{\rY}(\tau_U(u_1,u_2,u_3))$$ for every $d\in \mathcal{I}^*$ and $(u_1, u_2, u_3)\in \mathbb{R}^3$.
    \item
    The third condition of Definition 3.13 in \citet{beckers2019abstracting}
    would be trivially violated if we chose
    $\mathcal{I}_Y \neq \omega_\tau(\mathcal{I}^*)$. Therefore, we choose $\mathcal{I}_Y=\omega_\tau(\mathcal{I}^*)$ (here, $\omega_\tau$ is defined for all interventions in $\mathcal{I}^*$).
\end{enumerate}

Assume for contradiction that $(M_{\rY},\omega_\tau(\mathcal{I}^*))$ is a $\tau$-abstraction of $(M_{\rX},\mathcal{I}^*)$. Notice that 
$\omega_\tau(\doop(X_1=x_1))=\emptyset$ for all $x_1\in \mathbb{R}$ \citep[Definition 3.12]{beckers2019abstracting}. 
The second condition then implies that 
\begin{align*}
    \tau(M_{\rX}^{\doop(X_1=x_1)}(u_1, u_2, u_3))&=(x_1+u_2, x_1+u_2+u_3)\\
    &=M_{\rY}(\tau_U(u_1, u_2, u_3)),
\end{align*}
for every $x_1,u_1,u_2,u_3\in \mathbb{R}$. But this is impossible since no function $\tau_U: \mathbb{R}^3\to \mathbb{R}^2$ exists that simultaneously satisfies all the above equations, for example, since this would imply that
the first coordinate $M_{\rY}(\tau_U(1,1,1))_1=\sqrt{2}\tau_U(1,1,1)_1=x_1+1$ for every $x_1\in \mathbb{R}$.
\bibliography{bib_file}       %

\end{document}